\documentclass[10pt]{article}
\RequirePackage{natbib}
\usepackage[top=1in, bottom=1in, left=1in, right=1in]{geometry}

\usepackage[utf8]{inputenc} %
\usepackage[T1]{fontenc}    %
\usepackage{hyperref}       %
\usepackage{url}            %
\usepackage{booktabs}       %
\usepackage{amsfonts}       %
\usepackage{nicefrac}       %
\usepackage{microtype}      %
\usepackage{xcolor}         %

\usepackage[english]{babel}
\usepackage{latexsym,mathtools}
\usepackage{amsmath}
\usepackage{amssymb}
\usepackage{mathtools}
\usepackage{amsthm}
\usepackage{float}
\usepackage{bbm}
\usepackage{subcaption}
\usepackage{wrapfig}  %
\usepackage[dvipsnames]{xcolor}
\usepackage{array}
\usepackage{adjustbox}
\usepackage{float}

\usepackage[capitalize,noabbrev]{cleveref}
\newcommand{\mathd}{\mathrm{d}}
\newcommand{\tmop}[1]{\ensuremath{\operatorname{#1}}}

\theoremstyle{plain}
\newtheorem{theorem}{Theorem}[section]
\newtheorem{proposition}[theorem]{Proposition}
\newtheorem{lemma}[theorem]{Lemma}
\newtheorem{corollary}[theorem]{Corollary}
\newtheorem{fact}[theorem]{Fact}
\theoremstyle{definition}
\newtheorem{definition}[theorem]{Definition}
\newtheorem{assumption}[theorem]{Assumption}
\theoremstyle{remark}

\usepackage[dvipsnames]{xcolor}
\definecolor{darkred}{rgb}{.7,0,0}
\definecolor{darkgreen}{rgb}{.15,.55,0}
\definecolor{darkblue}{rgb}{0,0,0.7}
\definecolor{darksomething}{rgb}{0,0.7,0.7}

\usepackage[textsize=tiny]{todonotes}

\newcommand{\M}{{\mathcal{M}}}

\newcommand{\R}{\mathbb{R}}
\newcommand{\lse}{\operatorname{LSE}}
\newcommand{\E}{\mathbb{E}}
\newcommand{\prob}{\mathbb{P}}
\newcommand{\Pieps}{\Pi_{\M^\epsilon}}

\DeclareMathOperator*{\essinf}{ess\,inf}
\newcommand{\mud}{\mu_{\operatorname{data}}}
\newcommand{\hmud}{\hat{\mu}_{\operatorname{data}}}
\newcommand{\dist}{\operatorname{dist}}

\newcommand{\law}{\operatorname{law}}
\newcommand{\gap}{\Delta_\M}

\newcommand{\sm}{{\operatorname{sm}}}

\usepackage{accents}

\title{Diffusion Models and the Manifold Hypothesis:\\
Log-Domain Smoothing is Geometry Adaptive}

\author{
    \hspace{50pt}
    \and Tyler Farghly\thanks{Authors contributed equally to this work. Correspondence to \texttt{\{last name\}@stats.ox.ac.uk}}
    \and Peter Potaptchik\footnotemark[1]
    \and Samuel Howard\footnotemark[1]
    \and \hspace{50pt}
    \and George Deligiannidis
    \and Jakiw Pidstrigach \\[-10pt]
    \and Department of Statistics, University of Oxford
}

\begin{document}

\maketitle

\begin{abstract}
Diffusion models have achieved state-of-the-art performance, demonstrating remarkable generalisation capabilities across diverse domains. However, the mechanisms underpinning these strong capabilities remain only partially understood. A leading conjecture, based on the manifold hypothesis, attributes this success to their ability to adapt to low-dimensional geometric structure within the data. This work provides evidence for this conjecture, focusing on how such phenomena could result from the formulation of the learning problem through score matching. We inspect the role of implicit regularisation by investigating the effect of smoothing minimisers of the empirical score matching objective. Our theoretical and empirical results confirm that smoothing the score function—or equivalently, smoothing in the log-density domain—produces smoothing tangential to the data manifold. In addition, we show that the manifold along which the diffusion model generalises can be controlled by choosing an appropriate smoothing.
\end{abstract}

\section{Introduction: Diffusion, manifolds and generalisation}\label{sec:intro}

Diffusion models \citep{pmlr-v37-sohl-dickstein15, Song19_estimating_grads, Ho202_DDPM, song2021scorebased} have emerged as a powerful generative framework, achieving state-of-the-art performance across diverse domains, including images \citep{dhariwal2021diffusion, karras22_edm, rombach2022_latentdm}, audio \citep{kong2021diffwave, liu2023_audio_ldm}, and video \citep{ho2022imagenvideo, Blattmann_2023_CVPR}. Beyond their ability to generate high-quality outputs, they are also capable of producing novel samples not present in the training data, indicating a surprising capacity for generalisation.

The goal of diffusion models is to produce samples from a target distribution \(\mu_\text{data}\) on \(\R^d\), given only a finite number of samples. They do this by learning to \emph{reverse} a noising process, \(X_t\), which begins with a random sample of the data distribution \(X_0 \sim \mud\) and gradually transforms it into noise. This process is defined by the stochastic differential equation (SDE),
\begin{equation} \label{eqn:forward_sde}
  \mathd X_t = - \alpha X_t \mathd t + \sqrt{2} \mathd B_t, \quad X_0 \sim
  \mu_{\tmop{data}},
\end{equation}
for some $\alpha \ge 0$, where \(B_t\) denotes the \(d\)-dimensional Brownian motion.

It is well known \citep{haussmann1986time} that the time reversal $Y_{t} := X_{T-t}$ of \eqref{eqn:forward_sde} satisfies
\begin{equation}
        \label{eqn:reverse_sde}
        \mathd Y_t = \alpha Y_t 
        + 2 {\nabla \log p_{T-t}(Y_t)} \mathd t
        + \sqrt{2} \mathd B_t,
\end{equation}
where \(p_t\) denotes the density of the $X_t$.
Therefore, the task of generating samples from $\mu_\text{data}$ can be solved by simulating paths of \eqref{eqn:reverse_sde}. To that end, the unknown \emph{score function}, $\nabla \log p_t$ in \eqref{eqn:reverse_sde}, is approximated by minimizing the \textit{(population) score matching loss} (see \citet{hyvarinen05_score_matching}):
\begin{equation}\label{eqn:SM}
    \ell_\sm(s) = \int_0^T \mathbb{E} \left[\| s (t, X_t) - \nabla \log p_t(X_t) \|^2\right] \mathrm{d} t.
\end{equation}

\begin{figure}[ht]
    \centering    \includegraphics[width=0.65\linewidth]{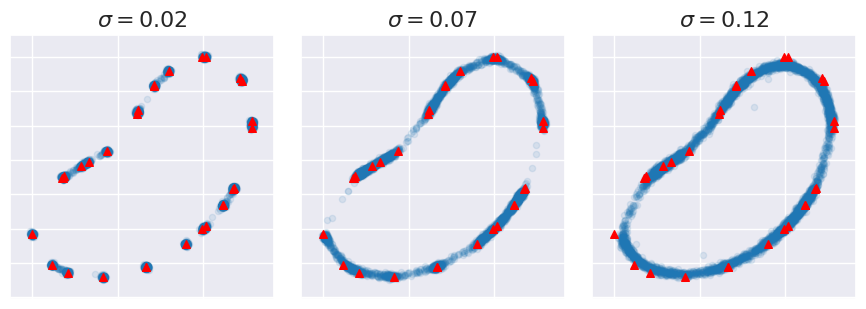}
    \caption{\textbf{Isotropic smoothing of the score function \emph{identifies manifold structure}.} The figure shows training data ($\textcolor{red}{\blacktriangle}$) against generated samples ($\textcolor{MidnightBlue}{\bullet}$) from a diffusion model that is run with the smoothed score $\nabla \log \hat{p}_t \ast \mathcal{N}_{\sigma}$, where the width of the Gaussian smoothing kernel increases from $\sigma = 0.02$ to $\sigma = 0.12$. Notice that for low amounts of smoothing, generated samples are concentrated close to training data and as $\sigma$ increases, generated samples begin to fill out more of the manifold \emph{without} having seen training samples in those regions.}\label{fig:lima_bean}
\end{figure}

In \eqref{eqn:SM}, the expectation is taken over samples from $X_t$ (see \eqref{eqn:forward_sde}), when started from the true data distribution $X_0 \sim \mu_\text{data}$.
In practice, one has access only to a finite dataset of samples $\{x_i\}_{i=1}^N$ from \(\mud\) and so \(\ell_\sm\) must be approximated empirically. Therefore, during training, the noising process is \emph{not} started from the target distribution $\mu_\text{data}$, but is instead initialized from the \emph{empirical measure}, $\hat{\mu}_\text{data} = \frac{1}{N} \sum_i \delta_{x_i}$.
This gives rise to the \emph{empirical score matching loss},
\begin{equation}
    \hat{\ell}_\sm(s) := \int_0^T \mathbb{E} \left[\| s (t, \hat{X}_t) - \nabla \log \hat{p}_t(\hat{X}_t) \|^2\right] \mathrm{d} t,
    \label{eq:SM_emp}
\end{equation}
where $\hat{p}_t$ is the density of forward process, $\hat{X}_t$, which is initialised from the empirical measure $\hat{\mu}_\text{data}$.

One quirk of this objective is that it possesses a unique minimiser\footnote{Here we mean in the \(L^2\) sense: any minimiser of \(\hat{\ell}_\sm\) is identical to \(\nabla \log \hat{p}_t\) almost everywhere.} identical to the \emph{empirical score function}, $\nabla \log \hat{p}_t(x)$. As a result, if one were to reverse the noising process with this minimiser, one would arrive close to the empirical measure $\hat{\mu}_{\text{data}}$, reproducing the training data instead of generating novel samples from the target distribution. In fact, it has been shown that any approximation sufficiently close to \(\nabla \log \hat{p}_t\) will produce samples belonging to the training dataset \citep{pidstrigach2022scorebased}. However, in practice, diffusion models trained with this objective perform well and avoid memorisation, suggesting that regularisation is key to their generalisation capabilities.

The study of generalisation in diffusion models can be divided into three parts:
\begin{enumerate}
    \item The formulation of the learning problem via score-matching and its empirical approximation.
    \item The inductive bias of the training procedure and model architecture employed.
    \item The effect of replacing the minimiser of \eqref{eq:SM_emp} by a regularized version of the minimiser on the reverse SDE and the samples generated.
\end{enumerate}
The second part has been widely investigated in the machine learning and statistics literature, where neural architectures, regularisation and optimisation schemes have been analysed extensively through the lens of generalisation. While the effects of these techniques are far from completely understood---especially in the context of generative modelling---there are numerous studies into how neural network training promotes bias towards smooth functions that interpolate the data \citep{rahaman_spectral_bias, mulayoff2021implicit, ma2021linear, vardi2023implicit}. In contrast, the first and third parts are specific to diffusion models and have received less attention. Therefore, our work focuses on these under-explored parts.

To account for the inductive bias during score matching, we propose a simple model built upon smooth approximations to the minimiser of \(\hat{\ell}_\sm\). In particular, we consider the score function \(s^k\), which smooths the empirical score function $\nabla \log \hat{p}_t$, with a generic probability kernel \(k\):
\begin{equation}
    s^k(t, x) = \int \nabla \log \hat{p}_t(y) \, k_x(\mathd y).
    \label{eq:kernel_smoothed_score}
\end{equation}
While this significantly simplifies the possible inductive bias employed during training, it succeeds in capturing a defining property of diffusion models: as a result of the approach of score-matching, any smoothing resulting from inductive bias occurs at the level of the score function—in the \emph{log-domain}.

Beyond these considerations, understanding the generalisation of diffusion models also requires an analysis of the data distributions they successfully model. There is growing support for the theory that diffusion models are particularly successful at modelling distributions adhering to the \emph{manifold hypothesis}, wherein high-dimensional data concentrates on a lower-dimensional manifold \citep{Tenenbaum_manifold_2000, Bengio2012RepresentationLA, Goodfellow16_DL}. A standing conjecture is that generative models, including diffusion models and flow-based approaches, owe their success partly to their capacity to uncover these hidden structures \citep{pidstrigach2022scorebased, de_bortoli2022convergence, loaiza-ganem2024deep,Farghly2025-fn}. This raises a critical question: what mechanism allows diffusion models to so effectively identify and leverage this underlying manifold structure? In this work, we argue that the practice of smoothing in the \emph{logarithmic domain} plays a key role.

\section{Log-domain smoothing retains geometric structure}\label{sec:intuition}

This work contributes to a small but growing literature on the effect of score-smoothing in diffusion models \citep{scarvelis2023closedformdiffusionmodels, chen2025interpolationeffectscoresmoothing, gabriel2025kernelsmoothedscoresdenoisingdiffusion}.
In this section, we provide intuition for how diffusion models inherently perform smoothing in the log-density domain via their score-matching objective. We then show that log-domain smoothing is crucial for preserving the underlying manifold structure of the data. Finally, we revisit the manifold hypothesis and explore how specific characteristics of the smoothing kernel can guide the model to generalise along different geometries interpolating the data.

\begin{figure}[t]
    \centering
    \begin{tabular}[t]{c c}
        \begin{tabular}{m{0.085\linewidth} c}
            \footnotesize{Density smoothing} & \adjustbox{valign=m}{\includegraphics[width=0.38\linewidth]{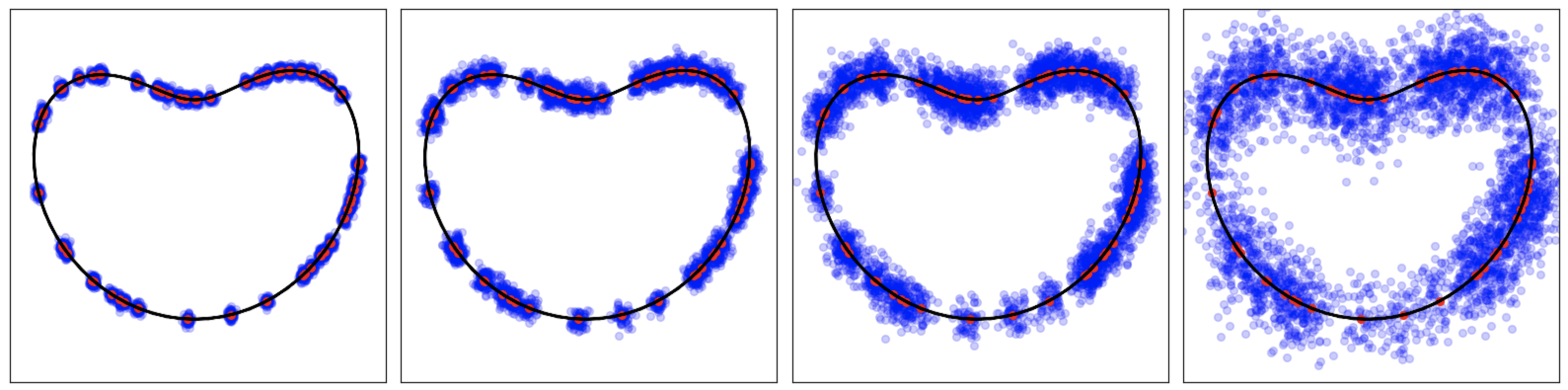}} \\
            \footnotesize{Score smoothing} & \adjustbox{valign=m}{\includegraphics[width=0.38\linewidth]{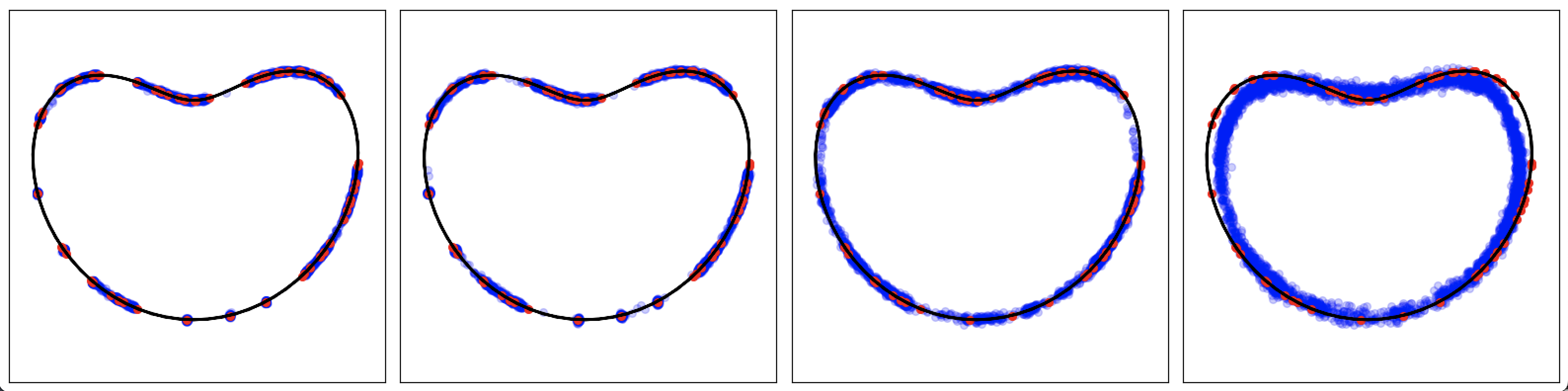}}
        \end{tabular} &
        \adjustbox{valign=m}{\includegraphics[width=0.42\linewidth]{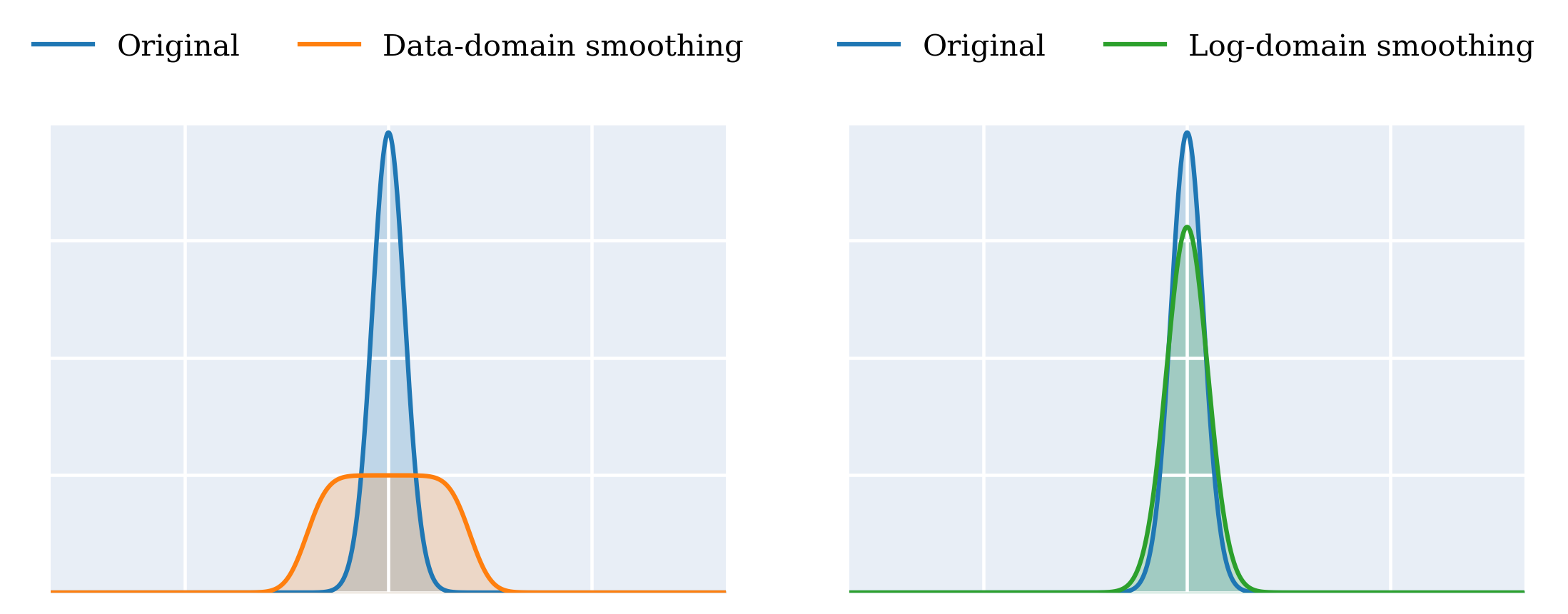}}
    \end{tabular}
    \caption{\textbf{Density smoothing generates samples off-manifold, whereas score smoothing generates samples that retain manifold structure.} Left: The plots compare samples ($\textcolor{blue}{\bullet}$) drawn from a KDE (top) versus from a diffusion model with the smoothed score (bottom) from Figure~\ref{fig:lima_bean} (training data is $\textcolor{red}{\bullet}$). The scale of the smoothing kernel increases from left to right. Right: 1D intuition for data-domain versus log-domain smoothing. The left sub-figure shows the Gaussian ($\textcolor{blue}{-}$) smoothed in data-domain ($\textcolor{orange}{-}$), and the right sub-figure shows the Gaussian smoothed in log-domain ($\textcolor{darkgreen}{-}$) with the same kernel.}
    \label{fig:two_images}
\end{figure}

\subsection{Diffusion models smooth in the log-domain}
As outlined in the introduction, we model the inductive bias of neural network training by a smoothing kernel \(k\) (see \eqref{eq:kernel_smoothed_score}). Assuming that the nature of the inductive bias does not vary too rapidly over the spatial parameter, we can treat the kernel as locally constant. In this case, the convolution will commute with the gradient operation, and we obtain the following simple but consequential equation:
\begin{equation}\label{eq:gradexch}
    s^k(t, x) = k * \nabla \log \hat{p}_t(x) = \nabla \left(
    k * \log \hat{p}_t(x)
    \right).
\end{equation}
Therefore, smoothing the score function corresponds to smoothing the empirical density $\hat{p}_t$ in the \emph{log-domain}, as opposed to smoothing at the density-level directly. Consequently, when a trained diffusion model generates samples by following the reverse SDE in  \eqref{eqn:reverse_sde}, it effectively utilises scores derived from this log-smoothed version of the empirical density:
\begin{equation}
\label{eq:log_smoothed_density}
    \hat{p}_t^k(\mathd x) \propto \exp \bigg ( \int \log \hat{p}_t(y) k_x(\mathd y) \bigg ) \mathd x,
\end{equation}
where we use \(k_x\) to denote the distribution of the smoothing kernel centred at \(x\).

Sampling from a diffusion model involves discretising the backwards process in \eqref{eqn:reverse_sde} using the learned score function approximation. To correct for discretisation and approximation error, so-called \emph{corrector-steps} are interspersed between iterations \citep{song2021scorebased, karras22_edm}. This involves running Langevin Monte Carlo to correct the distribution of the diffusion model, maintaining correspondence between the diffusion model samples and the distribution associated with the score function. Furthermore, to account for instability near convergence, the technique of \emph{early stopping} is often used, where the reverse process is terminated an amount of time \(\epsilon > 0\) before convergence \citep{song2021scorebased}. With this, we arrive at our approximation of the diffusion model output as the log-domain smoothed empirical measure, \(\hat{p}_\epsilon^k\). Indeed, this is the density recovered by the diffusion model with score function \(s^k\) with sufficient correction steps and sufficiently fine discretisation.

This characterisation of diffusion model output through smoothing in the log-domain identifies a distinction between diffusion models and classical density-level estimators. For example, the classical method of kernel density estimation (KDE) \citep{tsybakov2009nonparametric} approximates the underlying data distribution by smoothing the empirical measure \(\hmud\) with respect to a smoothing kernel \(k\), providing an estimator of the form,
\begin{equation*}
    \hat{q}^k_{\text{KDE}}(dx) = \int k_x(y) \hmud(dy) dx.
\end{equation*}
In words, the KDE also approximates the data distribution by smoothing the empirical data distribution, but it performs its smoothing in the \emph{data-domain} as opposed to the log-domain.

\subsection{Smoothing in log-domain preserves manifold structure }\label{sec:preserves}
Effective capture of the geometry underlying the data distribution is a critical aspect of effective generative modelling. We briefly provide some intuition for why log-domain smoothing plays a vital role here. Consider a data distribution that is concentrated on a manifold within the larger data space. The density of this distribution would be positive on this manifold and zero (or practically negligible) elsewhere. Data-domain smoothing techniques, such as the KDE, yield a positive probability density wherever the smoothing kernel overlaps with the manifold, leading to a \textit{smearing} of the density away from the manifold.
In contrast, smoothing in the log-domain offers a distinct advantage—when we transition to the log-domain, locations where the original density is zero are mapped to $-\infty$. Consequently, if a smoothing kernel extends into regions off-manifold, the resulting smoothed log-density in those regions will also equal $-\infty$. While the empirical density  $\hat{p}_t$ is technically positive everywhere, it will have exceedingly small values in regions distant from the data manifold, and so the intuition carries over.

In Section \ref{sec:geometry-adaptivity-of-log-domain-smoothing}, we make this intuition more concrete, theoretically showing that smoothing the empirical density in the log-domain approximates smoothing along the data manifold. We start by analysing the case in which the data is supported on a linear manifold (see Section \ref{sec:warmup}), where we obtain a perfect correspondence between smoothing the empirical density in the log-domain and smoothing along the (linear) data manifold. Then, in Section \ref{sec:main_result}, we state our main theoretical result that generalises this to the curved manifold setting—showing that smoothing in the log-domain approximates a geometry-adapted smoothing, which generates samples close to the underlying manifold.

\begin{figure}[t]
    \centering
    \begin{subfigure}[b]{0.2\linewidth}
        \centering
        \includegraphics[width=\textwidth]{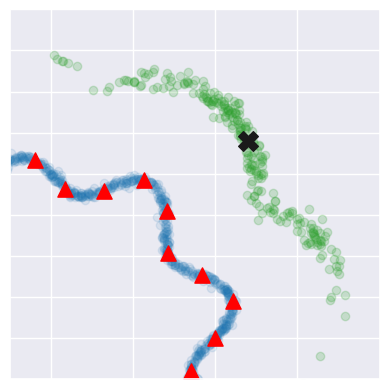}
    \end{subfigure}%
    \hspace{0.3em}%
    \begin{subfigure}[b]{0.2\linewidth}
        \centering
        \includegraphics[width=\textwidth]{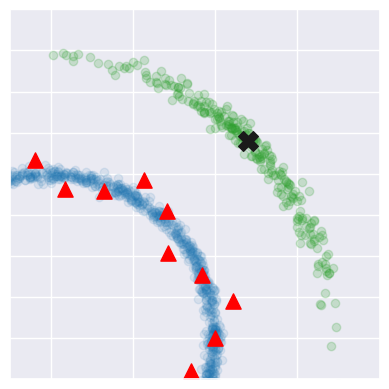}
    \end{subfigure}
    \hspace{2em}
    \begin{subfigure}[b]{0.2\linewidth} %
        \centering
        \includegraphics[width=\linewidth]{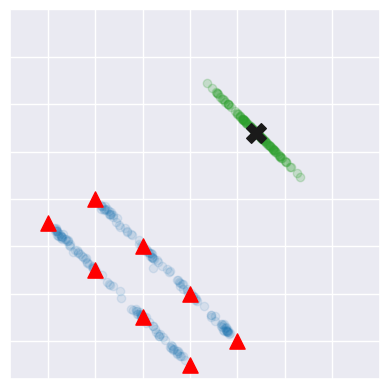}
    \end{subfigure}
    \hspace{0.3em}%
    \begin{subfigure}[b]{0.2\linewidth}
        \centering
        \includegraphics[width=\linewidth]{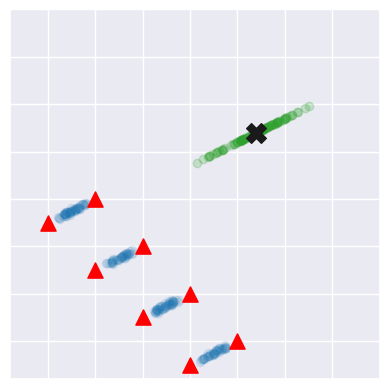}
    \end{subfigure}
    \caption{\textbf{The choice of smoothing kernel influences the manifold on which generated samples lie.} The empirical score function corresponding to the training data ($\textcolor{red}{\blacktriangle}$) is smoothed with different (data-dependent) kernels. To visualize the smoothing kernels, we generate samples ($\textcolor{LimeGreen}{\bullet}$) from $k_x$. We use the smoothed score functions to generate samples ($\textcolor{MidnightBlue}{\bullet}$) from the resulting diffusion models. Notice that despite using the same training data, different smoothing kernels generate samples that lie on different manifolds.}
    \label{fig:toy_different_smoothing}
\end{figure}

\subsection{Choosing an interpolating manifold via geometric bias}\label{sec:intro_geom_bias}
The manifold hypothesis traditionally assumes that data lies on a low-dimensional \emph{true} submanifold that must be identified by your learning algorithm. However, in many real-world scenarios, this assumption is too rigid: rather than adhering to a single well-defined manifold, data likely exhibits geometric structure that different interpolating manifolds can approximate. This is especially true when the size of the dataset is small relative to the dimension and curvature of the space. In such settings, the focus is no longer on recovering a \emph{true} manifold, but on choosing a \emph{plausible} interpolating manifold. With this, we arrive at our next question of interest: how does the algorithm \emph{choose} the interpolating manifold?

Returning to score smoothing, this reframes our central question to one of understanding how smoothing induces biases in the geometric structure of generated samples. As Figure~\ref{fig:toy_different_smoothing} illustrates, the geometry of the output distribution depends entirely on the directions of smoothing. By choosing a kernel that aligns with certain geometric structures (e.g., tangent to a circle), the diffusion model is biased to interpolate along the corresponding manifold. We term this relationship between sample geometry and the smoothing kernel---or more broadly, the inductive bias of the score matching algorithm---as the \emph{geometric bias} of the diffusion model. In Section \ref{sec:geometric_bias}, we develop theory and experiments that identify key structural properties of the smoothing kernel that dictate this bias.

\section{Geometry-adaptivity of log-domain smoothing}\label{sec:geometry-adaptivity-of-log-domain-smoothing}

In this section, we provide theoretical results that aim to capture and make concrete the intuition presented in \Cref{sec:preserves}, examining the smoothed density $\hat{p}_\epsilon^k$ in \eqref{eq:log_smoothed_density} as a tractable proxy for the diffusion model output. We also consider a \textit{manifold-adapted} counterpart to $\hat{p}_\epsilon^k$, denoted by $\hat{p}_\epsilon^{k^\mathcal{M}}$. The kernel $k^\mathcal{M}$ acts similarly to $k$, but restricts the smoothing to occur only along level sets of the manifold, spreading mass along the manifold without destroying the geometric structure. Note that a priori, one may not have knowledge of the manifold structure and thus could not construct such a kernel $k^\M$---here, we use it merely as a theoretical tool to represent a desirable behaviour of a generative model: that interpolation is performed in a way that identifies and preserves geometric structure. That is, to show that smoothing with a generic kernel \(k\) is geometry-adaptive, we show that $\hat{p}_\epsilon^k$ is \textit{close} to its manifold-adapted counterpart, $\hat{p}_\epsilon^{k^\mathcal{M}}$.

This raises the following natural questions: How do these two objects $\hat{p}_\epsilon^k$ and $\hat{p}_\epsilon^{k^\mathcal{M}}$ relate? Ideally, they should be similar as this would show that the uninformed kernel $k$ can automatically adapt to manifold structure. Moreover, this raises the question of which properties of the data manifold and the diffusion model influence this relation?
In this section, we present theoretical results addressing these questions.

\subsection{Warm-up: linear setting}\label{sec:warmup}
We first restrict our analysis to the setting where the data distribution is supported on a \(d^*\)-dimensional affine subspace \(\M = \{x \in \R^d: Ax = b\}\), where \(A \in \R^{d^* \times d}\) is a row-orthonormal matrix and \(b \in \R^{d^*}\). This allows us to provide intuition for our main result in \cref{sec:main_result} where we generalise to the more complex general case. Consider the simplified setting where the kernel $k$ is location-independent, i.e. it is of the form
\begin{equation*}
    k_x := \law(x + \xi),
\end{equation*}
where \(\law(\cdot)\) denotes the distribution of a given random variable and \(\xi\) is a zero-mean random variable, independent of \(x\). In this setting, we establish the following result.
\begin{proposition}\label{prop:linear_result}
The log-domain smoothed density satisfies the property,
\begin{equation}\label{eq:warmup_manifold_kernel}
    \hat{p}^k_\epsilon = \hat{p}^{k^\M}_\epsilon, \qquad \text{where } k^{\M}_x := \law(x + P\xi),
\end{equation}
where \(P := I-A^T A\) is the projection onto $\text{Null}(A) =  \{x \in \R^d: Ax = 0\}$.
\end{proposition}
The kernel \(k^{\M}_x\) is a modification of \(k_x\) that smooths only along the plane parallel to $\M$ passing through $x$. From this proposition, we see that in the affine setting, smoothing in the log-domain with respect to a generic kernel \(k\) is equivalent to smoothing with the geometry-adapted kernel \(k^\M\). In other words, log-domain smoothing is fundamentally \textit{geometry-adaptive}.

We now provide a brief exposition of the proof technique which also forms the basis of the proof in the more general setting. Given the training set $\{x_i\}_{i = 1}^N$, we can directly compute the noised empirical densities $\hat{p}_t(x)$ and the corresponding score functions. Recall that the LogSumExp ($\lse$) function is defined on any finite set \(\{r_i\}_i \subset \R\) and is given by \(\lse(\{r_i\}_i) := \log(\sum_i \exp(r_i))\). Using this function, we can succinctly express the empirical log-density as
\begin{equation}\label{eqn:emp_density_lse}
    \log \hat{p}_t(x) = \lse \Big ( \big \{ - \|x - \mu_t x_i\|^2/(2 \sigma^2_t) \big \}_{i=1}^N \Big ) +C_t,
\end{equation}
for data-independent quantities \(C_t, \mu_t, \sigma_t\) given in Appendix \ref{app:diffusion_details}. This allows us to utilise the following property of the \(\lse\) function.
\begin{fact}\label{fact:lse_shift}
For any \(\{r_i\}_i \subset \R\) and any constant \(c \in \R\), $ \lse(\{r_i + c\}_i) = \lse(\{r_i\}_i) + c.$
\end{fact}

Using this fact, we can decompose the log-density into directions tangent and normal to the data manifold. Indeed, using the fact that \(x_i \in \M\) we obtain,
\begin{align}
    \log \hat{p}_t(x + \xi) &= \lse \Big ( \big \{ -(\|P(x + \xi - x_i)\|^2 + \|A^T A(x + \xi - x_i)\|^2)/(2 \sigma_t^2) \big\}_{i} \Big ) + C_t \nonumber \\
    &= \lse \Big ( \big \{ -(\|P(x + \xi - x_i)\|^2 + \|A(x + \xi) - b\|^2)/(2 \sigma_t^2) \big\}_{i} \Big ) + C_t \nonumber \\
    &= \underbrace{\lse \Big ( \big \{ -\|P(x + \xi - x_i)\|^2/(2 \sigma_t^2) \big\}_{i} \Big )}_{\text{tangent}} 
    - \underbrace{\|A(x + \xi) - b\|^2 / (2 \sigma_t^2)}_{\text{normal}} + C_t,\label{eqn:tang_norm_decomp}
\end{align}
In other words, interactions between the noise \(\xi\) and the data occur only in the tangent direction, and the normal direction is constant with respect to the examples \(\{x_i\}_i\). Once taking the expectation of the above expression, we obtain that the log-density of \(\hat{p}_t^k\) is identical, up to a constant, to the log-density of \(\hat{p}^{k^\M}_t\), which only applies smoothing in directions tangent to the manifold. We refer to Appendix \ref{app:linear} for the complete derivation.

\subsection{The case of curved manifolds}\label{sec:main_result}
In this section, we state the main theoretical contribution of this work in which we show that smoothing in log-density is fundamentally geometry-adaptive. Similar to the analysis in Section \ref{sec:warmup}, we do this by deriving a relationship between $p_\epsilon^k$ using an uninformed kernel $k$, and $p_\epsilon^{k^\mathcal{M}}$ using its manifold-adapted counterpart $k^\mathcal{M}$. We depart from the linear case, allowing for curved manifolds satisfying the following assumption.

\begin{assumption}\label{ass:manifold}
Suppose that \(\mud\) lies on a smooth compact submanifold \(\M \subset \R^d\), and that \(\mud\) restricted to \(\M\) admits a density \(p_\mu\) satisfying \(c_\mu := \inf_{\M} p_{\mu} > 0\).\footnote{Here, we take \(p_\mu\) to be the density with respect to the volume measure of the manifold \(\M\), which is itself inherited from the Lebesgue measure.}
\end{assumption}

Our approach to generalising the proof from the previous section is to use the defining feature of Riemannian manifolds---that in proximity, the manifold behaves as if it were flat. The distance that one must be to the manifold depends on the curvature of the manifold, which we control with an object from differential geometry called the \textit{reach}. This object defines the maximum distance from the manifold for which the projection to the manifold, \(\Pi_\M\), is well-defined, i.e. a unique element of the manifold is closest.
\begin{assumption}[Manifold reach]\label{ass:reach}
The manifold \(\M\) has a reach no smaller than \(\tau > 0\), i.e. for all \(x \in \R^d\) with \(\dist(x, \M) < \tau\), there exists a \textit{unique} \(x^\star \in \M\) such that \(\dist(x, \M) = \|x - x^\star\|\).
\end{assumption}
For example, if \(\M\) were a sphere, the projection to the manifold would be unique as soon as the distance to the manifold is restricted to be less than the radius. Hence, the reach in this case is given by the radius. In the case where the manifold is an affine subspace, the projection is always well-defined, and so the reach is infinite. The reach is related to the maximum curvature of the manifold, becoming smaller as the manifold becomes more curved. By assuming that the reach is bounded below, we are effectively requiring that the curvature of the manifold is globally upper-bounded. This assumption, as well as the lower bound on the density, has been used in several recent works and is common in the manifold hypothesis and manifold learning literature \citep{Aamari2017-he, potaptchik2024linearconvergencediffusionmodels, azangulov2024convergencediffusionmodelsmanifold}. We refer to Appendix \ref{app:reach} for further discussion and details regarding the reach of the manifold.

To generalise the manifold-adapted kernel in \eqref{eq:warmup_manifold_kernel} to this more general setting, we consider the projection of the kernel onto the level sets of the manifold \(\M_r = \{x \in \R^d: \dist(x, \M) = r\}\). We define this manifold-adapted modification by
\begin{equation}\label{eq:man_adapt_curved}
    k^\M_x := ( \Pi_{\mu_\epsilon \M_{r(x)}} )_* k_x, \qquad r(x) := \E_{Y \sim k_x}[\dist(Y, \mu_\epsilon \M)^2]^{1/2}
\end{equation}
where \(\mu_\epsilon \M\) is the element-wise scaling of \(\M\) by \(\mu_\epsilon\) and \((\Pi_{\mu_\epsilon \M_{r(x)}})_*\) denotes the push-forward by the projection mapping \(\Pi_{\mu_\epsilon \M_{r(x)}}\), that is, the distribution of \(\Pi_{(\mu_\epsilon \M)_{r(x)}}(Y), Y \sim k_x\). The function \(r(x)\) approximates the distance of \(x\) to the manifold, but with some correction according to the variance of the kernel in directions normal to the manifold. Therefore, similar to the definition in \eqref{eq:warmup_manifold_kernel}, the kernel \(k^\M_x\) is a modification of \(k_x\) adapted to the geometry of \(\M\) by smoothing only in directions tangential to the manifold. We refer to Appendix \ref{app:projection_details} for some additional details regarding the definition of \(k^\M\), including a discussion on the well-posedness of the projection function.

The variance of the smoothing kernel \(k\) in directions normal to the manifold will prove to be an important object in our bound, leading us to make the following assumption.
\begin{assumption}\label{ass:sg_smooth}
There are constants \(K, K_{\max} \geq 0\) such that for all \(x \in \R^d\), \(Y \sim k_x\),
\begin{equation*}
    \E \big [ | \dist(Y, \M) - \dist(x, \M) |^2 \big ] \leq K^2, \qquad | \dist(Y, \M) - \dist(x, \M) | \leq K_{\max}, \ \text{ almost surely.}
\end{equation*}
\end{assumption}
By measuring the change in distance to the manifold caused by the smoothing, the quantities \(K\) and \(K_{\max}\) capture the scale of the noise applied in directions normal to the manifold. This quantity also captures how adapted the smoothing kernel \(k\) already is to the manifold structure since if the majority of the mass of \(k\) lies tangential to the manifold, then \(Y \sim k_x\) will remain close to the level set going through \(x\) and so \(K\) can be chosen to be small. For example, in the Gaussian case \(k_x = \mathcal{N}(x, \sigma^2 I_d)\), we have that \(K^2 \approx (d-d^*) \sigma^2\) whenever \(\sigma\) is taken small.

Unlike in the affine case, we cannot obtain equivalence between \(\hat{p}^k_\epsilon\) and its manifold-adapted counterpart \(p_\epsilon^{k^\M}\) in general, so instead we show that these distributions are \emph{close}. The notion of closeness we consider is the R\'enyi divergence, \(D_q\)---a natural generalisation of the Kullback-Leibler divergence (the case of \(q = 1\)) that interpolates between stronger divergences as \(q\) is taken larger. For the sake of brevity, we leave the definition and a brief exposition on the R\'enyi divergence to Appendix \ref{app:renyi} and we state our main result.

\begin{theorem}\label{thm:density_ratio_bound_simplified}
Suppose that assumptions \ref{ass:manifold}, \ref{ass:reach}, and \ref{ass:sg_smooth} hold and that \(K_{\max} < \tau/96\). Then, for any \(q \in [1, 1+\tau/96K], \delta \in (0, 1]\), whenever \(N > N_{\min}(\delta), \epsilon < \epsilon_{\max}\) we obtain with probability at least \(1 - \delta\) that,
\begin{equation*}
    D_q \big ( \hat{p}_\epsilon^{k^\M} \big \| \hat{p}_\epsilon^k \big ) \lesssim \frac{K}{\tau} \, \max \Big \{ d^* + 1, \, \big ( c_\mu^2 N \big )^{-\frac{1}{d^*}} \epsilon^{-1} \Big \}\,,
\end{equation*}
where the quantities \(\epsilon_{\max}\) and \(N_{\min}(\delta) \lesssim (d^* + \log(\delta^{-1})) \tau^{-2d^*}\) are defined in \eqref{eq:epsilon_min} and \eqref{eq:N_min}, respectively.\footnote{Here, \(\lesssim\) denotes an upper bound that ignores multiplicative logarithmic factors.}
\end{theorem}
For large \(N\), the right-hand side becomes a function of three terms capturing dimensionality, curvature of the manifold and scale of the smoothing kernel. This bound shows that for log-domain smoothing to become geometry-adaptive, it is sufficient for the scale of smoothing normal to the manifold to be small relative to the manifold curvature and dimension. When \(N\) is small relative to \(\epsilon^{-d^*}\), the bound becomes \(K/\tau \epsilon\), highlighting the role that early stopping plays in the data-sparse setting. The dependence on $K$ also provides insight for how the behaviour of $\hat{p}_\epsilon^k$ depends on the kernel's manifold-alignment---when $k$ is already more aligned with the manifold structure, $K$ is smaller, and the closer $\hat{p}_\epsilon^k$ is to its manifold-adapted counterpart $\hat{p}_\epsilon^{k^\mathcal{M}}$.

We once again emphasise the key difference between the log-domain smoothing that we consider, and the traditional KDE approach which instead smooths the empirical measure in the density-domain. As KDE bandwidth increases, samples rapidly leave the data manifold, whereas smoothed-score diffusion models produce new samples along the manifold structure without deviating far from it.

\begin{figure}[t]
    \centering
    \begin{minipage}[t]{0.65\textwidth}
        \vspace{-13em}  %
        \includegraphics[width=\linewidth]{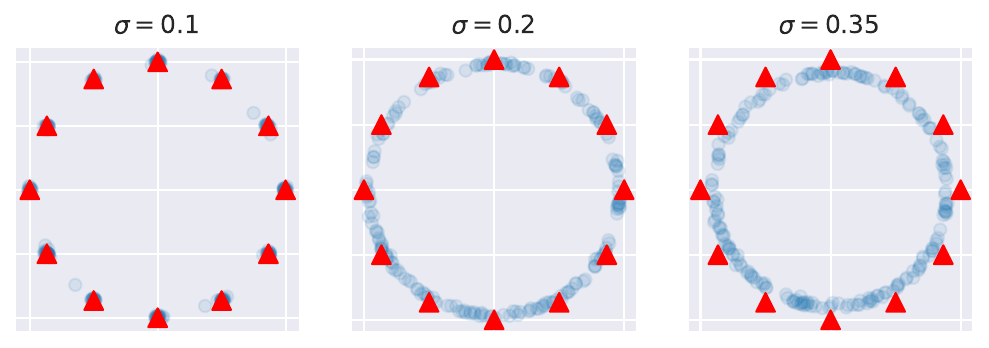}
    \end{minipage}%
    \hspace{1em}%
    \begin{minipage}[t]{0.3\textwidth}
        \includegraphics[width=\linewidth]{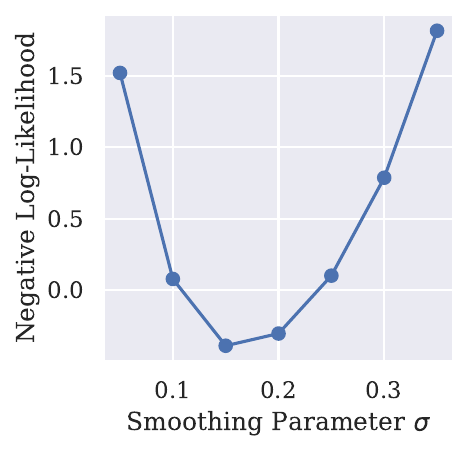}
    \end{minipage}
    \caption{\textbf{Score smoothing can promote generalisation along curved manifolds, but too much smoothing can distort the desired structure.}
    Left: Training data ($\textcolor{red}{\blacktriangle}$) against generated samples ($\textcolor{MidnightBlue}{\bullet}$) using isotropic Gaussian score smoothing with variance \(\sigma^2\). Right: Corresponding population negative log-likelihood, calculated for 1000 points on the true circular manifold. See Appendix \ref{app:circle_details} for details.}
    \label{fig:2d_circle_smoothing}
\end{figure}

\subsection{Log-domain smoothing and generalisation}\label{sec:theory-manifold-adapted-smoothing}
So far, we have presented theoretical results pertaining to the similarity of $\hat{p}_\epsilon^k$ and its manifold-adapted counterpart $\hat{p}_\epsilon^{k^\mathcal{M}}$. While it is intuitively clear that smoothing with the manifold adapted kernel $k^\mathcal{M}$ will help promote the kind of generalisation behaviour that we desire, we provide two results to validate that this is indeed the case. The following result demonstrates that \(\hat{p}_\epsilon^{k}\) preserves mass concentration around the manifold structure.
\begin{corollary}
\label{prop:manifold_concentration}
Consider the setting of Theorem \ref{thm:density_ratio_bound_simplified}, then for any \(\delta \in (0, 1]\), whenever \(\epsilon < \epsilon_{\max}, N > \max\{N_{\min}(\delta), \epsilon^{-2}\}\), we obtain that with probability at least \(1-2\delta\) that,
\begin{equation*}
    \prob_{Y \sim \hat{p}^{k}_\epsilon} \Big ( \operatorname{dist}(Y, \M) \geq r + m \Big | S \Big ) \leq 2 \exp \Big (- r^2 / 8 \epsilon \Big ), \qquad \text{for all } r \geq 0,
\end{equation*}
where \(m > 0\) is some quantity satisfying,
\begin{equation*}
    m^2 \lesssim K^2 + \frac{K}{\tau} \, \max \Big \{ d^* + 1, \, \big ( c_\mu^2 N \big )^{-\frac{1}{d^*}} \epsilon^{-1} \Big \} + \epsilon d + (\epsilon/c_\mu)^{2/d^*}.
\end{equation*}
\end{corollary}

This corollary shows that the distance to the manifold decays exponentially fast, obtaining nearly the same rate of concentration as the noised empirical measure \(\hat{p}_\epsilon\), prior to smoothing. In other words, smoothing in the log-density domain \textit{preserves} concentration to the manifold. We note that when \(K\) is large, the concentration bound becomes less strong.

Next, we show that smoothing with $k^\M$ does indeed distribute mass along the manifold structure. We let \(T_x \M\) denote the tangent space of \(\M\) at \(x\), i.e the set of vectors tangent to \(\M\) at \(x\).
\begin{proposition}
\label{prop:distributed_on_manifold}
Consider the setting of Theorem \ref{thm:density_ratio_bound_simplified}, let \(\delta \in (0, 1]\) and suppose that \(N > N_{\min}(\delta)\). Then, with probability at least \(1-\delta\), it holds that for any \(x \in \M\),
\begin{equation*}
     \hat{p}^{k}_\epsilon(x) \geq \hat{p}^{k}_\epsilon(x_i^\star) \exp \bigg ( - \frac{C F (K K_{\max} + (c_\mu^2 N)^{-1/2d^*})}{\sigma_\epsilon^{2}}  \|x - x_i^\star\| \bigg ), \qquad x_i^\star := \operatorname{argmin}_{\{x_i\}_{i=1}^N} \|x - x_i\|,
\end{equation*}
for some \(C \lesssim 1\) where we define \(F^2 := \sup_{x \in \M, v \in T_x \M} \frac{v^T \mathcal{I}(x) v}{\|v\|^2}\) and \(\mathcal{I}(x) \in \R^{d \times d}\) is the Fisher information matrix of \(k_x\).
\end{proposition}
The quantity \(F\) is an upper bound on the Fisher information of the kernel along the manifold. In the case where \(k\) is a Gaussian kernel with variance \(\sigma^2 I_d\), we have that \(F \approx 1/\sigma^2\) for \(\sigma^2\) small. Thus, whenever \(\sigma_\epsilon^2 = \mathcal{O}(\epsilon)\) is small relative to \(\sigma^2\), we obtain that arbitrary points on the manifold receive similar density as the training data.

The combination of these two results show that log-domain smoothing distributes probability mass along the manifold (\Cref{prop:distributed_on_manifold}) while preserving geometric structure (\Cref{prop:manifold_concentration}). While further work must be done to obtain rigorous generalisation bounds, the above results already suggest an interesting relationship between the scale of log-domain smoothing and generalisation. As the scale of the smoothing grows, the quantity \(K\) grows also, and once it becomes large relative to \(\tau/\epsilon d^*\) the strength of the bound in Corollary \ref{prop:manifold_concentration} weakens. Simultaneously, as the scale of smoothing grows, \(F\) decays, increasing the distribution of mass along the manifold. This suggests a possible trade-off in the generalisation error that is governed by the scale of the smoothing and its relationship to \(\epsilon, d^*\) and \(\tau\). We explore this in Figure \ref{fig:2d_circle_smoothing} where we plot population error (given by the negative log-likelihood) against scale of smoothing for a simple 2\(d\) example. This produces a U-shaped curve, suggesting that such a trade-off is indeed present and that the smoothing parameter can be tuned to obtain generalisation gains over the empirical measure \(\hat{p}_\epsilon\), with too much smoothing worsening generalisation.

\section{Rethinking the manifold hypothesis: geometry and inductive bias} \label{sec:geometric_bias}
So far, we have considered the traditional setting of the manifold hypothesis, where the goal is to identify the \emph{true} geometry hidden within the data. In Section \ref{sec:geometry-adaptivity-of-log-domain-smoothing}, we perform our analysis under the assumption that the data distribution is supported on a well-defined ground truth manifold.
However, given a finite dataset there are often many plausible ways to interpolate between data---particularly in high dimensions, where the distance between neighbouring examples becomes large. In practice, the notion of a \textit{correct} interpolation is task-specific and subjective: as long as the interpolation aligns with application-specific criteria, the resulting generative model is deemed successful. This is exemplified by the use of evaluation metrics such as FID, which are designed to align with human visual perception rather than reflect distances in the original data space.
Therefore, we argue that practitioners are not aiming to recover the \emph{true} interpolation, but are instead implicitly designing for biases, through network architectures and training algorithms, that steer the model towards desirable generative behaviour. The success of the bias is then judged by how well the generated samples align with subjective or task-specific criteria.

This motivates the shift in perspective in this section: rather than assuming the existence of a ground truth manifold, we advocate for studying how inductive biases in these models influence how training data is interpolated---that is, how the model \textit{chooses} an interpolating manifold. We refer to this form of bias as the \textit{geometric bias} of the model. In the setting of diffusion models, understanding how geometric bias manifests requires not only understanding inductive bias at the stage of score matching, but also how that inductive bias is transformed into geometry by the sampling algorithm. By analysing this transformation, we can begin to understand which kinds of generalisation are possible, how inductive bias shapes the nature of samples produced, and how such biases can be purposefully engineered to match the needs of specific applications.

\subsection{Geometric bias of log-domain smoothing}\label{sec:geometric_bias_theory}

\begin{figure}[t]
    \centering
    \begin{subfigure}[b]{0.25\linewidth}
      \centering
      \caption*{$\sigma=0.1$}
      \includegraphics[width=0.9\textwidth]{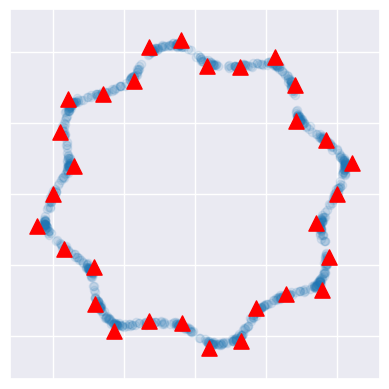}
    \end{subfigure}%
    \begin{subfigure}[b]{0.25\linewidth}
      \centering
      \caption*{$\sigma=0.25$}
      \includegraphics[width=0.9\textwidth]{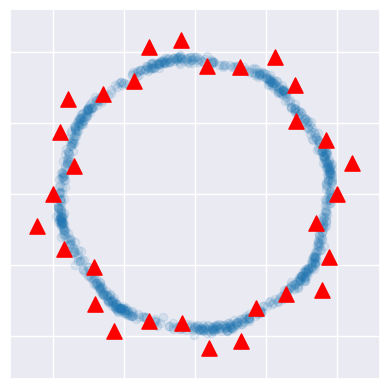}
    \end{subfigure}%
    \begin{subfigure}[b]{0.25\linewidth}
      \centering
      \caption*{$\sigma=1.0$}
      \includegraphics[width=0.9\textwidth]{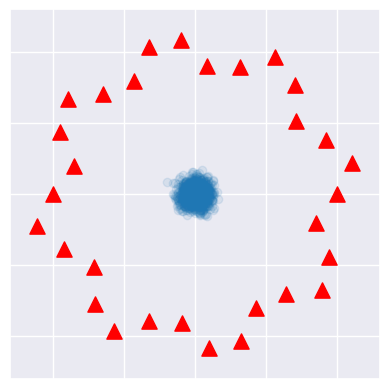}
    \end{subfigure}
    \caption{\textbf{Different smoothing kernels can isolate alternative manifolds, given the same training data.} Training data ($\textcolor{red}{\blacktriangle}$) against generated samples ($\textcolor{MidnightBlue}{\bullet}$) using isotropic Gaussian score smoothing. By changing the smoothing variance \(\sigma^2\), different geometries are realised.}
    \label{fig:choosing_different_manifolds}
\end{figure}

In this section, we consider this alternative perspective from the viewpoint of log-domain smoothing, inspecting how the smoothing kernel influences the choice of manifold that the resulting diffusion model is attracted to. We begin with some experiments in simple toy settings.

In Figure \ref{fig:toy_different_smoothing}, we investigate settings where the data (red triangles) could reasonably be interpolated by several different manifolds with the generated samples (blue) revealing how different choices of smoothing kernel influences the recovered manifold. The two figures on the left-hand side of Figure \ref{fig:toy_different_smoothing} consider a setting where data is sampled from a wavy circle. The plots demonstrate that when smoothing is adapted to level sets of the wavy structure, samples remain faithful to this geometry, whereas when smoothing aligns with level sets of the base circle, the samples conform to this simpler shape. These observations demonstrate that choosing a kernel that smooths the empirical score \emph{parallel} or \emph{tangentially} to a target manifold $\mathcal{M}$ induces a bias towards that manifold. On the right-hand side of Figure \ref{fig:toy_different_smoothing}, we show that by tailoring the kernel to favour particular geometric structures, the diffusion model can interpolate between drastically different geometries, altering both the dimension and the connectedness of the resulting manifold. In Figure \ref{fig:choosing_different_manifolds}, we once again consider the wavy circle example but with isotropic Gaussian smoothing, and identify how the choice of the manifold can be influenced also by the \emph{scale} of the smoothing. At small scales, the generated samples preserve the finer details of the wavy structure, while at larger scales the broader circular shape is recovered. Increasing the noise level further results in samples collapsing to the centre point, which can be considered an overly-simplified manifold representation of the training data.

\subsection{Geometry-adaptivity and geometric bias}\label{sec:theoretical-geometric-bias}
The theoretical analysis of Section \ref{sec:main_result} can also be extended to the present setting, allowing us to further elucidate the relationship between the smoothing kernel and geometric bias. In particular, we provide a modification of Theorem \ref{thm:density_ratio_bound_simplified} that quantifies how well log-domain smoothing adapts to different manifolds, without the requirement for the data to belong to that manifold.

We define the set of \textit{permissible manifolds} \(\mathbb{M}_\mu\) to be the set of all smooth compact submanifolds \(\M \subseteq \R^d\) with non-zero reach \(\tau_\M > 0\), satisfying the property \(c_\M := \operatorname{ess \, inf}_{\mud} p_{\mu, \M} > 0\) where \(p_{\mu, \M}\) denotes the density of \((\Pi_\M)_* \mud\) with respect to the volume measure on \(\M\). In other words, \(\mathbb{M}_\mu\) consists of all manifolds with bounded curvature, such that the projection of \(\mud\) onto \(\M\) has full support. Furthermore, given any \(\M \in \mathbb{M}_\mu\), we let \(d^*_\M\) denote the manifold dimension and we define,
\begin{equation*}
    K_\M^2 := \sup_{x \in \R^d} \E_{Y \sim k_x} \big [ | \dist(Y, \M) - \dist(x, \M) |^2 \big ], \qquad K_{\max, \M} := \| \dist(Y, \M) - \dist(x, \M) \|_{L^\infty}.
\end{equation*}
With this, we can state our second main result.

\begin{theorem}
\label{thm:geometric_bias}
Let \(\M \in \mathbb{M}_{\mu}\) and \(\Delta_\M := \dist(\{x_i\}_{i=1}^N, \M)\). Then, for any \(\delta \in (0, 1]\), whenever \(K_{\M, \max} + \gap \leq \tau_\M/96\) and \(N, \epsilon^{-1}\) is taken sufficiently large, we obtain with probability at least \(1 - \delta\) that,
\begin{equation}\label{eq:geometric_bias_00}
    D_2 \big ( \hat{p}_\epsilon^{k^{\M}} \big \| \hat{p}^k_\epsilon \big ) \lesssim \frac{K_\M (d_\M^* + 1)}{\tau_\M} + \frac{K_\M \gap}{\epsilon} + \frac{K_\M^2 \gap^2 d_\M^*}{\tau_\M \, \epsilon}.
\end{equation}
\end{theorem}

While Theorem \ref{thm:density_ratio_bound_simplified} controls how closely log-domain smooths adapts to an underlying true geometry, Theorem \ref{thm:geometric_bias} instead bounds how well \emph{any given} manifold describes the geometry produced by log-domain smoothing. When a manifold \(\M \in \mathbb{M}_\mu\) causes the right-hand side of this bound to be small, this indicates that smoothing with respect to a generic kernel \(k\) is similar to smoothing with respect to the kernel \(k^\M\) which is adapted to \(\M\). Indeed, any manifold \(\M\) that makes the right-hand side small is one that \emph{effectively captures the geometric bias of the smoothing kernel \(k\)}. With this, we have a strategy for identifying which manifolds log-domain smoothing will adapt to: by optimising the right-hand side with respect to \(\M \in \mathbb{M}_\mu\).

Furthermore, by analysing the effects of different parameters on this optimisation problem, we can begin to understand the logic that dictates the relationship between the smoothing kernel and the geometric bias of the resulting diffusion model. Indeed, the two terms of the bound in \eqref{eq:geometric_bias_00} represent a trade-off between decreasing the curvature \(\tau_\M^{-1}\), and decreasing the accuracy of the interpolating manifold \(\gap\), where the trade-off is dictated by the parameters, \(K_\M\), \(d_\M^*\) and \(\epsilon\). Since \(\epsilon\) is small, the second term can be large in magnitude and so optimising with respect to \(\M\) will prioritise manifolds with small \(\Delta_\M\), even if it comes at the expense of a larger \(\tau_\M\). However, as the scale of the smoothing grows, \(K_\M\) grows with it and so the term that is quadratic in \(K_\M\) will dominate. As a result, the optimisation of this bound will prioritise manifolds with lower curvature. This captures the phenomenon in Figure \ref{fig:choosing_different_manifolds} where larger smoothing scales results in the diffusion model selecting lower curvature geometries. Additionally, this analysis is consistent with the outcome of the experiment in Figure \ref{fig:toy_different_smoothing}. Indeed, if \(k_x\) prioritises particular directions in its smoothing, then choosing \(\M\) tangent to those directions will make \(K_\M\) small and so optimising \eqref{eq:geometric_bias_00} will produce the manifold tangent to those directions that minimises \(\gap\) and \(d^*_\M\). That is, a low-dimensional manifold that is tangent to the smoothing directions while best interpolating the data.

While our analysis agrees with our empirical findings, fully characterising the geometric bias of score function smoothing through this strategy requires the development of a matching lower bound. This  would require significant additional analysis and would likely require different techniques to the upper bound, so we leave this for future work.

\section{High-dimensional experiments}\label{sec:high-dim-experiments}

\begin{figure}[!t]
    \begin{minipage}[t]{0.55\textwidth}
        \centering
        \includegraphics[width=\textwidth]{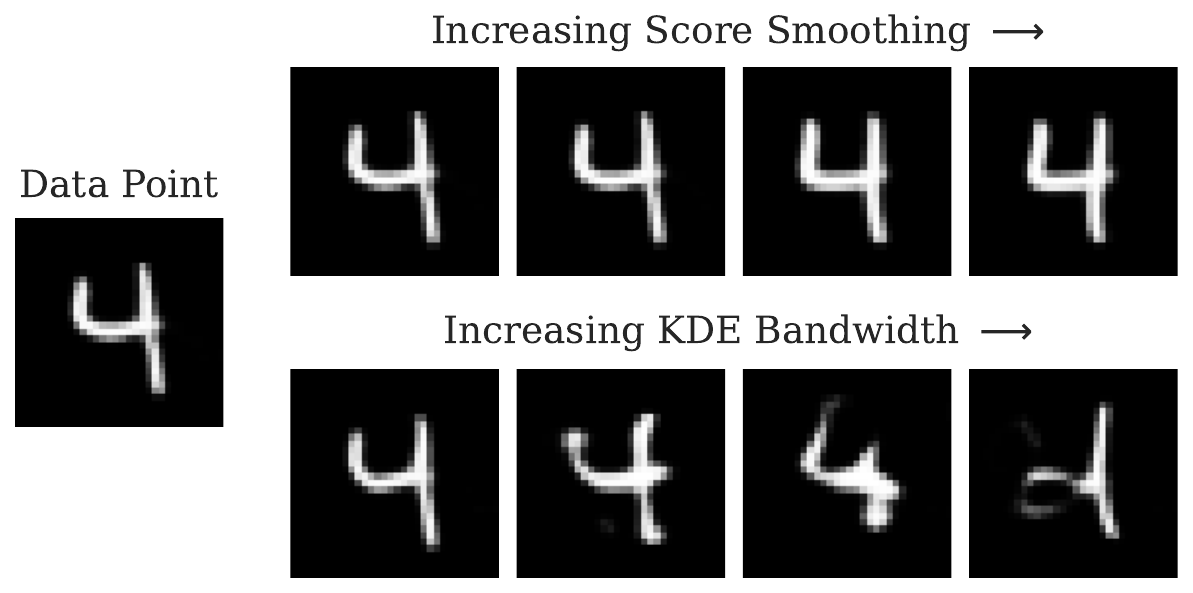}%
        \captionof{figure}{As smoothing level increases, generations from the score-smoothed diffusion model remain in the manifold structure. In contrast, samples from KDE quickly deviate from the manifold as the kernel scale increases, leading to poor reconstructions.}
        \label{fig:mnist_KDE_smoothing_plots_4s}
    \end{minipage}
    \hfill
    \begin{minipage}[t]{0.4\textwidth}
        \centering
        \includegraphics[width=0.85\textwidth]{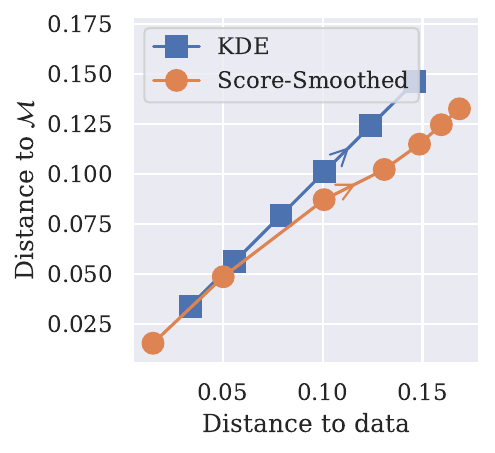}
        \captionof{figure}{Comparison of distance to closest point in training data and closest point in $\mathcal{M}$, for score-smoothed diffusion and KDE. Arrows indicate increasing amounts of smoothing.}
        \label{fig:mnist_KDE_smoothing_L2_4s}
    \end{minipage}%
\end{figure}

Up until this section, our experiments have focused on illustrative low-dimensional settings to complement our theoretical presentation. In this section, we consider higher-dimensional settings to investigate the extent to which the identified phenomena extend to scenarios more representative of practical diffusion model applications. Specifically, we focus on the setting of image synthesis.

\subsection{Generation in latent space}
\label{sec:KDE_vs_smooth_experiments}
We first present an experiment on the MNIST dataset that illustrates how score-function smoothing can preserve the geometry of the data, in contrast to density-level smoothing which quickly loses such structure.
In order to consider a well-structured manifold, we follow \citet{rombach2022_latentdm} and use a 32-dimensional VAE latent-space encoding of MNIST digits, and perform generation in the latent space. We consider $\mathcal{M}$ to be the set corresponding to the digit \(4\), which comprises a lower-dimensional structure in the latent space. This ground-truth manifold is approximated using all samples of the digit 4, from which we use a subset of 100 samples as our training dataset. We consider a smoothed-score diffusion model using an isotropic Gaussian kernel, and compare with kernel density estimation which corresponds to smoothing at the density level. For details regarding the experimental setup, see Appendix \ref{app:KDE_vs_smooth_experiments}.

We wish to assess how well the manifold structure is preserved and therefore how well generated samples decode to resemble the digit \(4\).
On the top row of \cref{fig:mnist_KDE_smoothing_plots_4s}, we display samples from a score-smoothed diffusion model as the scale of smoothing increases. For small scales of smoothing, the samples perfectly recover one of the training examples (plotted on the left). As the amount of smoothing increases, the samples become novel images that are not present in the dataset yet nonetheless decode to resemble samples from the class of \(4\)'s, suggesting that they remain close to the underlying geometry. The bottom row displays KDE samples centred on the same training example at different scales of smoothing. We see that the quality of the reconstructed quickly deteriorates, indicating that in order to induce the same degree of novelty and difference from the training set, the KDE moves significantly further off-manifold, thereby failing to preserve the geometric structure of the data.

We also provide a quantitative assessment of this behaviour in \cref{fig:mnist_KDE_smoothing_L2_4s}, which plots the average distance to the manifold against the average distance to the training dataset. As the KDE smoothing kernel increases in scale, the distance to the manifold increases identically with the distance to the data. By contrast, as smoothing increases the score-smoothed diffusion samples sometimes become closer to other points outside of the training dataset that are still within the class of \(4\)'s, indicating that mass is being spread along the manifold structure.

\subsection{Generation in pixel space}

While image generation in latent space is common, diffusion models are also known to succeed when learning in pixel space directly. Thus, in this section, we consider the effect of different types of smoothing in this setting. As working in pixel space is more challenging, we will focus on simple one-dimensional image manifolds.

Now that we are operating in the pixel space, the training datapoints are further apart and thus there are many permissible manifolds that interpolate the data. In \Cref{sec:geometric_bias}, we considered how the \textit{type} of smoothing can induce different structures in the generated samples, and illustrated this effect with low-dimensional experiments. In this section, we assess to what extent this intuition transfers to higher-dimensional settings by also considering smoothing kernels adapted to the manifold structure, and seeing whether this can influence the geometric structure of the sampling distribution.

\subsubsection{Synthetic image manifold}
\label{sec:synthetic image manifold}

We begin with a synthetic image setting in which the data is constructed to lie on a 1-dimensional manifold produced by a closed curve $\phi : [0, 2\pi) \to \R^{\text{64} \times \text{64}}$. The curve maps each angle to an image of a bump function centred at the corresponding angle around a circle. This construction yields a closed manifold in pixel space, and enables precise control over the geometric structure whilst retaining key properties of image-like data.  A visualisation showing the traversal of the manifold is provided in \cref{fig:gaussian_bump_manifold}. The training dataset is then constructed using 16 equidistant points along the synthetic manifold.
For full details of the experimental setup, see Appendix \ref{app:synthetic_image}.

\begin{figure}[t]
    \centering
    \includegraphics[width=0.52\textwidth]{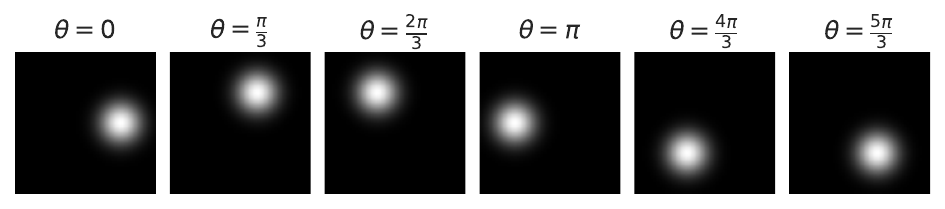}
    \caption{Visualisation of traversing the synthetic image manifold.}
    \label{fig:gaussian_bump_manifold}
\end{figure}

In \Cref{fig:synth_images_smoothing_comparison}, we again compare isotropic Gaussian score function smoothing to the KDE by examining the average distances to the manifold relative to the empirical dataset, for different scales of smoothing. 
For the KDE, we again observe that the distance to the manifold increases identically with the distance to the data, while the Gaussian score-smoothing exhibits a slight downwards curve.

We additionally consider a smoothed-score diffusion model using a \textit{manifold-adapted} kernel, constructed by smoothing along the a translated version of the manifold $\mathcal{M}$ passing through the current point. Importantly, the sampling process incorporates knowledge of the true manifold only through the smoothing mechanism---the empirical score function that is being smoothed only contains information about the training dataset. We see in \Cref{fig:synth_images_smoothing_comparison} that the manifold-adapted smoothing generates samples significantly closer to the true manifold compared to simple isotropic Gaussian smoothing---this illustrates that the \textit{type} of smoothing can influence the geometric structure of the generated samples, supporting the discussion in \Cref{sec:geometric_bias}. We include additional plots that further elucidate this effect in the Appendix \ref{app:synthetic_image}.

We also observe in \cref{fig:circle_images_curvature} that changing the width of the bump function, which we denote by $\eta$, influences the degree to which this effect occurs. For larger $\eta$, manifold-adapted smoothing is able to effectively generate samples close to the manifold, but far from training data, avoiding direct memorization of training data. However, as $\eta$ decreases, the generated samples deviate further from the manifold. We anticipate that this reflects our theoretical results which suggest that higher manifold curvature makes it less likely that samples remain close to the manifold.

\subsubsection{Image manifolds}
\label{sec:image manifold}
We now return to the MNIST dataset, and consider generations in pixel space. Following a similar setup to above, we define a 1-dimensional image manifold given by a curve $\phi : [0,1] \to \R^{\text{32} \times \text{32}}$. This is constructed by interpolating between three datapoints from the same class, drawing a triangle between their latent representations using the same VAE as in \cref{sec:KDE_vs_smooth_experiments}, and decoding this triangular latent path to obtain the closed loop $\phi(t)$ in pixel space. We take 10 points along this curve to form the dataset and, as above, we consider Gaussian and manifold-adapted smoothing kernels. We emphasise that the diffusion procedure takes place in \textit{pixel space}; the VAE is used only to define the manifold structure.

To quantitively assess the generated samples, in \cref{fig:mnist_L2_4s} we again plot distance to the manifold $\mathcal{M}$ versus distance to the training dataset. The same pattern as before is observable in the $L_2$ distances as previously; samples generated using manifold-adaptive smoothing remain comparatively closer to the manifold compared to those generated using Gaussian smoothing. As a measure of generalisation and visual sample quality, we also report the Fréchet inception distance (FID) \citep{heusel17_FID} of samples compared to a held-out test set. We plot this in \cref{fig:mnist_fid_4s} for both smoothing kernels, demonstrating a generalisation benefit to score function smoothing. 
The Gaussian smoothing is known to produce barycentres of training datapoints \citep{scarvelis2023closedformdiffusionmodels}, which is not well-suited for generations in pixel space as it results in `blurring' as the smoothing level increases. This is observed in a steep increase in the FID value for large enough smoothing. In contrast, the manifold-adapted smoothing tends to suffer less from this blurring phenomenon and therefore avoids the same steep increase in FID value. We include full details, and additional plots for different curves $\phi$, in \cref{app:MNIST_compare_smoothing}.

\begin{figure}[!t]
    \noindent
    \begin{minipage}[t]{0.48\textwidth}
        \begin{figure}[H]
            \centering
            \begin{subfigure}[b]{0.48\linewidth}
                \includegraphics[width=\textwidth]{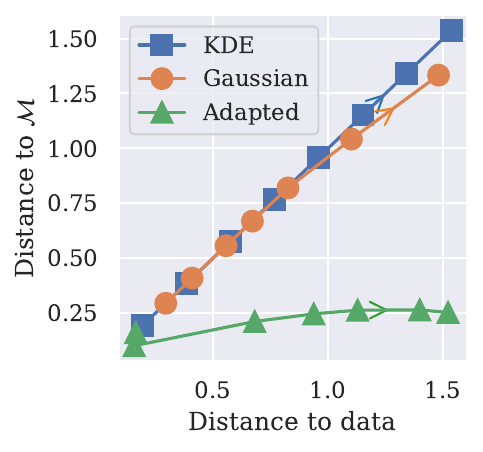}
                \caption{Comparing $L_2$ distance to data and $\mathcal{M}$.}
                \label{fig:synth_images_smoothing_comparison}
            \end{subfigure}
            \vspace{0.5em}
            \begin{subfigure}[b]{0.47\linewidth}
                \includegraphics[width=\textwidth]{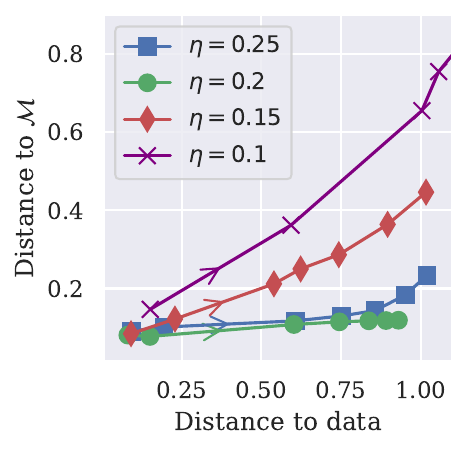}
                \caption{Effect of bump width $\eta$, for adapted smoothing.}
                \label{fig:circle_images_curvature}
            \end{subfigure}
            \caption{Comparison of Gaussian and manifold-adapted smoothing kernels for the synthetic image manifold. Arrows indicate increasing smoothing.}
            \label{fig:circle_images_plots}
        \end{figure}
    \end{minipage}%
    \hfill
    \begin{minipage}[t]{0.48\textwidth}
        \begin{figure}[H]
            \centering
            \begin{subfigure}[b]{0.48\linewidth}
                \vspace{-2em}
                \includegraphics[width=\textwidth]{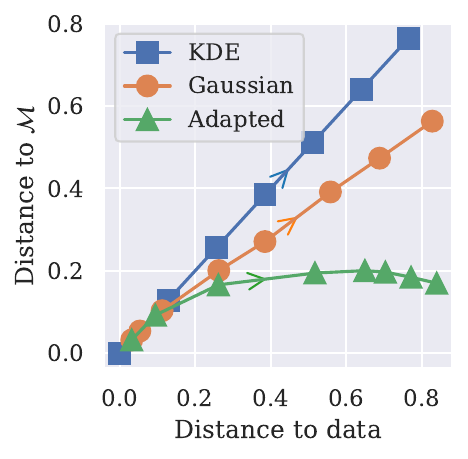}
                \caption{Comparing $L_2$ distance to data and $\mathcal{M}$.}
                \label{fig:mnist_L2_4s}
            \end{subfigure}
            \hfill
            \begin{subfigure}[b]{0.48\linewidth}
                \includegraphics[width=\textwidth]{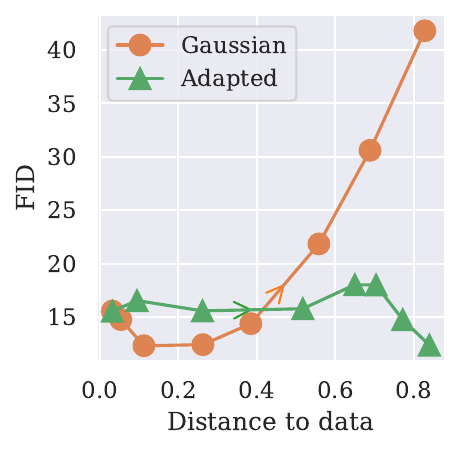}
                \caption{FID of generated samples.}
                \label{fig:mnist_fid_4s}
            \end{subfigure}
            \caption{Comparison of Gaussian and manifold-adapted smoothing kernels for the MNIST image manifold.}
            \label{fig:mnist_experiments}
        \end{figure}
    \end{minipage}
    \vspace{-1em}
\end{figure}

\section{Discussion}

\subsection{Related work}

In the manifold setting, \citet{pidstrigach2022scorebased, de_bortoli2022convergence} provide precise convergence bounds for diffusion models. Recently, there has been a surge of interest in providing refined results in the manifold setting \citep{oko23_minimax, chen23_scoreapproxlowdim, tang24_adaptmanifold,  li2024adapting, azangulov2024convergencediffusionmodelsmanifold, potaptchik2024linearconvergencediffusionmodels}, or under other specific structural settings \citep{shah2023learning, chen2024learninggeneralgaussianmixtures, wang2024diffusionmodelslearnlowdimensional}.
Additionally, many works have focused on empirically validating the manifold hypothesis for data such as images \citep{Fefferman16_testingmanifoldhypothesis, pope21_intrinsicdim, stanczuk2024diffusion, brown2023verifying, kamkari24_flipd}. Our work also shares similarities with wider literature regarding how manifold structure interacts with learning tasks, such as \citet{genovese12_minimaxmanifold,Cheng13_locallinearregression,moscovich17a_minimax,Gao222_adaptivemanifold}.

Recently, there has been an increased interest in understanding generalisation and memorisation in diffusion models. Memorisation of training data has been observed empirically by \citet{somepalli2023_diffusionforgery, Carlini23_extract_training} when the capacity of the network is large relative to the number of training samples. Other works investigate how inductive biases of neural network architectures aid in generalisation \citep{kadkhodaie2024generalization, niedoba2024mechanisticexplanationdiffusionmodel, kamb2024_creativity}. The recent work of \citet{vastola2025generalization} examines the role of noise in the objective. The dichotomy between generalisation and memorisation has been investigated in \citet{yoon2023diffusion, gu2023memorizationdiffusionmodels, wen2024detecting, zhang2024_emergence, baptista2025memorization}. The works of \citet{raya2023spontaneous, Biroli_2024, ventura2025manifoldsrandommatricesspectral} examine the roles of distinct regimes in the generative process.

This work contributes to a small but growing line of research into the effect of score-smoothing in diffusion models.
\citet{scarvelis2023closedformdiffusionmodels} previously studied isotropic Gaussian and Gumbel smoothing of the score function, as a training-free alternative for running diffusion models, and show that this generates barycentres of training datapoints. \citet{chen2025interpolationeffectscoresmoothing} investigates the effect of score smoothing on generalisation in the 1$d$ linear setting.
Concurrently, the work of \citet{gabriel2025kernelsmoothedscoresdenoisingdiffusion} also studies the effect of a kernel-smoothed score function and its relation to preserving manifold structure, though with different analysis techniques.

\subsection{Limitations and further investigations}

Our argument that the log-domain smoothed measure \(\hat{p}^k_{\epsilon}\) approximates the output of the diffusion model with score smoothing relies on the exchangeability of gradients (see \eqref{eq:gradexch}), a property that holds for location-independent kernels. Extending our framework to the more general case of location-dependent kernels is an important next step. Similarly, our key theorems (Theorem \ref{thm:density_ratio_bound_simplified} and \ref{thm:geometric_bias}) currently require the scale of noise normal to the manifold, \(K\), to be small relative to its curvature, \(\tau\). A more complete characterisation of the geometric bias would therefore require relaxing this assumption. Furthermore, a more robust characterisation of the geometric bias of log-domain smoothing would also benefit from a matching lower bound on the R\'enyi divergence. While we have demonstrated through heuristic arguments how our results show the generalisation potential of log-domain smoothing, our work stops short of deriving a formal generalisation error bound. Whether log-domain smoothing alone could produce optimal generalisation bounds is a question that we leave open for future investigation. Finally, we believe this theoretical framework could be a valuable tool for the related literature analysing the memorisation and privacy properties of diffusion models.

The experiments presented here illustrate that the \textit{type} of smoothing can influence the geometric structure of the generated samples. 
However, they also highlight the challenges posed by high-curvature manifolds, where smoothing the empirical score alone may not suffice.
Real-world image manifolds tend to be highly curved, yet diffusion models still generalize well from relatively few training samples \citep{kadkhodaie2024generalization}. Furthermore, in practical settings we do not have knowledge of the ground-truth manifold structure to explicitly apply manifold-adaptive smoothing. If such adaptation occurs, it must arise implicitly from the model's inductive biases. This suggests that other factors must also be at play, such as biases induced by neural architectural choices like convolutions and attention \citep{kamb2024_creativity}. We do not examine the behaviour of such practical architectures in this work, but remark that understanding to what extent architectural designs choices interact with the ideas presented here is an interesting direction for future study.

\subsection{Conclusion}

In this work, we have investigated how the implicit regularisation caused by smoothing of the empirical score function interacts with manifold structure in the data. In particular, we identify that smoothing at the level of the log-domain is implicitly geometry-adaptive, behaving similarly to a manifold-adapted kernel when given enough samples. 
Beyond the data-rich setting, we observe that the \textit{choice} of smoothing kernel can induce different structures in the generated distribution. Future work could investigate how inductive biases due to neural architectural designs influence the types of smoothing that occur in practice.

\section*{Acknowledgements}
Tyler Farghly was supported by Engineering and Physical Sciences Research Council (EPSRC) [grant number EP/T517811/1] and by the DeepMind scholarship.
Peter Potaptchik and Samuel Howard are supported by the EPSRC CDT in Modern Statistics and Statistical Machine Learning [grant number EP/S023151/1].
George Deligiannidis and Jakiw Pidstrigach acknowledge support from EPSRC [grant number EP/Y018273/1]. The authors would like to thank Iskander Azangulov and Arya Akhavan for stimulating discussions and Ioannis Siglidis for valuable feedback.

\bibliography{references}

\begin{thebibliography}{79}
\providecommand{\natexlab}[1]{#1}
\providecommand{\url}[1]{\texttt{#1}}
\expandafter\ifx\csname urlstyle\endcsname\relax
  \providecommand{\doi}[1]{doi: #1}\else
  \providecommand{\doi}{doi: \begingroup \urlstyle{rm}\Url}\fi

\bibitem[Aamari(2017)]{Aamari2017-he}
E.~Aamari.
\newblock \emph{Convergence Rates for Geometric Inference}.
\newblock PhD thesis, Université Paris-Saclay, Sept. 2017.

\bibitem[Aamari et~al.(2019)Aamari, Kim, Chazal, Michel, Rinaldo, and Wasserman]{Aamari2019-gk}
E.~Aamari, J.~Kim, F.~Chazal, B.~Michel, A.~Rinaldo, and L.~Wasserman.
\newblock Estimating the reach of a manifold.
\newblock \emph{Electronic Journal of Statistics}, 13\penalty0 (1):\penalty0 1359--1399, 2019.

\bibitem[Azangulov et~al.(2024)Azangulov, Deligiannidis, and Rousseau]{azangulov2024convergencediffusionmodelsmanifold}
I.~Azangulov, G.~Deligiannidis, and J.~Rousseau.
\newblock Convergence of diffusion models under the manifold hypothesis in high-dimensions.
\newblock \emph{arXiv [stat.ML]}, 2024.

\bibitem[Baptista et~al.(2025)Baptista, Dasgupta, Kovachki, Oberai, and Stuart]{baptista2025memorization}
R.~Baptista, A.~Dasgupta, N.~B. Kovachki, A.~Oberai, and A.~M. Stuart.
\newblock Memorization and regularization in generative diffusion models.
\newblock \emph{arXiv [cs.LG]}, 2025.

\bibitem[Bengio et~al.(2012)Bengio, Courville, and Vincent]{Bengio2012RepresentationLA}
Y.~Bengio, A.~C. Courville, and P.~Vincent.
\newblock Representation learning: A review and new perspectives.
\newblock \emph{IEEE Transactions on Pattern Analysis and Machine Intelligence}, 35:\penalty0 1798--1828, 2012.

\bibitem[Biroli et~al.(2024)Biroli, Bonnaire, de~Bortoli, and Mézard]{Biroli_2024}
G.~Biroli, T.~Bonnaire, V.~de~Bortoli, and M.~Mézard.
\newblock Dynamical regimes of diffusion models.
\newblock \emph{Nature Communications}, 15\penalty0 (1), Nov. 2024.
\newblock ISSN 2041-1723.
\newblock \doi{10.1038/s41467-024-54281-3}.

\bibitem[Blattmann et~al.(2023)Blattmann, Rombach, Ling, Dockhorn, Kim, Fidler, and Kreis]{Blattmann_2023_CVPR}
A.~Blattmann, R.~Rombach, H.~Ling, T.~Dockhorn, S.~W. Kim, S.~Fidler, and K.~Kreis.
\newblock Align your latents: High-resolution video synthesis with latent diffusion models.
\newblock In \emph{Proceedings of the IEEE/CVF Conference on Computer Vision and Pattern Recognition (CVPR)}, pages 22563--22575, 2023.

\bibitem[Bobkov(2003)]{Bobkov2003-et}
S.~G. Bobkov.
\newblock Spectral gap and concentration for some spherically symmetric probability measures.
\newblock In \emph{Lecture Notes in Mathematics}, Lecture notes in mathematics, pages 37--43. Springer Berlin Heidelberg, Berlin, Heidelberg, 2003.

\bibitem[Brown et~al.(2023)Brown, Caterini, Ross, Cresswell, and Loaiza-Ganem]{brown2023verifying}
B.~C. Brown, A.~L. Caterini, B.~L. Ross, J.~C. Cresswell, and G.~Loaiza-Ganem.
\newblock Verifying the union of manifolds hypothesis for image data.
\newblock In \emph{The Eleventh International Conference on Learning Representations}, 2023.

\bibitem[Carlini et~al.(2023)Carlini, Hayes, Nasr, Jagielski, Sehwag, Tram{\`e}r, Balle, Ippolito, and Wallace]{Carlini23_extract_training}
N.~Carlini, J.~Hayes, M.~Nasr, M.~Jagielski, V.~Sehwag, F.~Tram{\`e}r, B.~Balle, D.~Ippolito, and E.~Wallace.
\newblock Extracting training data from diffusion models.
\newblock In \emph{32nd USENIX Security Symposium (USENIX Security 23)}, pages 5253--5270, Anaheim, CA, Aug. 2023. USENIX Association.
\newblock ISBN 978-1-939133-37-3.

\bibitem[Chen et~al.(2023)Chen, Huang, Zhao, and Wang]{chen23_scoreapproxlowdim}
M.~Chen, K.~Huang, T.~Zhao, and M.~Wang.
\newblock Score approximation, estimation and distribution recovery of diffusion models on low-dimensional data.
\newblock In \emph{Proceedings of the 40th International Conference on Machine Learning}, 2023.

\bibitem[Chen et~al.(2018)Chen, Rubanova, Bettencourt, and Duvenaud]{chen18_neuralodes}
R.~T.~Q. Chen, Y.~Rubanova, J.~Bettencourt, and D.~K. Duvenaud.
\newblock Neural ordinary differential equations.
\newblock In \emph{Advances in Neural Information Processing Systems}, 2018.

\bibitem[Chen et~al.(2024)Chen, Kontonis, and Shah]{chen2024learninggeneralgaussianmixtures}
S.~Chen, V.~Kontonis, and K.~Shah.
\newblock Learning general gaussian mixtures with efficient score matching.
\newblock \emph{arXiv [cs.DS]}, 2024.

\bibitem[Chen(2025)]{chen2025interpolationeffectscoresmoothing}
Z.~Chen.
\newblock On the interpolation effect of score smoothing.
\newblock \emph{arXiv [cs.LG]}, 2025.

\bibitem[Cheng and Wu(2013)]{Cheng13_locallinearregression}
M.-y. Cheng and H.-t. Wu.
\newblock Local linear regression on manifolds and its geometric interpretation.
\newblock \emph{Journal of the American Statistical Association}, 108\penalty0 (504):\penalty0 1421--1434, 2013.
\newblock \doi{10.1080/01621459.2013.827984}.

\bibitem[Chewi et~al.(2022)Chewi, Erdogdu, Li, Shen, and Zhang]{Chewi2022-of}
S.~Chewi, M.~A. Erdogdu, M.~Li, R.~Shen, and S.~Zhang.
\newblock Analysis of {L}angevin {M}onte {C}arlo from {P}oincare to log-{S}obolev.
\newblock In \emph{Proceedings of Thirty Fifth Conference on Learning Theory}, 2022.

\bibitem[{De Bortoli}(2022)]{de_bortoli2022convergence}
V.~{De Bortoli}.
\newblock Convergence of denoising diffusion models under the manifold hypothesis.
\newblock \emph{Transactions on Machine Learning Research}, 2022.

\bibitem[Dhariwal and Nichol(2021)]{dhariwal2021diffusion}
P.~Dhariwal and A.~Q. Nichol.
\newblock Diffusion models beat {GAN}s on image synthesis.
\newblock In \emph{Advances in Neural Information Processing Systems}, 2021.

\bibitem[Erdogdu et~al.(2022)Erdogdu, Hosseinzadeh, and Zhang]{Erdogdu2022-sp}
M.~A. Erdogdu, R.~Hosseinzadeh, and S.~Zhang.
\newblock Convergence of {L}angevin {M}onte {C}arlo in {C}hi-squared and {R}ényi divergence.
\newblock In \emph{Proceedings of The 25th International Conference on Artificial Intelligence and Statistics}, 2022.

\bibitem[Farghly et~al.(2025)Farghly, Rebeschini, Deligiannidis, and Doucet]{Farghly2025-fn}
T.~Farghly, P.~Rebeschini, G.~Deligiannidis, and A.~Doucet.
\newblock Implicit regularisation in diffusion models: An algorithm-dependent generalisation analysis.
\newblock \emph{arXiv [stat.ML]}, 2025.

\bibitem[Fefferman et~al.(2016)Fefferman, Mitter, and Narayanan]{Fefferman16_testingmanifoldhypothesis}
C.~Fefferman, S.~Mitter, and H.~Narayanan.
\newblock Testing the manifold hypothesis.
\newblock \emph{Journal of the American Mathematical Society}, 29\penalty0 (4):\penalty0 983--1049, oct 2016.
\newblock ISSN 0894-0347.
\newblock \doi{10.1090/jams/852}.

\bibitem[Gabriel et~al.(2025)Gabriel, Ged, Veiga, and Schertzer]{gabriel2025kernelsmoothedscoresdenoisingdiffusion}
F.~Gabriel, F.~Ged, M.~H. Veiga, and E.~Schertzer.
\newblock Kernel-smoothed scores for denoising diffusion: A bias-variance study.
\newblock \emph{arXiv [cs.LG]}, 2025.

\bibitem[Gao et~al.(2022)Gao, Jiang, and Qian]{Gao222_adaptivemanifold}
J.-X. Gao, D.-Q. Jiang, and M.-P. Qian.
\newblock Adaptive manifold density estimation.
\newblock \emph{Journal of Statistical Computation and Simulation}, 92\penalty0 (11):\penalty0 2317--2331, 2022.
\newblock \doi{10.1080/00949655.2022.2028283}.

\bibitem[Genovese et~al.(2012)Genovese, Perone-Pacifico, Verdinelli, and Wasserman]{genovese12_minimaxmanifold}
C.~Genovese, M.~Perone-Pacifico, I.~Verdinelli, and L.~Wasserman.
\newblock Minimax manifold estimation.
\newblock \emph{Journal of Machine Learning Research}, 13\penalty0 (43):\penalty0 1263--1291, 2012.

\bibitem[Goodfellow et~al.(2016)Goodfellow, Bengio, and Courville]{Goodfellow16_DL}
I.~Goodfellow, Y.~Bengio, and A.~Courville.
\newblock \emph{Deep Learning}.
\newblock MIT Press, 2016.

\bibitem[Gray(2004)]{Gray2004-rz}
A.~Gray.
\newblock \emph{Tubes}.
\newblock Birkhäuser Basel, Basel, 2004.

\bibitem[Gu et~al.(2023)Gu, Du, Pang, Li, Lin, and Wang]{gu2023memorizationdiffusionmodels}
X.~Gu, C.~Du, T.~Pang, C.~Li, M.~Lin, and Y.~Wang.
\newblock On memorization in diffusion models.
\newblock \emph{arXiv [cs.LG]}, 2023.

\bibitem[Haussmann and Pardoux(1986)]{haussmann1986time}
U.~G. Haussmann and E.~Pardoux.
\newblock Time reversal of diffusions.
\newblock \emph{The Annals of Probability}, pages 1188--1205, 1986.

\bibitem[Heusel et~al.(2017)Heusel, Ramsauer, Unterthiner, Nessler, and Hochreiter]{heusel17_FID}
M.~Heusel, H.~Ramsauer, T.~Unterthiner, B.~Nessler, and S.~Hochreiter.
\newblock Gans trained by a two time-scale update rule converge to a local {N}ash equilibrium.
\newblock In \emph{Advances in Neural Information Processing Systems}, 2017.

\bibitem[Ho et~al.(2020)Ho, Jain, and Abbeel]{Ho202_DDPM}
J.~Ho, A.~Jain, and P.~Abbeel.
\newblock Denoising diffusion probabilistic models.
\newblock In \emph{Advances in Neural Information Processing Systems}, 2020.

\bibitem[Ho et~al.(2022)Ho, Chan, Saharia, Whang, Gao, Gritsenko, Kingma, Poole, Norouzi, Fleet, and Salimans]{ho2022imagenvideo}
J.~Ho, W.~Chan, C.~Saharia, J.~Whang, R.~Gao, A.~Gritsenko, D.~P. Kingma, B.~Poole, M.~Norouzi, D.~J. Fleet, and T.~Salimans.
\newblock Imagen video: High definition video generation with diffusion models.
\newblock \emph{arXiv [cs.CV]}, 2022.

\bibitem[Hyv{{\"a}}rinen(2005)]{hyvarinen05_score_matching}
A.~Hyv{{\"a}}rinen.
\newblock Estimation of non-normalized statistical models by score matching.
\newblock \emph{Journal of Machine Learning Research}, 6\penalty0 (24):\penalty0 695--709, 2005.

\bibitem[Kadkhodaie et~al.(2024)Kadkhodaie, Guth, Simoncelli, and Mallat]{kadkhodaie2024generalization}
Z.~Kadkhodaie, F.~Guth, E.~P. Simoncelli, and S.~Mallat.
\newblock Generalization in diffusion models arises from geometry-adaptive harmonic representation.
\newblock In \emph{The Twelfth International Conference on Learning Representations}, 2024.

\bibitem[Kamb and Ganguli(2025)]{kamb2024_creativity}
M.~Kamb and S.~Ganguli.
\newblock An analytic theory of creativity in convolutional diffusion models.
\newblock In \emph{Forty-second International Conference on Machine Learning}, 2025.

\bibitem[Kamkari et~al.(2024)Kamkari, Ross, Hosseinzadeh, Cresswell, and Loaiza-Ganem]{kamkari24_flipd}
H.~Kamkari, B.~L. Ross, R.~Hosseinzadeh, J.~C. Cresswell, and G.~Loaiza-Ganem.
\newblock A geometric view of data complexity: Efficient local intrinsic dimension estimation with diffusion models.
\newblock In \emph{The Thirty-eighth Annual Conference on Neural Information Processing Systems}, 2024.

\bibitem[Karras et~al.(2022)Karras, Aittala, Aila, and Laine]{karras22_edm}
T.~Karras, M.~Aittala, T.~Aila, and S.~Laine.
\newblock Elucidating the design space of diffusion-based generative models.
\newblock In \emph{Advances in Neural Information Processing Systems}, 2022.

\bibitem[Kingma and Ba(2015)]{kingma15_adam}
D.~P. Kingma and J.~Ba.
\newblock Adam: A method for stochastic optimization.
\newblock In \emph{International Conference on Learning Representations (ICLR)}, 2015.

\bibitem[Kingma and Welling(2014)]{kingma2022autoencodingvariationalbayes}
D.~P. Kingma and M.~Welling.
\newblock Auto-encoding variational bayes.
\newblock In \emph{International Conference on Learning Representations}, 2014.

\bibitem[Kong et~al.(2021)Kong, Ping, Huang, Zhao, and Catanzaro]{kong2021diffwave}
Z.~Kong, W.~Ping, J.~Huang, K.~Zhao, and B.~Catanzaro.
\newblock Diffwave: A versatile diffusion model for audio synthesis.
\newblock In \emph{International Conference on Learning Representations}, 2021.

\bibitem[LeCun et~al.(2010)LeCun, Cortes, and Burges]{lecun2010mnist}
Y.~LeCun, C.~Cortes, and C.~Burges.
\newblock Mnist handwritten digit database.
\newblock \emph{ATT Labs [Online]. Available: http://yann.lecun.com/exdb/mnist}, 2, 2010.

\bibitem[Li and Yan(2024)]{li2024adapting}
G.~Li and Y.~Yan.
\newblock Adapting to unknown low-dimensional structures in score-based diffusion models.
\newblock In \emph{The Thirty-eighth Annual Conference on Neural Information Processing Systems}, 2024.

\bibitem[Liu et~al.(2023)Liu, Chen, Yuan, Mei, Liu, Mandic, Wang, and Plumbley]{liu2023_audio_ldm}
H.~Liu, Z.~Chen, Y.~Yuan, X.~Mei, X.~Liu, D.~Mandic, W.~Wang, and M.~D. Plumbley.
\newblock {A}udio{LDM}: Text-to-audio generation with latent diffusion models.
\newblock In \emph{Proceedings of the 40th International Conference on Machine Learning}, 2023.

\bibitem[Loaiza-Ganem et~al.(2024)Loaiza-Ganem, Ross, Hosseinzadeh, Caterini, and Cresswell]{loaiza-ganem2024deep}
G.~Loaiza-Ganem, B.~L. Ross, R.~Hosseinzadeh, A.~L. Caterini, and J.~C. Cresswell.
\newblock Deep generative models through the lens of the manifold hypothesis: A survey and new connections.
\newblock \emph{Transactions on Machine Learning Research}, 2024.

\bibitem[Ma and Ying(2021)]{ma2021linear}
C.~Ma and L.~Ying.
\newblock On linear stability of {SGD} and input-smoothness of neural networks.
\newblock \emph{Advances in Neural Information Processing Systems}, 2021.

\bibitem[Mironov(2017)]{Mironov2017-vf}
I.~Mironov.
\newblock Rényi differential privacy.
\newblock In \emph{2017 IEEE 30th Computer Security Foundations Symposium (CSF)}. IEEE, Aug. 2017.

\bibitem[Moscovich et~al.(2017)Moscovich, Jaffe, and Boaz]{moscovich17a_minimax}
A.~Moscovich, A.~Jaffe, and N.~Boaz.
\newblock {Minimax-optimal semi-supervised regression on unknown manifolds}.
\newblock In \emph{Proceedings of the 20th International Conference on Artificial Intelligence and Statistics}, 2017.

\bibitem[Mousavi-Hosseini et~al.(2023)Mousavi-Hosseini, Farghly, He, Balasubramanian, and Erdogdu]{Mousavi-Hosseini2023-xs}
A.~Mousavi-Hosseini, T.~K. Farghly, Y.~He, K.~Balasubramanian, and M.~A. Erdogdu.
\newblock Towards a complete analysis of langevin monte carlo: Beyond poincaré inequality.
\newblock In G.~Neu and L.~Rosasco, editors, \emph{Proceedings of Thirty Sixth Conference on Learning Theory}, volume 195 of \emph{Proceedings of Machine Learning Research}, pages 1--35. PMLR, 2023.

\bibitem[Mulayoff et~al.(2021)Mulayoff, Michaeli, and Soudry]{mulayoff2021implicit}
R.~Mulayoff, T.~Michaeli, and D.~Soudry.
\newblock The implicit bias of minima stability: A view from function space.
\newblock \emph{Advances in Neural Information Processing Systems}, 2021.

\bibitem[Niedoba et~al.(2025)Niedoba, Zwartsenberg, Murphy, and Wood]{niedoba2024mechanisticexplanationdiffusionmodel}
M.~Niedoba, B.~Zwartsenberg, K.~P. Murphy, and F.~Wood.
\newblock Towards a mechanistic explanation of diffusion model generalization.
\newblock In \emph{Forty-second International Conference on Machine Learning}, 2025.

\bibitem[Oko et~al.(2023)Oko, Akiyama, and Suzuki]{oko23_minimax}
K.~Oko, S.~Akiyama, and T.~Suzuki.
\newblock Diffusion models are minimax optimal distribution estimators.
\newblock In \emph{Proceedings of the 40th International Conference on Machine Learning}, 2023.

\bibitem[Pavliotis(2014)]{Pavliotis2014-dc}
G.~A. Pavliotis.
\newblock \emph{Stochastic Processes and Applications: Diffusion Processes, the Fokker-Planck and Langevin Equations}.
\newblock Springer, Nov. 2014.

\bibitem[Pidstrigach(2022)]{pidstrigach2022scorebased}
J.~Pidstrigach.
\newblock Score-based generative models detect manifolds.
\newblock In \emph{Advances in Neural Information Processing Systems}, 2022.

\bibitem[Pope et~al.(2021)Pope, Zhu, Abdelkader, Goldblum, and Goldstein]{pope21_intrinsicdim}
P.~Pope, C.~Zhu, A.~Abdelkader, M.~Goldblum, and T.~Goldstein.
\newblock The intrinsic dimension of images and its impact on learning.
\newblock In \emph{International Conference on Learning Representations}, 2021.

\bibitem[Potaptchik et~al.(2025)Potaptchik, Azangulov, and Deligiannidis]{potaptchik2024linearconvergencediffusionmodels}
P.~Potaptchik, I.~Azangulov, and G.~Deligiannidis.
\newblock Linear convergence of diffusion models under the manifold hypothesis.
\newblock In \emph{The Thirty Eighth Annual Conference on Learning Theory}, 2025.

\bibitem[Rahaman et~al.(2019)Rahaman, Baratin, Arpit, Draxler, Lin, Hamprecht, Bengio, and Courville]{rahaman_spectral_bias}
N.~Rahaman, A.~Baratin, D.~Arpit, F.~Draxler, M.~Lin, F.~Hamprecht, Y.~Bengio, and A.~Courville.
\newblock On the spectral bias of neural networks.
\newblock In \emph{Proceedings of the 36th International Conference on Machine Learning}, 2019.

\bibitem[Raya and Ambrogioni(2023)]{raya2023spontaneous}
G.~Raya and L.~Ambrogioni.
\newblock Spontaneous symmetry breaking in generative diffusion models.
\newblock In \emph{Thirty-seventh Conference on Neural Information Processing Systems}, 2023.

\bibitem[Rombach et~al.(2022)Rombach, Blattmann, Lorenz, Esser, and Ommer]{rombach2022_latentdm}
R.~Rombach, A.~Blattmann, D.~Lorenz, P.~Esser, and B.~Ommer.
\newblock High-resolution image synthesis with latent diffusion models.
\newblock In \emph{Proceedings of the IEEE/CVF Conference on Computer Vision and Pattern Recognition}, pages 10684--10695, 2022.

\bibitem[Scarvelis et~al.(2025)Scarvelis, de~Oc{\'a}riz~Borde, and Solomon]{scarvelis2023closedformdiffusionmodels}
C.~Scarvelis, H.~S. de~Oc{\'a}riz~Borde, and J.~Solomon.
\newblock Closed-form diffusion models.
\newblock \emph{Transactions on Machine Learning Research}, 2025.

\bibitem[Shah et~al.(2023)Shah, Chen, and Klivans]{shah2023learning}
K.~Shah, S.~Chen, and A.~Klivans.
\newblock Learning mixtures of gaussians using the {DDPM} objective.
\newblock In \emph{Thirty-seventh Conference on Neural Information Processing Systems}, 2023.

\bibitem[Sohl-Dickstein et~al.(2015)Sohl-Dickstein, Weiss, Maheswaranathan, and Ganguli]{pmlr-v37-sohl-dickstein15}
J.~Sohl-Dickstein, E.~Weiss, N.~Maheswaranathan, and S.~Ganguli.
\newblock Deep unsupervised learning using nonequilibrium thermodynamics.
\newblock In \emph{Proceedings of the 32nd International Conference on Machine Learning}, 2015.

\bibitem[Somepalli et~al.(2023)Somepalli, Singla, Goldblum, Geiping, and Goldstein]{somepalli2023_diffusionforgery}
G.~Somepalli, V.~Singla, M.~Goldblum, J.~Geiping, and T.~Goldstein.
\newblock Diffusion art or digital forgery? investigating data replication in diffusion models.
\newblock In \emph{Proceedings of the IEEE/CVF Conference on Computer Vision and Pattern Recognition}, pages 6048--6058, 2023.

\bibitem[Song and Ermon(2019)]{Song19_estimating_grads}
Y.~Song and S.~Ermon.
\newblock Generative modeling by estimating gradients of the data distribution.
\newblock In \emph{Advances in Neural Information Processing Systems}, 2019.

\bibitem[Song et~al.(2021)Song, Sohl-Dickstein, Kingma, Kumar, Ermon, and Poole]{song2021scorebased}
Y.~Song, J.~Sohl-Dickstein, D.~P. Kingma, A.~Kumar, S.~Ermon, and B.~Poole.
\newblock Score-based generative modeling through stochastic differential equations.
\newblock In \emph{International Conference on Learning Representations}, 2021.

\bibitem[Stanczuk et~al.(2024)Stanczuk, Batzolis, Deveney, and Sch{\"o}nlieb]{stanczuk2024diffusion}
J.~P. Stanczuk, G.~Batzolis, T.~Deveney, and C.-B. Sch{\"o}nlieb.
\newblock Diffusion models encode the intrinsic dimension of data manifolds.
\newblock In \emph{Forty-first International Conference on Machine Learning}, 2024.

\bibitem[Tang and Yang(2024)]{tang24_adaptmanifold}
R.~Tang and Y.~Yang.
\newblock Adaptivity of diffusion models to manifold structures.
\newblock In \emph{Proceedings of The 27th International Conference on Artificial Intelligence and Statistics}, 2024.

\bibitem[Tenenbaum et~al.(2000)Tenenbaum, de~Silva, and Langford]{Tenenbaum_manifold_2000}
J.~B. Tenenbaum, V.~de~Silva, and J.~C. Langford.
\newblock A global geometric framework for nonlinear dimensionality reduction.
\newblock \emph{Science}, 290\penalty0 (5500):\penalty0 2319--2323, 2000.
\newblock \doi{10.1126/science.290.5500.2319}.

\bibitem[Tsybakov(2009)]{tsybakov2009nonparametric}
A.~B. Tsybakov.
\newblock Nonparametric estimators.
\newblock \emph{Introduction to Nonparametric Estimation}, pages 1--76, 2009.

\bibitem[van Erven and Harremoes(2014)]{Van_Erven2014-ry}
T.~van Erven and P.~Harremoes.
\newblock Rényi divergence and {K}ullback-{L}eibler divergence.
\newblock \emph{IEEE Trans. Inf. Theory}, 60\penalty0 (7):\penalty0 3797--3820, July 2014.

\bibitem[van Handel(2014)]{van-Handel2014-cs}
R.~van Handel.
\newblock Probability in high dimension.
\newblock \emph{Lecture Notes (Princeton University)}, 2014.

\bibitem[Vardi(2023)]{vardi2023implicit}
G.~Vardi.
\newblock On the implicit bias in deep-learning algorithms.
\newblock \emph{Communications of the ACM}, 66\penalty0 (6):\penalty0 86--93, 2023.

\bibitem[Vastola(2025)]{vastola2025generalization}
J.~Vastola.
\newblock Generalization through variance: how noise shapes inductive biases in diffusion models.
\newblock In \emph{The Thirteenth International Conference on Learning Representations}, 2025.

\bibitem[Vempala and Wibisono(2019)]{Vempala2019-xd}
S.~Vempala and A.~Wibisono.
\newblock Rapid convergence of the unadjusted {L}angevin algorithm: Isoperimetry suffices.
\newblock In \emph{Advances in Neural Information Processing Systems}, 2019.

\bibitem[Ventura et~al.(2025)Ventura, Achilli, Silvestri, Lucibello, and Ambrogioni]{ventura2025manifoldsrandommatricesspectral}
E.~Ventura, B.~Achilli, G.~Silvestri, C.~Lucibello, and L.~Ambrogioni.
\newblock Manifolds, random matrices and spectral gaps: The geometric phases of generative diffusion.
\newblock In \emph{The Thirteenth International Conference on Learning Representations}, 2025.

\bibitem[Wainwright(2019)]{Wainwright2019-rz}
M.~J. Wainwright.
\newblock \emph{High-Dimensional Statistics: A Non-Asymptotic Viewpoint}.
\newblock Cambridge Series in Statistical and Probabilistic Mathematics. Cambridge University Press, 2019.

\bibitem[Wang et~al.(2024)Wang, Zhang, Zhang, Chen, Ma, and Qu]{wang2024diffusionmodelslearnlowdimensional}
P.~Wang, H.~Zhang, Z.~Zhang, S.~Chen, Y.~Ma, and Q.~Qu.
\newblock Diffusion models learn low-dimensional distributions via subspace clustering.
\newblock \emph{arXiv [cs.LG]}, 2024.

\bibitem[Wen et~al.(2024)Wen, Liu, Chen, and Lyu]{wen2024detecting}
Y.~Wen, Y.~Liu, C.~Chen, and L.~Lyu.
\newblock Detecting, explaining, and mitigating memorization in diffusion models.
\newblock In \emph{The Twelfth International Conference on Learning Representations}, 2024.

\bibitem[Weyl(1939)]{Weyl1939-ne}
H.~Weyl.
\newblock On the volume of tubes.
\newblock \emph{American Journal of Mathematics}, 61\penalty0 (2):\penalty0 461--472, 1939.

\bibitem[Yoon et~al.(2023)Yoon, Choi, Kwon, and Ryu]{yoon2023diffusion}
T.~Yoon, J.~Y. Choi, S.~Kwon, and E.~K. Ryu.
\newblock Diffusion probabilistic models generalize when they fail to memorize.
\newblock In \emph{ICML 2023 Workshop on Structured Probabilistic Inference {\&} Generative Modeling}, 2023.

\bibitem[Zhang et~al.(2024)Zhang, Zhou, Lu, Guo, Wang, Shen, and Qu]{zhang2024_emergence}
H.~Zhang, J.~Zhou, Y.~Lu, M.~Guo, P.~Wang, L.~Shen, and Q.~Qu.
\newblock The emergence of reproducibility and consistency in diffusion models.
\newblock In \emph{Forty-first International Conference on Machine Learning}, 2024.

\end{thebibliography}

\newpage
\appendix

\section{Notation and omitted details}\label{sec:main_details}
In this section, we include some technical details concerning the theoretical results of the paper that were omitted for the sake of readability.

\subsection{Properties of the forward process}\label{app:diffusion_details}
In Section \ref{sec:intro}, we introduced the forward process in \eqref{eqn:forward_sde}. Throughout the proofs, we use the following property of the forward process:
\begin{equation}
    X_t|X_0 \sim N(\mu_t X_0, \sigma_t I_d), \qquad \mu_t = e^{- \alpha t}, \qquad \sigma^2_t = \begin{cases}
        \alpha^{-1} (1 - \mu_t^2), & \text{if } \alpha > 0,\\
        2t, & \text{otherwise.}
    \end{cases}\label{eq:cond_score_form}
\end{equation}
When \(\alpha = 0\), this follows immediately from properties of the Wiener process and when \(\alpha > 0\), \(X_t\) becomes the Ornstein-Uhlenbeck process and the result follows from a standard analysis (e.g. see \cite{Pavliotis2014-dc}).

Using this fact, we have the following closed-form expression for the density \(\hat{p}_t\), the density of the empirical forward process \(\hat{X}_t\):
\begin{align*}
    \log \hat{p}_t &= \log \bigg ( \frac{1}{N} \sum_{i=1}^N p_{X_t|X_0}(x|x_i) \bigg )\\
    &= \log \bigg ( \frac{1}{N} \sum_{i=1}^N \exp(-\|x - \mu_t x_i\|^2/2\sigma_t^2) \bigg ) - \frac{d}{2} \log(2 \pi \sigma_t^2)\\
    &= \lse \bigg ( \bigg \{ -\|x - \mu_t x_i\|^2/2\sigma_t^2) \bigg \}_{i=1}^N \bigg ) + C_t,
\end{align*}
where \(C_t = - \log(N) - \frac{d}{2} \log(2 \pi \sigma_t^2)\) and we recall the definition of the function \(\lse(\{r_i\}_i) := \log (\sum_i \exp(r_i))\).

\subsection{R\'enyi divergence}\label{app:renyi}
We provide a brief exposition of the R\'enyi divergence, which is a measure of difference between two measures. Given two measures \(\mu, \nu\) on \(\R^d\) and \(q \in (1, \infty)\) we define the \(q\)-R\'enyi divergence by
\begin{equation*}
    D_q(\mu \| \nu) = \begin{cases}
        \frac{1}{q-1} \log \int (\frac{d\mu}{d\nu}(x))^{q-1} \mu(dx), &\text{ if } \mu \ll \nu,\\
        \infty, &\text{ otherwise.}
    \end{cases}
\end{equation*}
For the case of \(q = 1\) we set \(D_q\) to be the Kullback-Leibler (KL) divergence,
\begin{equation*}
    D_1(\mu \| \nu) = \begin{cases}
        \int \log \frac{d\mu}{d\nu}(x) \ \mu(dx), &\text{ if } \mu \ll \nu,\\
        \infty, &\text{ otherwise.}
    \end{cases}
\end{equation*}
Indeed, whenever \(D_q(\mu \| \nu) < \infty\) for some \(q > 1\), it can be shown that \(\lim_{q \to 1^+} D_q(\mu \| \nu) = D_1(\mu \| \nu)\). Furthermore, whenever \(\frac{d\mu}{d\nu}\) is bounded \(\mu\)-almost surely, we obtain,
\begin{equation*}
    \lim_{q \to \infty} D_q(\mu\|\nu) = \log \bigg ( \operatorname{ess \, sup}_{\mu} \frac{d\mu}{d\nu} \bigg ),
\end{equation*}
which is taken to be \(D_\infty(\mu\|\nu)\). Thus, the R\'{e}nyi divergence provides a natural interpolation between the KL divergence and the worst-case regret, with \(D_q\) increasing in \(q\). This measure of distance recently gained popularity in the sampling \citep{Vempala2019-xd, Chewi2022-of, Erdogdu2022-sp, Mousavi-Hosseini2023-xs} and privacy \citep{Mironov2017-vf} literatures as a stronger alternative to traditional divergences. We refer to \citep{Van_Erven2014-ry} and \citep{Chewi2022-of} for further properties of this divergence.

\subsection{Projections}\label{app:projection_details}
Throughout this work, we frequently utilise the projection mapping \(\Pi_\M: \R^d \to \M\) which maps \(x \in \R^d\) to the nearest element of \(\M\). In cases where \(\M\) is curved, we run in to the issue that the projection is not well-defined as there could be multiple elements that are equally close to \(x\). In most places in the proof we consider quantities \(x\) that are sufficiently close that the projection function is uniquely defined (see reach in Appendix \ref{app:reach}) but, for example, when we define the manifold adapted kernel in \eqref{eq:man_adapt_curved}, we use the projection for all \(x \in \R^d\).

Throughout the proofs of this work we do not utilise any property of the projection aside from the fact that it maps \(x\) to some element of the manifold that is of distance \(\dist(x, \M)\) away from \(x\). For that reason, \(\Pi_\M\) can be taken to be any mapping onto \(\M\) such that \(\|x - \Pi_\M(x)\| = \dist(x, \M)\). Since \(\M\) is taken to be a closed set, such a mapping always exists and we will take this choice of mapping to be fixed throughout the work.

When \(\alpha > 0\) the samples generated at early stopping time \(\epsilon\) are slightly biased due to contractions of the Ornstein-Uhlenbeck process. For this reason, we will frequently consider the contracted manifold,
\begin{equation*}
    \mu_\epsilon \M = \{\mu_\epsilon x: x \in \M\},
\end{equation*}
where \(\mu_\epsilon\) is as defined in \eqref{eq:cond_score_form}, and we will frequently use the shorthand \(\M^\epsilon := \mu_\epsilon \M\). Given a projection mapping \(\Pi_\M\) onto \(\M\), we take the projection mapping \(\Pi_{\M^\epsilon}\) onto \(\M^\epsilon\) to be given by
\begin{equation*}
    \Pi_{\M^\epsilon}(x) = \mu_\epsilon \Pi_\M(x/\mu_\epsilon).
\end{equation*}

\section{Manifold-adaptivity in the affine setting}\label{app:linear}
We begin with the proof of Proposition \ref{prop:linear_result} concerning the affine setting. Recall that we assume that the support of \(\mud\) is restricted to the affine subspace \(\M = \{x \in \R^d: Ax = b\}\), where \(A \in \R^{d^* \times d}\) is row-orthonormal and \(b \in \R^{d^*}\), and we write the smoothing kernel in the following form:
\begin{equation*}
    k_x := \law(x + \xi),
\end{equation*}
where \(\xi\) is a centred random variable independent of \(x\). Throughout the proof, we will use the null space projection matrix \(P := I-A^TA\).

\begin{proof}[Proof of Proposition \ref{prop:linear_result}]
Since $P$ is the projection matrix onto $\text{Null}(A)$, any $z \in \R^d$ can be decomposed as
\begin{align*}
    \|z\|^2 &= \|Pz\|^2 + \|(I - P)z\|^2\\
    &= \|Pz\|^2 + \|A^T Az\|^2\\
    &= \|Pz\|^2 + \|Az\|^2,
\end{align*}
where the final line follows from the fact that \(A\) is row-orthonormal and so \(A A^T = I_{d^*}\). Using \cref{fact:lse_shift} and the assumption that $Ax_i = b$ for every $i \in [N]$, we obtain the identity
\begin{align}
    \log \hat{p}_\epsilon(z) &= \lse \bigg ( \bigg \{ -\frac{\|z - x_i\|^2}{2 \sigma_\epsilon^2} \bigg\}_{i} \bigg ) + C_t\\
    &= \lse \bigg ( \bigg \{ -\frac{\|P(z - x_i)\|^2 + \|A(z - x_i)\|^2}{2 \sigma_\epsilon^2} \bigg\}_{i} \bigg ) + C_t\\
    &= \lse \bigg ( \bigg \{ -\frac{\|P(z - x_i)\|^2}{2 \sigma_\epsilon^2} \bigg\}_{i} \bigg )
    - \frac{\|Az - b\|^2}{2 \sigma_\epsilon^2} + C_t. \label{eqn:tang_norm_decomp}
\end{align}
This decomposition separates the influence of $z$ into normal and tangent directions with respect to $\M$.

Now, let \(x \in \R^d\) and define \(Y_x = (x + \xi) \sim k_x\) and \(\widetilde{Y}_x = (x + P\xi) \sim k_x^{\M}\). Since $P = P^2$, we observe that
\begin{align*}
    \|P(Y_x-x_i)\| = \|P(\tilde{Y}_x-x_i)\|.
\end{align*}
Furthermore, using the fact that \(\xi_x\) is centred, we also have,
\begin{align*}
    \E[\|AY_x - b\|^2] = \|Ax - b\|^2 + \E[\|A \xi\|^2], \qquad
    \E[\|A\tilde{Y}_x - b\|^2] = \|Ax - b\|^2 + \E[\|AP \xi\|^2].
\end{align*}
In particular, substituting into (\ref{eqn:tang_norm_decomp}) and taking the expectation, we conclude that
\begin{align*}
    \E[\log \hat{p}_\epsilon(Y_x)] = \E \bigg [ \lse \bigg ( \bigg \{ -\frac{\|P(Y_x - x_i)\|^2}{2 \sigma_\epsilon^2} \bigg\}_{i} \bigg ) \bigg ] - \frac{\|A x - b\|^2}{2 \sigma_\epsilon^2} + C,
\end{align*}
and
\begin{align*}
    \E[\log \hat{p}_\epsilon(\tilde{Y}_x)] = \E \bigg [ \lse \bigg ( \bigg \{ -\frac{\|P(Y_x - x_i)\|^2}{2 \sigma_\epsilon^2} \bigg\}_{i} \bigg ) \bigg ] - \frac{\|A x - b\|^2}{2 \sigma_\epsilon^2} + \tilde{C},
\end{align*}
for constants $C, \tilde{C}$ independent of $x$. Therefore, the log-density of \(\hat{p}^k_\epsilon\) and \(\hat{p}^{k^\M}_\epsilon\) are identical up to a constant.
\end{proof}

\section{Lemmata}
For proving the results for the more the general manifold setting, we require several additional properties of the log-sum-exp function, smooth manifolds and tubular neighbourhoods. In this section we collect these results.

\subsection{Stability of the \(\lse\) function}

The following lemma provides stability bounds for the LSE function which we make use of throughout our analysis. It can be seen as a generalisation of Fact \ref{fact:lse_shift} which is heavily used in the analysis of the affine case.

\begin{lemma}\label{lem:lse_properties}
For any \(\{x_i\}_{i=1}^N \subset \R\) and \(\{\varepsilon_i\}_{i=1}^N \subset \R\), we have
\begin{equation}\label{eq:lse_properties_1}
    \lse(\{x_i + \epsilon_i\}_{i=1}^N) - \lse(\{x_i\}_{i=1}^N) = \int_0^1 \frac{\sum_{i=1}^N \exp(x_i + r \epsilon_i) \epsilon_i}{\sum_{i=1}^N \exp(x_i + r \epsilon_i)} dr.
\end{equation}
In particular, we have that
\begin{equation}\label{eq:lse_properties_2}
    \lse(\{x_i + \epsilon_i\}_{i=1}^N) - \lse(\{x_i\}_{i=1}^N) \leq \max\{\epsilon_i\}.
\end{equation}
\end{lemma}
\begin{proof}
From the chain rule, we compute the partial derivatives,
\begin{equation*}
    \frac{\partial}{\partial x_j} \lse(\{x_i\}_i) = \frac{\exp(x_j)}{\sum_{i=1}^N \exp(x_i)}.
\end{equation*}
Therefore, by the fundamental theorem of calculus, we obtain that
\begin{align*}
    \lse(\{x_i + \epsilon_i\}_i) - \lse(\{x_i\}_i) &= \int_0^1 \sum_{j=1}^N \epsilon_j \frac{\partial}{\partial x_j} \lse(\{x_i + r \epsilon_i\}_i) \, dr\\
    &= \int_0^1 \sum_{j=1}^N \frac{\exp(x_j + r \epsilon_j) \epsilon_j}{\sum_{i=1}^N \exp(x_i + r \epsilon_i)} dr,
\end{align*}
completing the proof of \eqref{eq:lse_properties_1}. To obtain \eqref{eq:lse_properties_2}, we use the fact that the sum is a weighted average of the sequence $\{\epsilon_i\}_i$, and so it has the property that,
\begin{equation*}
    \sum_{i=1}^N \frac{\exp(x_i + r \epsilon_i) \epsilon_i}{\sum_{j=1}^N \exp(x_j + r \epsilon_j)} \leq \max_i \epsilon_i.
\end{equation*}
\end{proof}
We also note that if all \(\epsilon_i\) are identical, we readily recover Fact \ref{fact:lse_shift} from \eqref{eq:lse_properties_1}.

\subsection{Manifold reach}\label{app:reach}
Next, we collect some facts about the \textit{reach} of the manifold, which quantifies how far one can extend from the manifold before the projection onto it ceases to be unique. We refer to \citep{Aamari2017-he} for a more detailed exposition. We begin with the rigorous definition.
\begin{definition}
The reach of a set \(A \subset \R^d\), is defined by \(\tau_A = \inf_{p \in A} d(p, Med(A))\), where we define the set,
\begin{equation*}
    Med(A) = \Big \{z \in \R^d: \exists p, q \in A \text{ s.t. } p \neq q, \|p - z\| = \|q - z\| = d(z, A) \Big \}.
\end{equation*}
\end{definition}

The reach defines the maximum distance at which the projection to the set is unique. In the case where the set \(A\) is a smooth submanifold of \(\R^d\), the reach captures the curvature of the manifold and provides an upper bound on the distance at which the manifold appears approximately flat. The following lemma demonstrates this, controlling the curvature of paths along the manifold using the reach. We use the notation \(N_x \M\) to denote the \textit{normal space} of the manifold \(\M\) at \(x \in \M\) which consists of all vectors perpendicular to the tangent space at \(x\).

\begin{lemma}\label{lem:reach_angle}
Suppose that the manifold \(\M\) has reach \(\tau_\M > 0\), then for any \(x, y \in \M\), \(v \in N_x \M\),
\begin{equation*}
    |\langle v, x - y \rangle | \leq \frac{\|v\|}{2\tau_\M} \|x - y\|^2.
\end{equation*}
\end{lemma}
\begin{proof}
Let \(x, y \in \M\), \(v \in N_x \M\) and define \(z = x + r v/\|v\|\) for some \(r \in (0, \tau_\M)\) so that \(z \not \in Med(\M)\). Since the projection is uniquely defined with \(\Pi_\M(z) = x\), we have, 
\begin{align*}
    \|z - y\| \geq \dist(z, M) = r.
\end{align*}
On the other hand, we have,
\begin{align*}
    \|z - y\|^2 = \|x - y\|^2 + r^2 + \frac{2 r}{\|v\|} \langle x - y, v \rangle.
\end{align*}
Combining these two and rearranging, we obtain the bound,
\begin{equation*}
    \langle y - x, v \rangle \leq \frac{\|v\|}{2r} \|x-y\|^2.
\end{equation*}
By replacing \(v\) with \(-v\) which is also in the normal space, we obtain the opposite direction, hence obtaining,
\begin{equation*}
    |\langle x - y, v \rangle| \leq \frac{\|v\|}{2r} \|x-y\|^2.
\end{equation*}
Since this bound holds for all \(r \in (0, \tau_\M)\), we can take \(r \to \tau_\M^-\) to obtain the bound in the statement.
\end{proof}

With this, we can control the geodesic distance on the manifold by the standard Euclidean distance.
\begin{lemma}\label{lem:geo_to_euc}
Suppose that the manifold \(\M\) has reach \(\tau_\M > 0\). Let \(x, y \in \M\) such that \(\|x - y\| \leq \tau_\M/2\) and let \(\gamma_t\) be a geodesic (shortest path) between \(x, y\) on \(\M\). Then, we have the bound,
\begin{equation*}
    \int \|\partial_t \gamma_t\| dt \leq 2 \|x - y\|.
\end{equation*}
\end{lemma}
Therefore, if we are close enough to the manifold, its inherited metric behaves roughly like the Euclidean one. The proof of this lemma can be found in Lemma III.21 of \citep{Aamari2017-he}.

\subsection{Concentration under the manifold hypothesis}

We now turn to results concerning probability measures supported on submanifolds with bounded reach. The next lemma controls the mass of a small ball centred on the manifold, showing that the rates are similar to the affine case.

\begin{lemma}\label{lem:reach_ball_mass}
Suppose that the measure \(\mud\) is supported on a smooth compact submanifold \(\M\) with reach \(\tau_\M > 0\) and dimension \(d^*\). Then, for any \(r \leq \pi \tau_\M/2\sqrt{2}\), we have
\begin{equation*}
    \mud(B_r(x)) \geq c_\mu r^{d^*}, \qquad c_\mu = \inf_{B_r(x)} p_\mu,
\end{equation*}
where \(p_\mu\) denotes the density of \(\mud\) with respect to the volume measure on \(\M\).
\end{lemma}
For the proof of this lemma, we refer to the proof of Proposition 4.3 of \citep{Aamari2019-gk} or Lemma III.23 of \citep{Aamari2017-he}.

We use this bound, to obtain a result concerning the concentration of the empirical measure on the manifold. To this end, we recall a bound on the covering number of the manifold. Given \(r > 0\), the covering number \(N_{\text{cov}}(\M, r)\) is defined as the minimum number of Euclidean balls of radius \(r\) required to cover the subset of \(\R^d\) defined by the space \(\M\). The following lemma is from Proposition III.11 of \cite{Aamari2017-he}.
\begin{lemma}\label{lem:mani_cover}
Consider the setting of Lemma \ref{lem:reach_ball_mass} and suppose that \(c_\mu > 0\), then for any \(\varepsilon \in (0, \tau_\M/2)\) we have,
\begin{equation*}
    N_{\text{cov}}(\M, 2\varepsilon) \leq \frac{1}{c_\mu \varepsilon^{d^*}}.
\end{equation*}
\end{lemma}

We can now prove a bound on the concentration of the empirical measure.
\begin{lemma}\label{lem:lower_emp_meas}
Suppose that the measure \(\mud\) is supported on a compact smooth submanifold \(\M\) with reach \(\tau_\M > 0\) and dimension \(d^*\) and \(c_\mu > 0\). Then, for any \(r \in (0, \tau_\M]\), \(\delta \in (0, 1)\), we have
\begin{equation*}
    \prob \bigg ( \inf_{x \in M} \hmud(B_r(x)) \geq c_\mu \frac{r^{d^*}}{4} \bigg ) \geq 1 - \delta,
\end{equation*}
whenever,
\begin{equation}\label{eq:lower_emp_meas_1}
    N \geq 64 \Big ( (144 d^* \log(2/r) + 144 \log(2/c_\mu)) \vee \tfrac{\log(\delta^{-1})}{2} \Big ) r^{-2d^*} c_\mu^{-2}.
\end{equation}
\end{lemma}
\begin{proof}
According to Lemma \ref{lem:mani_cover}, there exists a set \(\mathcal{C} \subset \M\) such that \(\{B_{2^{-1/d^*}r}(c): c \in \mathcal{C}\}\) forms a covering of \(\M\) with \(|\mathcal{C}| \leq c_\mu^{-1} 2^{d^* + 1} r^{-d^*}\). Thus, for any \(x \in M\) there exists \(c \in \mathcal{C}\) such that \(x \in B_{2^{-1/d^*} r}(c)\) and therefore \(B_{2^{-1/d^*} r}(c) \subset B_r(x)\). From this we deduce the following bound
\begin{equation*}
    \inf_{x \in \M} \hmud(B_r(x)) \geq \inf_{c \in \mathcal{C}} \hmud(B_{2^{-1/d^*} r}(c)).
\end{equation*}
Therefore, it suffices to lower bound this object on the right-hand side.

Next, using a rudimentary bound from the empirical processes literature (for example, see Section 4.2 of \citep{Wainwright2019-rz} or Section 7.1 of \citep{van-Handel2014-cs}), we obtain the bound,
\begin{equation*}
    \prob \bigg ( \sup_{c \in \mathcal{C}}| \hmud(B_{2^{-1/d^*}r}(c)) - \mud(B_{2^{-1/d^*}r}(c))| \geq 12 \sqrt{\frac{\log{|\mathcal{C}|}}{N}} + \varepsilon \bigg ) \leq \exp(-2N\varepsilon^2).
\end{equation*}
Thus, choosing \(\varepsilon = c_\mu r^{d^*}/8\), we obtain that under \eqref{eq:lower_emp_meas_1}, it follows that
\begin{equation}\label{eq:lower_emp_meas_3}
    \prob \bigg ( \sup_{C \in \mathcal{C}}| \hmud(C) - \mud(C)| \geq c_\mu r^{d^*}/4 \bigg ) \leq \delta.
\end{equation}
To conclude the proof, we use Lemma \ref{lem:reach_ball_mass} to obtain that,
\begin{align*}
    \inf_{c \in \mathcal{C}} \hmud(B_{2^{-1/d^*}r}(c)) &\geq \inf_{c \in \mathcal{C}} \hmud(B_{2^{-1/d^*}r}(c)) - \sup_{c \in M} | \hmud(B_{2^{-1/d^*}r}(c)) - \mud(B_{2^{-1/d^*}r}(c))|\\
    &\geq c_\mu r^{d^*}/2 - \sup_{x \in M} | \hmud(B_{2^{-1/d^*}r}(x)) - \mud(B_{2^{-1/d^*}r}(x))|.
\end{align*}
Combining this with \eqref{eq:lower_emp_meas_3}, we arrive at the bound in the statement.
\end{proof}

\subsection{Weyl's tube formula}\label{app:weyl}

The sets \(\M_r\) and \(\M^\epsilon_r\) are related to the notion of \textit{tubes} that have been investigated in the differential geometry literature \citep{Gray2004-rz, Weyl1939-ne}. We borrow a result from \cite{Weyl1939-ne} that computes the volume enclosing these sets. Let \(\M^\epsilon_{\leq r} := \{x \in \R^d: \dist(x, \M^\epsilon) \leq r\}\).
\begin{proposition}[Weyl's Tube Formula]\label{prop:weyl}
Suppose Assumption \ref{ass:manifold} holds, then for all \(r \geq 0\),
\begin{equation*}
    \lambda(\M^\epsilon_{\leq r}) = \sum_{p=0}^{\lfloor d^*/2 \rfloor} \tilde{k}_{2p}(\M^\epsilon) r^{d-d^*+2p},
\end{equation*}
for some quantities \(\tilde{k}_{2p}(\M^\epsilon) \geq 0\).
\end{proposition}
The quantities \(\tilde{k}\) are related to the integrated mean curvature of \(\M\) and further details about these quantities can be found in \citep{Gray2004-rz} where the result is stated in Section 1.1. In this work, we develop upper bounds in such a way that the final result does not depend on these quantities.

Note that using this result, we can obtain estimates for the integrals of functions depending on  \(\lambda(\M^\epsilon_{\leq r})\) using the expression,
\begin{align}
    \int_{\R^d} f(\dist(x, \M^\epsilon)) dx &= \int_0^\infty f(r) \frac{d}{dr} \lambda(\M^\epsilon_{\leq r}) dr \nonumber\\
    &= \sum_{p=0}^{\lfloor d^*/2 \rfloor} (d-d^*+2p) \tilde{k}_{2p}(\M^\epsilon) \int_0^\infty f(r) r^{d-d^*-1+2p} dr.\label{eq:weyl_int}
\end{align}

\section{Proofs for the main results}
This section of the appendix provides the proofs for theorems \ref{thm:density_ratio_bound_simplified} and \ref{thm:geometric_bias}. These theorems establish the core result that smoothing in the log-domain is approximately geometry-adaptive, meaning that smoothing with a generic kernel $k$ behaves similarly to smoothing with a manifold-adapted kernel $k^\M$. We begin by proving some lemmas that are involved in the proof of both of these theorems.

\subsection{Controlling the log-density ratio}
To establish the proximity between $\hat{p}_{\epsilon}^k$ and $\hat{p}_{\epsilon}^{k^\M}$ in divergence, we must control the ratio of their densities. In this section, we fix a permissible manifold \(\M \in \mathbb{M}_\mu\), and use \(K_\M, K_{\max, \M}\), \(\tau_\M\) and \(d^*_\M\) as in Section \ref{sec:geometric_bias_theory}. We also fix \(x \in \R^d\) and let \(Y \sim k_x\), \(\tilde{Y} = \Pi_{\M_{r(x)}}(Y)\), so that \(\tilde{Y} \sim k^\M_x\). Using the expression in \eqref{eqn:emp_density_lse}, we can express the density ratio by,
\begin{equation}
    \log {\frac{d\hat{p}^k_\epsilon}{\ d\hat{p}^{k^\M}_\epsilon}}(x) = \E \bigg [ \lse \bigg ( \bigg \{ -\frac{\|Y - \mu_\epsilon x_i\|^2}{2 \sigma_\epsilon^2} \bigg\}_{i=1}^N \bigg ) - \lse \bigg ( \bigg \{ -\frac{\|\tilde{Y} - \mu_\epsilon x_i\|^2}{2 \sigma_\epsilon^2} \bigg\}_{i=1}^N \bigg ) \bigg | S \bigg ] + \log \bigg ( \frac{C^\M}{C} \bigg ),\label{eq:density_ratio_delta}
\end{equation}
where we define the normalising constants,
\begin{equation}
    C = \int \exp \bigg ( \int \log \hat{p}_\epsilon(y) k_x(dy) \bigg ) dx, \qquad C^\M = \int \exp \bigg ( \int \log \hat{p}_\epsilon(y) k^\M_x(dy) \bigg ) dx.\label{eq:pf_normalisation}
\end{equation}
We proceed similarly to the proof in the affine case (see Appendix \ref{app:linear}), decomposing the \(\lse\) function into normal and perpendicular components. We use the decomposition,
\begin{equation*}
    \|y - \mu_\epsilon x_i\|^2 = \|y - \Pieps(y)\|^2 + 2 \langle y - \Pieps(y), \Pieps(y) - \mu_\epsilon x_i \rangle + \|\Pieps(y)- \mu_\epsilon x_i\|^2,
\end{equation*}
along with Fact \ref{fact:lse_shift} to obtain that,
\begin{align*}
    &\lse \bigg ( \bigg \{ -\frac{\|Y - \mu_\epsilon x_i\|^2}{2 \sigma_\epsilon^2} \bigg\}_{i \in [N]} \bigg ) \\
    & \quad = \lse \bigg ( \bigg \{ -\frac{\|\Pieps(Y) - \mu_\epsilon x_i\|^2 + 2\langle Y - \Pieps(Y), \Pieps(Y) - \mu_\epsilon x_i \rangle}{2 \sigma_\epsilon^2} \bigg\}_{i} \bigg ) - \frac{\|Y - \Pieps(Y)\|^2}{2 \sigma_\epsilon^2}.
\end{align*}
It follows from the definition of \(\tilde{Y}\) and \(r(x)\) that,
\begin{align*}
    \|\tilde{Y} - \Pieps(Y)\|^2 &= \|\tilde{Y} - \Pieps(\tilde{Y})\|^2 \\
    &= r(x)^2\\
    &= \E[\|Y - \Pieps(Y)\|^2].
\end{align*}
Therefore we may apply Fact \ref{fact:lse_shift} once more to obtain,
\begin{align*}
    &\E \bigg [ \lse \bigg ( \bigg \{ -\frac{\|Y - \mu_\epsilon x_i\|^2}{2 \sigma_\epsilon^2} \bigg\}_{i \in [N]} \bigg ) \bigg | S \bigg ] \\
    & \quad = \E \bigg [ \lse \bigg ( \bigg \{ -\frac{\|\Pieps(Y) - \mu_\epsilon x_i\|^2 + 2\langle Y - \Pieps(Y), \Pieps(Y) - \mu_\epsilon x_i \rangle}{2 \sigma_\epsilon^2} \bigg\}_{i} \bigg ) - \frac{\|\tilde{Y} - \Pieps(Y)\|^2}{2 \sigma_\epsilon^2} \bigg | S \bigg ] \\
    & \quad = \E \bigg [ \lse \bigg ( \bigg \{ -\frac{\|\Pieps(Y) - \mu_\epsilon x_i\|^2 + 2\langle Y - \Pieps(Y), \Pieps(Y) - \mu_\epsilon x_i \rangle + \|\tilde{Y} - \Pieps(Y)\|^2}{2 \sigma_\epsilon^2} \bigg\}_{i} \bigg ) \bigg | S \bigg ] \\
    & \quad = \E \bigg [ \lse \bigg ( \bigg \{ -\frac{\|\tilde{Y} - \mu_\epsilon x_i\|^2 + \Delta_i}{2 \sigma_\epsilon^2} \bigg\}_{i \in [N]} \bigg ) \bigg | S \bigg ],
\end{align*}
where we define the quantity \(\Delta_i := 2\langle Y - \tilde{Y}, \Pieps(Y) - \mu_\epsilon x_i \rangle\). Therefore, we obtain the simple expression,
\begin{gather}
    \log {\frac{d\hat{p}^k_\epsilon}{\ d\hat{p}^{k^\M}_\epsilon}}(x) = \E[\Delta \lse_\M(x)|S] + \log \bigg ( \frac{C^\M}{C} \bigg ),\label{eq:density_ratio_delta_2}\\
    \Delta \lse_\M(x) := \lse \bigg ( \bigg \{ -\frac{\|\tilde{Y} - \mu_\epsilon x_i\|^2 + \Delta_i}{2 \sigma_\epsilon^2} \bigg\}_{i \in [N]} \bigg ) - \lse \bigg ( \bigg \{ -\frac{\|\tilde{Y} - \mu_\epsilon x_i\|^2}{2 \sigma_\epsilon^2} \bigg\}_{i=1}^N \bigg ).\nonumber
\end{gather}
Having expressed the log-density ratio in terms of \(\Delta \lse_\M\), our next task is to bound this quantity. For the sake of intuition, we can consider the linear setting: In this case, \(\tilde{Y}-Y\) is normal to the manifold and \(\Pieps(Y) - \mu_\epsilon x_i\) is tangent to the manifold, so it would follow that \(\Delta_i = 0\) and thus \(\Delta \lse_\M (x) = 0\). In the case where the manifold is curved, it is no longer necessarily true that \(\Delta_i\) is \(0\) and so we control \(\Delta \lse_\M\) using the curvature of the manifold and the stability of the \(\lse\) function.

We begin with a simple lemma.

\begin{lemma}\label{lem:simple_delta_bound}
Suppose that \(\gap := \dist(\{x_i\}_{i=1}^N, \M) < \infty\) and \(\tau_\M > 0\). Then, for any \(x \in \R^d, i \in [N]\) we have,
\begin{equation*}
    |\Delta_i| \leq \mu_\epsilon |\zeta| \Big ( \Big ( \tfrac{1}{2 \tau_\M} d_i^2 \Big ) \wedge d_i + \gap \Big ), \qquad d_i = \| \Pi_\M(Y/\mu_\epsilon) - \Pi_\M(x_i)\|,
\end{equation*}
where we define the quantity,
\begin{equation*}
    \zeta := \|Y - \Pieps(Y)\| - \E \Big [ \|Y - \Pieps(Y)\|^2 \Big ]^{1/2}.
\end{equation*}
\end{lemma}
\begin{proof}
Using the definition of \(\tilde{Y}\), we obtain that,
\begin{align*}
    Y - \tilde{Y} &= (Y - \Pieps(Y)) - (\tilde{Y} - \Pieps(Y))\\
    &= (Y - \Pieps(Y)) - \frac{Y - \Pieps(Y)}{\|Y - \Pieps(Y)\|}\E \Big [ \|Y - \Pieps(Y)\|^2 \Big ]^{1/2}\\
    &= \frac{Y - \Pieps(Y)}{\|Y - \Pieps(Y)\|} \bigg ( \|Y - \Pieps(Y)\| - \E \Big [ \|Y - \Pieps(Y)\|^2 \Big ]^{1/2} \bigg ).
\end{align*}
With this, we can write \(\Delta_i\) in the following form:
\begin{equation}\label{eq:drbg_1}
    \Delta_i = 2 \zeta \bigg \langle \frac{Y - \Pieps(Y)}{\|Y - \Pieps(Y)\|}, \Pieps(Y) - \mu_\epsilon x_i \bigg \rangle.
\end{equation}
To control \(\Delta_i\), we use Lemma \ref{lem:reach_angle} as well as the Cauchy-Schwarz inequality to obtain, 
\begin{align}
    |\Delta_i| &= 2 \bigg | \zeta \bigg ( \bigg \langle \frac{Y - \Pieps(Y)}{\|Y - \Pieps(Y)\|}, \Pieps(Y) - \mu_\epsilon \Pi_\M(x_i) \bigg \rangle + \bigg \langle \frac{Y - \Pieps(Y)}{\|Y - \Pieps(Y)\|}, \mu_\epsilon \Pi_\M(x_i) - \mu_\epsilon x_i \bigg \rangle \bigg ) \bigg | \nonumber\\
    &\leq 2 \mu_\epsilon |\zeta| \bigg (\frac{\| \Pi_\M(Y/\mu_\epsilon) - \Pi_\M(x_i) \|^2}{2\tau_\M} \wedge \| \Pi_\M(Y/\mu_\epsilon) - \Pi_\M(x_i) \| + \gap \bigg ),\label{eq:drbg_1h}
\end{align}
completing the proof of the lemma.
\end{proof}

Since \(\zeta\) is a random variable, we next find ways of controlling it using \(K_\M\) and \(K_{\max, \M}\).
\begin{lemma}\label{lem:zeta_bound}
Let \(\zeta\) be as in Lemma \ref{lem:simple_delta_bound} and suppose that \(K_\M, K_{\max, \M} < \infty\), then we have that,
\begin{equation*}
    \E[|\zeta|^2|S]^{1/2} \leq 2 K_\M, \qquad |\zeta| \leq K^{\max}_\M + K_\M,
\end{equation*}
almost surely.
\end{lemma}
\begin{proof}
For the first bound, we use the \(L^2\)-triangle inequality to obtain,
\begin{align*}
    \E[|\zeta|^2|S]^{1/2} &\leq \E \Big [ \big | \|Y - \Pieps(Y)\| - \|x - \Pieps(x)\| \big |^2 \Big | S \Big ]^{1/2} + \big | \|x - \Pieps(x)\| - \E[\|Y - \Pieps(Y)\|^2 | S ]^{1/2} \big |\\
    &\leq \E \Big [ \big | \|Y - \Pieps(Y)\| - \|x - \Pieps(x)\| \big |^2 \Big | S \Big ]^{1/2} + \E \Big [ \Big ( \|Y - \Pieps(Y)\| - \|x - \Pieps(x)\| \Big )^2 \Big | S \Big ]^{1/2}\\
    &\leq 2 K_\M.
\end{align*}
Similarly, we can obtain \(L^\infty\) bounds via,
\begin{align*}
    \|\zeta\|_{L^\infty} &= \Big \| \|Y - \Pieps(Y)\| -  \|x - \Pieps(x)\| \Big \|_{L^\infty} + \Big \| \|x - \Pieps(x)\| -  \E[\|Y - \Pieps(Y)\|^2 | S ]^{1/2} \Big \|_{L^\infty}\\
    &\leq K_{\max, \M} + K_\M.
\end{align*}
\end{proof}

We now state the bound for \(\Delta \lse_\M\) that we use for our two main theorems.

\begin{lemma}\label{lem:tighter_delta_bound}
Consider the setting of Lemma \ref{lem:simple_delta_bound}, then for any \(x \in \R^d\), it holds that
\begin{align*}
    &\E[|\Delta \lse_\M(x)|S]\\
    & \qquad \leq \frac{8 K_\M}{\tau_\M} \bigg ( 1 + 2\log \big ( \tfrac{\tau_\M}{K_\M} \big )_+ + (\gap + \tau_\M + K_\M + D_x ) \frac{12\gap}{\sigma_\epsilon^2} + \Big ( 5 K_{\max, \M}^{1/2} K_\M^{1/2} + D_x \Big )^2 \frac{\prob(\mathcal{E}_x|S)^{1/2}}{\sigma_\epsilon^2}\\
    & \qquad \qquad + 2\inf_{\varepsilon_0 > 0} \bigg \{ \log \Big ( \E_{Y \sim k_x} [ \hat{\mu}_{\operatorname{data}}(B_{\varepsilon_0}(\Pi_\M(Y/\mu_\epsilon)))^{-1} | S ] \Big )_+ + \Big ( 1 + \tfrac{K_\M + D_x}{\tau_\M} \Big ) \frac{\varepsilon_0^2}{\sigma_\epsilon^2} \bigg \} \bigg ),
\end{align*}
for some universal constant \(\tilde{C} > 0\), where we define \(\mathcal{E}_x = \{ K_\M + D_x + 2 |\zeta| \geq \mu_\epsilon \tau_\M / 2 \}\), \(D_x = \|x - \Pieps(x)\|\).
\end{lemma}
\begin{proof}
For this bound, we begin with \eqref{eq:lse_properties_1} of Lemma \ref{lem:lse_properties} to obtain
\begin{equation*}
    | \Delta \lse_\M(x) | \leq \int_0^1 \frac{\sum_{i \in [N]} \exp(-\|\tilde{Y} - \mu_\epsilon x_i\|^2/2\sigma_\epsilon^2-r\Delta_i/2\sigma_\epsilon^2) |\Delta_i|/2\sigma_\epsilon^2}{\sum_{i \in [N]} \exp(-\|\tilde{Y} - \mu_\epsilon x_i\|^2/2\sigma_\epsilon^2-r\Delta_i/2\sigma_\epsilon^2)} dr.
\end{equation*}
We decompose this further as,
\begin{align*}
    | \Delta \lse_\M(x) | &\leq \int_0^1 \frac{\sum_{i \in I_1} \exp(-\|\tilde{Y} - \mu_\epsilon x_i\|^2/2\sigma_\epsilon^2-r\Delta_i/2\sigma_\epsilon^2) |\Delta_i|/2\sigma_\epsilon^2}{\sum_{i \in I_1} \exp(-\|\tilde{Y} - \mu_\epsilon x_i\|^2/2\sigma_\epsilon^2-r\Delta_i/2\sigma_\epsilon^2)} dr\\
    & \qquad + \int_0^1 \frac{\sum_{i \in I_1^\complement} \exp(-\|\tilde{Y} - \mu_\epsilon x_i\|^22\sigma_\epsilon^2-r\Delta_i/2\sigma_\epsilon^2) |\Delta_i|/2\sigma_\epsilon^2}{\sum_{i \in I_0} \exp(-\|\tilde{Y} - \mu_\epsilon x_i\|^2/2\sigma_\epsilon^2-r\Delta_i/2\sigma_\epsilon^2)} dr \\
    &\leq \mathbbm{1}_{|I_1| > 0} \max_{i \in I_{1}} \bigg \{ \frac{|\Delta_i|}{2\sigma_\epsilon^2} \bigg \} + \int_0^1 \frac{\sum_{i \in I_{1}^\complement} \exp(-\|\tilde{Y} - \mu_\epsilon x_i\|^2/2\sigma_\epsilon^2-r\Delta_i/2\sigma_\epsilon^2) |\Delta_i|/2\sigma_\epsilon^2}{\sum_{i \in I_0} \exp(-\|\tilde{Y} - \mu_\epsilon x_i\|^2/2\sigma_\epsilon^2-r\Delta_i/2\sigma_\epsilon^2)} dr \\
    &=: \mathtt{A} + \mathtt{B},
\end{align*}
where we define the sets,
\begin{gather*}
    I_0 = \{i \in [N]: \|\Pi_\M(Y/\mu_\epsilon) - \Pi_\M(x_i)\| \leq \varepsilon_0\}, \qquad I_1 = \{i \in [N]: \|\Pi_\M(Y/\mu_\epsilon) - \Pi_\M(x_i)\| \leq \varepsilon_1\}
\end{gather*}
for some random quantities \(\varepsilon_0, \varepsilon_1 > 0\).

The quantity \(\mathtt{A}\) can be bounded directly using Lemma \ref{lem:simple_delta_bound}. From this, we obtain,
\begin{align*}
    \mathtt{A} \leq \frac{\mu_\epsilon}{\sigma^2_\epsilon} |\zeta| \Big ( \tfrac{1}{2 \tau_\M} \varepsilon_1^2 + \gap \Big ).
\end{align*}
To bound \(\mathtt{B}\), we proceed with the following upper bound,
\begin{align}
    \mathtt{B} &\leq \frac{\sum_{j \in I_1^\complement} \exp(-\|\tilde{Y} - \mu_\epsilon x_j\|/2\sigma_\epsilon^2 + |\Delta_j|/2\sigma_\epsilon^2) |\Delta_j|/2\sigma_\epsilon^2}{\sum_{i \in I_0}\exp(-\|\tilde{Y} - \mu_\epsilon x_i\|/2\sigma_\epsilon^2-|\Delta_i|/2\sigma_\epsilon^2)} \nonumber\\
    &\leq |I_0|^{-1} \max_{i \in I_0}\sum_{j \in I_1^\complement} \exp(\|\tilde{Y} - \mu_\epsilon x_i\|^2/2\sigma_\epsilon^2 - \|\tilde{Y} - \mu_\epsilon x_j\|^2/2\sigma_\epsilon^2 + |\Delta_i|/2\sigma_\epsilon^2 + |\Delta_j|/\sigma_\epsilon^2),\label{eq:drbg_3}
\end{align}
where we have used the fact that \(r \leq \exp(r)\). To further control \(\mathtt{B}\), we control the quantity inside of the exponential function by choosing \(\varepsilon_0\) sufficiently small and \(\varepsilon_1\) sufficiently large so that the quantity in the exponential becomes negative. This allows for control of \(\mathtt{B}\) by taking \(\epsilon\) sufficiently small.

We start by controlling \(\|\tilde{Y} - \mu_\epsilon x_i\|^2-\|\tilde{Y} - \mu_\epsilon x_j\|^2\). For any \(i \in I_0, j \in [N]\), we have that
\begin{align}
    &\|\tilde{Y} - \mu_\epsilon x_i\|^2-\|\tilde{Y} - \mu_\epsilon x_j\|^2\nonumber\\
    & \qquad = \| \Pieps(Y) - \mu_\epsilon x_i\|^2  + 2 \langle \tilde{Y} - \Pieps(Y), \Pieps(Y) - \mu_\epsilon x_i \rangle - \|\Pieps(Y) - \mu_\epsilon x_j\|^2\nonumber\\
    & \qquad \qquad - 2 \langle \tilde{Y} - \Pieps(Y), \Pieps(Y) - \mu_\epsilon x_j \rangle.\label{eq:drbg_3h}
\end{align}
The first term is bounded using the fact that \(\|\Pieps(Y) - \mu_\epsilon x_i\| \leq \mu_\epsilon \varepsilon_0 + \mu_\epsilon \gap\). The second term is controlled using the technique from the proof of Lemma \ref{lem:simple_delta_bound} to deduce the bound,
\begin{align}
    &\langle \tilde{Y} - \Pieps(Y), \Pieps(Y) - \mu_\epsilon x_i \rangle\\
    & \qquad \leq \mu_\epsilon \|\tilde{Y} - \Pieps(Y)\| \Big ( \big ( \tfrac{1}{2\tau_\M} \|\Pi_\M(Y/\mu_\epsilon) - \Pi_\M(x_i)\|^2 \big ) \wedge \|\Pi_\M(Y/\mu_\epsilon) - \Pi_\M(x_i)\| + \gap \Big )\nonumber\\
    & \qquad \leq \mu_\epsilon (K_\M + D_x) \Big ( \big ( \tfrac{1}{2\tau_\M} \varepsilon_0^2 \big ) \wedge \varepsilon_0 + \gap \Big ),\label{eq:drbg_3hh}
\end{align}
where in the second line, we use that
\begin{align*}
    \|\tilde{Y} - \Pieps(Y)\| &= \E[\|Y - \Pieps(Y)\|^2]^{1/2} \\
    &\leq \E \Big [ \big | \|Y - \Pieps(Y)\| - \|x - \Pieps(x)\| \big |^2 \Big ]^{1/2} + \|x - \Pieps(x)\|\\
    &\leq K_\M + D_x.
\end{align*}
Similarly, the fourth term of \eqref{eq:drbg_3h} is bounded by,
\begin{align}
    &-\langle \tilde{Y} - \Pieps(Y), \Pieps(Y) - \mu_\epsilon x_j \rangle \nonumber\\
    & \qquad \leq \mu_\epsilon (K_\M + D_x) \Big ( \big ( \tfrac{1}{2\tau_\M} \|\Pi_\M(Y/\mu_\epsilon) - \Pi_\M(x_j)\|^2 \big ) \wedge \|\Pi_\M(Y/\mu_\epsilon) - \Pi_\M(x_j)\| + \gap \Big ),\label{eq:drbg_4h}
\end{align}
and finally, the third term of \eqref{eq:drbg_3h} is controlled using Young's inequality to obtain,
\begin{align}
    \|\Pieps(Y) - \mu_\epsilon x_j\|^2 &= \|\Pieps(Y) - \Pieps(\mu_\epsilon x_j)\|^2 + \|\Pieps(\mu_\epsilon x_j) - \mu_\epsilon x_j\|^2 \nonumber\\
    & \qquad + 2 \langle \Pieps(Y) - \Pieps(\mu_\epsilon x_j), \Pieps(\mu_\epsilon x_j) - \mu_\epsilon x_j \rangle \nonumber\\
    &\geq \frac{3}{4}\|\Pieps(Y) - \Pieps(\mu_\epsilon x_j)\|^2 - 3 \|\Pieps(\mu_\epsilon x_j) - \mu_\epsilon x_j\|^2 \nonumber\\
    &\geq \frac{3}{4}\|\Pieps(Y) - \Pieps(\mu_\epsilon x_j)\|^2 - 3 \mu_\epsilon^2 \gap^2,\label{eq:drbg_4}
\end{align}
Thus, substituting \eqref{eq:drbg_3hh}, \eqref{eq:drbg_4h} and \eqref{eq:drbg_4} in to \eqref{eq:drbg_3h} leads to the bound,
\begin{align}
    &\|\tilde{Y} - \mu_\epsilon x_i\|^2-\|\tilde{Y} - \mu_\epsilon x_j\|^2 \nonumber\\
    & \qquad \leq \mu_\epsilon^2 (\varepsilon_0 + \gap)^2 + 2\mu_\epsilon (K_\M + D_x) \Big ( \big ( \tfrac{1}{2\tau_\M} \varepsilon_0^2 \big ) \wedge \varepsilon_0 + \gap \Big ) - \frac{3}{4} \|\Pieps(Y) - \mu_\epsilon \Pi_\M(x_j)\|^2 + 3 \mu_\epsilon^2 \gap^2 \nonumber\\
    & \qquad \qquad + 2 \mu_\epsilon (K_\M + D_x) \Big ( \big ( \tfrac{1}{2\tau_\M}\|\Pi_\M(Y/\mu_\epsilon) - \mu_\epsilon \Pi_\M(x_j)\|^2 \big ) \wedge \|\Pi_\M(Y/\mu_\epsilon) - \Pi_\M(x_j)\| + \gap \Big ) \nonumber\\
    & \qquad \leq \mu_\epsilon \Big ( 2 \mu_\epsilon + \tfrac{1}{\tau_\M} K_\M + \tfrac{1}{\tau_\M} D_x \Big ) \varepsilon_0^2 + \mu_\epsilon \Big ( 5 \mu_\epsilon \gap + 4 K_\M + 4 D_x \Big ) \gap - \tfrac{3}{4} \|\Pieps(Y) - \mu_\epsilon \Pi_\M(x_j)\|^2 \nonumber \\
    & \qquad \qquad + 2 ( K_\M + D_x) \big ( \tfrac{1}{2 \mu_\epsilon \tau_\M}\|\Pieps(Y) - \mu_\epsilon \Pi_\M(x_j)\|^2 \big ) \wedge \|\Pieps(Y) - \mu_\epsilon \Pi_\M(x_j)\|. \label{eq:drbg_44}
\end{align}
Continuing with bounding the contents of the exponential function in \eqref{eq:drbg_3}, we next control \(|\Delta_i| + 2 |\Delta_j|\), using Lemma \ref{lem:simple_delta_bound} to obtain,
\begin{equation}
    |\Delta_i| + 2 |\Delta_j| \leq 2 \mu_\epsilon |\zeta| \Big ( \tfrac{1}{2\tau_\M} \varepsilon_0^2 + 2 \big ( \tfrac{1}{2\tau_\M} \|\Pi_\M(Y/\mu_\epsilon) - \Pi_\M(x_j) \|^2 \big ) \wedge \|\Pi_\M(Y/\mu_\epsilon) - \Pi_\M(x_j) \|  + 3 \gap \Big ).\label{eq:drbg_5}
\end{equation}
Therefore, combining \eqref{eq:drbg_44} and \eqref{eq:drbg_5}, we obtain the bound,
\begin{align*}
    &\|\tilde{Y} - \mu_\epsilon x_i\|^2-\|\tilde{Y} - \mu_\epsilon x_j\|^2 + |\Delta_i| + 2|\Delta_j|\\
    & \leq \mu_\epsilon \Big ( 2\mu_\epsilon + \tfrac{1}{\tau_\M} (|\zeta| + K_\M + D_x) \Big ) \varepsilon_0^2 + \mu_\epsilon \Big ( 5 \mu_\epsilon \gap + 6 |\zeta| + 4 K_\M + 4 D_x \Big ) \gap - \tfrac{3}{4} \|\Pieps(Y) - \mu_\epsilon \Pi_\M(x_j)\|^2\\
    & \qquad + 2 (K_\M + D_x + 2 |\zeta|) \big ( \tfrac{1}{2 \mu_\epsilon \tau_\M}\|\Pieps(Y) - \mu_\epsilon \Pi_\M(x_j)\|^2 \big ) \wedge \|\Pieps(Y) - \mu_\epsilon \Pi_\M(x_j)\|\\
    & \leq \mu_\epsilon \Big ( 2\mu_\epsilon + \tfrac{1}{\tau_\M} (|\zeta| + K_\M + D_x) \Big ) \varepsilon_0^2 + \mu_\epsilon \Big ( 5 \mu_\epsilon \gap + 6 |\zeta| + 4 K_\M + 4 D_x \Big ) \gap - \tfrac{1}{2} \|\Pieps(Y) - \mu_\epsilon \Pi_\M(x_j)\|^2\\
    & \qquad + 2 \mathbbm{1}_{\mathcal{E}_x} (K_\M + D_x + 2 |\zeta|) \|\Pieps(Y) - \mu_\epsilon \Pi_\M(x_j)\|,
\end{align*}
where the indicator function contains the event \(\mathcal{E}_x = \{ K_\M + D_x + 2 |\zeta| \geq \mu_\epsilon \tau_\M / 4 \}\).

Using this bound, we choose a value of \(\varepsilon_1\) that guarantees that the contents of the exponential function in \eqref{eq:drbg_3} is negative.
By solving the quadratic, it follows that to have \(\|\tilde{Y} - \mu_\epsilon x_i\|^2-\|\tilde{Y} - \mu_\epsilon x_j\|^2 + |\Delta_i| + 2|\Delta_j| \leq - \mu_\epsilon \kappa\), for some \(\kappa > 0\), it is sufficient to have \(\|\Pieps(Y) - \Pieps(\mu_\epsilon x_j)\| \geq \mu_\epsilon \varepsilon_1\) with,
\begin{align}
    \mu_\epsilon^2 \varepsilon_1^2 &= 6 \mu_\epsilon \kappa + 6 \mu_\epsilon \Big ( 2\mu_\epsilon + \tfrac{1}{\tau_\M} (|\zeta| + K_\M + D_x) \Big ) \varepsilon_0^2 + 6 \mu_\epsilon \Big ( 5 \mu_\epsilon \gap + 6 |\zeta| + 4 K_\M + 4 D_x \Big ) \gap \nonumber\\
    & \qquad + 8 \mathbbm{1}_{\tilde{\mathcal{E}}_x} (K_\M + D_x + 2 |\zeta|)^2.\label{eq:drbg_2}
\end{align}

Substituting this into \eqref{eq:drbg_3}, we then obtain the following bound for \(\mathtt{B}\),
\begin{equation*}
    \E[\mathtt{B}|S] \leq \E \bigg [ \frac{|I_1^\complement|}{|I_0|} \bigg | S \bigg ] \exp(-\mu_\epsilon\kappa/2\sigma^2_\epsilon).
\end{equation*}
We now return to bounding \(\mathtt{A}\) with this choice of \(\varepsilon_1\)
to obtain that,
\begin{align*}
    \E[\mathtt{A}|S] &\leq \frac{\mu_\epsilon K_\M}{\sigma^2_\epsilon} \bigg ( \frac{1}{\tau_\M \mu_\epsilon^2} \bigg ( 6 \mu_\epsilon \kappa + 6 \mu_\epsilon \Big ( 2\mu_\epsilon + \tfrac{3}{\tau_\M} K_\M + \tfrac{1}{\tau_\M} D_x \Big ) \varepsilon_0^2 + 6 \mu_\epsilon \Big ( 5 \mu_\epsilon \gap + 16 K_\M + 4 D_x \Big ) \gap \\
    & \qquad + 8 \prob(\mathcal{E}_x|S)^{1/2} (5 K_{\max, \M}^{1/2} K_\M^{1/2} + D_x)^2 \bigg ) + 2 \gap \bigg ),
\end{align*}
where we utilise the bounds in Lemma \ref{lem:zeta_bound} to control \(|\zeta|\).
We then optimise \(\kappa\) by choosing,
\begin{equation*}
    \kappa = \frac{2 \sigma_\epsilon^2}{\mu_\epsilon} \log \bigg ( \E \Big [ |I_1^\complement|/|I_0| \Big | S \Big ] \frac{\tau_\M}{K_\M} \bigg )_+,
\end{equation*}
which produces the bound,
\begin{align*}
    \E[|\Delta \lse(x)|S] &\leq \frac{K_\M}{\tau_\M} \bigg ( 1 + 12 \log \Big ( \E \Big [ |I_1^\complement|/|I_0| \Big | S \Big ] \tfrac{\tau_\M}{K_\M} \Big )_+ + \Big ( 2 + \tfrac{3}{\tau_\M} K_\M + \tfrac{1}{\tau_\M} D_x \Big ) \frac{6\varepsilon_0^2}{\sigma_\epsilon^2} \\
    & \quad + \Big ( 5 \gap + 16 K_\M + 4 D_x \Big ) \frac{6 \gap}{\sigma_\epsilon^2} +\frac{8 \prob(\mathcal{E}_x|S)^{1/2}}{\sigma_\epsilon^2} (5 K_{\max, \M}^{1/2} K_\M^{1/2} + D_x)^2 \bigg ) \bigg ) + \frac{2 K_\M \gap}{\sigma_\epsilon^2},
\end{align*}
where we have used that \(\mu_\epsilon \leq 1\) to simplify the expression.

To conclude the proof, we use the fact that \(|I_0| = N \hat{\mu}_{\operatorname{data}}(B_{\varepsilon_0}(\Pi_\M(Y/\mu_\epsilon))\) and \(|I_1^\complement| \leq N\). Then, optimising over \(\varepsilon_0\) leads to the bound in the statement.
\end{proof}

\subsection{Manifold concentration under log-domain smoothing}
The bound on \(\Delta \lse_\M\) developed in the previous section depends on the distance to the manifold, \(D_x\). Since, for the R\'enyi divergence, we will integrate \(\Delta \lse_\M\) with respect to \(\hat{p}^{k^\M}_\epsilon\), we must develop some bounds on the concentration of this measure to the manifold. Due to the complexity of log-domain smoothing, this is non-trivial and relies on Weyl's formula for the volume of tubular neighbourhoods. In the following lemma, we develop such a concentration inequality.

\begin{lemma}\label{lem:man_conc_simple}
Let \(\delta, \varepsilon > 0\) such that \(\essinf_{x \in \M} (\Pi_{\M})_* \hmud(B_\varepsilon(x)) \geq \delta\), then for all \(r^2 \geq 4 \sigma_\epsilon^2 d\), we obtain the bound,
\begin{equation*}
    \prob_{Z \sim \hat{p}_\epsilon^{k^\M}}(\dist(Z, \M^\epsilon) \geq r|S) \leq \exp \bigg ( d\log(8) + \frac{5(K_\M^2 + \mu_\epsilon^2 \gap)^2 + 4\mu_\epsilon^2\varepsilon^2}{2\sigma_\epsilon^2} + 2 \delta^{-1} - \frac{(r - \sqrt{2\sigma_\epsilon^2 d})^2}{4 \sigma_\epsilon^2} \bigg ).
\end{equation*}
\end{lemma}
\begin{proof}
Letting \(Z \sim \hat{p}_\epsilon^{k^\M}\), we begin by expressing the probability in integral form,
\begin{equation*}
    \prob(\dist(Z, \M^\epsilon) \geq r) = C_\M^{-1} \int_{\dist(\cdot, \M) \geq r} \exp \bigg ( \int \log \hat{p}_\epsilon(y) k_x^\M(dy) \bigg ) dx.
\end{equation*}
From the formulation of \(\log \hat{p}_\epsilon\) in \eqref{eqn:emp_density_lse}, we readily obtain that
\begin{align}
    \log \hat{p}_\epsilon(y) &\leq -\frac{\min_i \|y - \mu_\epsilon x_i\|^2}{2 \sigma_\epsilon^2} + \log(N) + C_\epsilon \nonumber \\
    &\leq -\frac{\min_i \{ (\|y - \mu_\epsilon \Pi_\M(x_i)\| - \|\mu_\epsilon \Pi_\M(x_i) - \mu_\epsilon x_i\|)_+\}^2}{2 \sigma_\epsilon^2} + \log(N) + C_\epsilon \nonumber\\
    &\leq - \frac{(\dist(y, \M^\epsilon) - \mu_\epsilon \gap)_+^2}{2 \sigma_\epsilon^2} - \frac{d}{2}\log(2 \pi \sigma_\epsilon^2),\label{eq:man_conc_simple_2}
\end{align}
where we use Young's inequality in the second inequality. Using the fact that \(Z \in \M^\epsilon_{r(x)}\), we obtain the following lower bound:
\begin{align}
    \dist(&Z, M^\epsilon) \nonumber\\
    &= \E_{Y \sim k_x}[\dist(Y, \M^\epsilon)^2]^{1/2} \nonumber\\
    &\geq \dist(x, \M^\epsilon) - \E_{Y \sim k_x}[(\dist(Y, \M^\epsilon) - \dist(x, \M^\epsilon))^2]^{1/2} \nonumber\\
    &\geq \dist(x, \M^\epsilon) - K_\M.\label{eq:man_conc_simple_3}
\end{align}
Thus, combining \eqref{eq:man_conc_simple_2} and \eqref{eq:man_conc_simple_3}, we obtain the bound,
\begin{align}
    \int_{\dist(\cdot, \M) \geq r} & \exp \bigg ( \int \log \hat{p}_\epsilon(y) k_x^\M(dy) \bigg ) dx \nonumber\\
    &\leq (2 \pi \sigma_\epsilon^2)^{-d/2} \int_{\dist(\cdot, \M) \geq r} \exp \bigg ( - \frac{(\dist(x, \M^\epsilon) - \mu_\epsilon \gap - K_\M)_+^2}{2 \sigma_\epsilon^2} \bigg ) dx \nonumber\\
    &\leq (2 \pi \sigma_\epsilon^2)^{-d/2} \int_{\dist(\cdot, \M) \geq r} \exp \bigg ( - \frac{\dist(x, \M^\epsilon)^2}{4\sigma_\epsilon^2} + \frac{(K_\M + \mu_\epsilon \gap)^2}{2\sigma_\epsilon^2} \bigg ) dx.\label{eq:man_conc_simple_5}
\end{align}

Next, we lower bound \(C_\M\). For this, we use Lemma \ref{lem:lse_properties} with the parameters,
\begin{equation*}
    \epsilon_i = -\frac{(\dist(y, \M^\epsilon)^2 - \|y - \mu_\epsilon x_i\|^2}{\sigma_\epsilon^2},
\end{equation*}
which produces the expression,
\begin{align*}
    \log \hat{p}_\epsilon(y) &\geq\lse \bigg ( \bigg \{ - \frac{(\dist(y, \M^\epsilon) + (\|y - \mu_\epsilon \Pi_\M(x_i)\|-\dist(y, \M^\epsilon)) + \mu_\epsilon \gap)^2}{2\sigma_\epsilon^2} \bigg \}_i \bigg ) + C_\epsilon\\
    &\geq -\frac{(\dist(y, \M^\epsilon) + \mu_\epsilon \gap)^2}{\sigma_\epsilon^2} + \int_0^1 \frac{\sum_{i=1}^N\exp(r\epsilon_i)\epsilon_i}{\sum_{i=1}^N \exp(r\epsilon_i)} + \log(N) + C_\epsilon.
\end{align*}
We control this further using a similar technique to that used in the proof of Theorem \ref{thm:density_ratio_bound_simplified}. We define the sets \(I_0 = \{i \in [N]: \epsilon_i \geq -\mu_\epsilon^2 \varepsilon^2/\sigma_\epsilon^2\}, I_1 = \{i \in [N]: \epsilon_i \geq -2\mu_\epsilon^2\varepsilon^2/\sigma_\epsilon^2\}\) for any quantity \(\varepsilon > 0\). With this, we obtain the following bound,
\begin{align*}
    \log \hat{p}_\epsilon(y) &= -\frac{(\dist(y, \M^\epsilon) + \mu_\epsilon \gap)^2}{\sigma_\epsilon^2} + \int_0^1 \frac{\sum_{i \in I_1} \exp(r\epsilon_i)\epsilon_i}{\sum_{i \in I_1} \exp(r\epsilon_i)} dr + \int_0^1 \frac{\sum_{i \in I_1^\complement} \exp(r\epsilon_i)\epsilon_i}{\sum_{i \in I_0} \exp(r\epsilon_i)} dr - \frac{d}{2}\log(2 \pi \sigma_\epsilon^2)\\
    &\geq -\frac{(\dist(y, \M^\epsilon) + \mu_\epsilon \gap)^2}{\sigma_\epsilon^2} + \min_{i \in I_1} \epsilon_i + \frac{|I_1^\complement|}{|I_0|} \int_0^1 \min_{i \in I_1^\complement, j \in I_0}\exp(r(\epsilon_i - \epsilon_j))\epsilon_i dr - \frac{d}{2}\log(2 \pi \sigma_\epsilon^2)\\
    &\geq -\frac{(\dist(y, \M^\epsilon) + \mu_\epsilon \gap)^2}{\sigma_\epsilon^2} - \frac{2 \mu_\epsilon^2 \varepsilon^2}{\sigma_\epsilon^2} - \frac{|I_1^\complement|}{|I_0|} \int_0^1 \exp \bigg ( -\frac{r\varepsilon^2}{\sigma_\epsilon^2} \bigg ) \frac{2\varepsilon^2}{\sigma_\epsilon^2} dr - \frac{d}{2}\log(2 \pi \sigma_\epsilon^2).
\end{align*}
This is further controlled using the fact that \(I_0 \supseteq \{i \in [N]: \|\Pi_\M(x_i) - \Pi_\M(y/\mu_\epsilon)\|^2 \leq \varepsilon^2\}\) and \(|I_1^\complement| \leq N\) and hence
\begin{equation*}
    \frac{|I_1^\complement|}{|I_0|} \leq \frac{N}{|\{\Pi_\M(x_i)\}_{i=1}^N \cap B_\varepsilon(\Pi_\M(y/\mu_\epsilon))|} \leq \hmud(B_\varepsilon(\Pi_\M(y/\mu_\epsilon)))^{-1} \leq \delta^{-1},
\end{equation*}
where the last inequality holds almost surely. In combination with the bound \(\E[\dist(Z, M^\epsilon)^2] \leq 2 \dist(x, \M^\epsilon)^2 + 2K_\M^2\), we obtain that,
\begin{equation}\label{eq:log_bound}
    \log \hat{p}_\epsilon(y) \geq -\frac{(\dist(y, \M^\epsilon) + K_\M + \mu_\epsilon \gap)^2}{\sigma_\epsilon^2} - \frac{2\mu_\epsilon^2\varepsilon^2}{\sigma_\epsilon^2} - 2 \delta^{-1} - \frac{d}{2}\log(2 \pi \sigma_\epsilon^2).
\end{equation}
With this, we lower bound \(C_\M\) by,
\begin{align}
    C_\M &= \int \exp \bigg ( \int \log \hat{p}_\epsilon(y) k_x^\M(dy) \bigg ) dx \nonumber\\
    &\geq (2 \pi \sigma_\epsilon^2)^{-d/2} \int \exp \bigg ( -\frac{2\dist(x, \M^\epsilon)^2 + 2(K_\M + \mu_\epsilon \gap)^2}{\sigma_\epsilon^2} - \frac{2\mu_\epsilon^2\varepsilon^2}{\sigma_\epsilon^2} - 2 \delta^{-1} \bigg ) dx.\label{eq:man_conc_simple_4}
\end{align}

Before combining the bounds in \eqref{eq:man_conc_simple_5} and \eqref{eq:man_conc_simple_4} we further simplify their expressions using Weyl's formula (see Section \ref{app:weyl}). Combining the bound in \eqref{eq:man_conc_simple_5} with the integral formula in \eqref{eq:weyl_int}, we obtain the bound,
\begin{align*}
    &\int_{\dist(\cdot, \M) \geq r} \exp \bigg ( \int \log \hat{p}_\epsilon(y) k_x^\M(dy) \bigg ) dx\\
    & \qquad \leq (2 \pi \sigma_\epsilon^2)^{-d/2} \exp \bigg ( \frac{(K_\M + \mu_\epsilon \gap)^2}{2\sigma_\epsilon^2} \bigg ) \sum_{p=0}^{\lfloor d^*/2 \rfloor} (d-d^*+2p) \tilde{k}_{2p}(\M^\epsilon) \int_r^\infty \exp \bigg ( - \frac{s^2}{4\sigma_\epsilon^2} \bigg ) s^{d-d^*-1+2p} ds.
\end{align*}
The integral on the right-hand side can be analysed by relating it to the measure of a related spherically symmetric measure (e.g. see equation (4) in \cite{Bobkov2003-et}). With this we relate the integral to the concentration of a Gaussian random variable:
\begin{align*}
    \int_r^\infty s^{k-1} \exp(-s^2/4\sigma_\epsilon^2) ds &= \frac{\Gamma(k/2 + 1)}{k \pi^{k/2}} \int_{\{x \in \R^k: \|x\| \geq r\}} \exp(-\|x\|^2/4 \sigma_\epsilon^2) dx\\
    &= \frac{\Gamma(k/2 + 1) (4 \sigma_\epsilon^2)^{k/2}}{k} \prob_{\xi \sim N(0, 2\sigma_\epsilon^2I_k)}(\|\xi\| \geq r).
\end{align*}
Thus, we obtain the expression,
\begin{gather*}
    \int_{\dist(\cdot, \M) \geq r} \exp \bigg ( \int \log \hat{p}_\epsilon(y) k_x^\M(dy) \bigg ) dx  = (2 \pi \sigma_\epsilon^2)^{-d/2} \exp \Big ( \tfrac{(K_\M + \mu_\epsilon \gap)^2}{2\sigma_\epsilon^2} \Big ) \sum_{p=0}^{\lfloor d^*/2 \rfloor} w_p \prob_{\xi \sim N(0, 2\sigma_\epsilon^2I_{d-d_\M^*+2p})}(\|\xi\| \geq r)\\
    w_p := \frac{\Gamma((d-d^*_\M+2p)/2 + 1) (4 \sigma_\epsilon^2)^{(d-d^*_\M+2p)/2}}{d-d_\M^*+2p} (d-d^*_\M+2p) \tilde{k}_{2p}(\M^\epsilon).
\end{gather*}
By a similar argument, we also obtain
\begin{align*}
    C_\M \geq (2 \pi \sigma_\epsilon^2)^{-d/2} \exp \bigg ( -\frac{4(K_\M + \mu_\epsilon \gap)^2 + 4\mu_\epsilon^2 \varepsilon^2}{2\sigma_\epsilon^2} - 2\delta^{-1} \bigg ) \sum_{p=0}^{\lfloor d^*/2 \rfloor} 64^{-(d-d_\M^* + 2p)/2} w_p.
\end{align*}
Dividing the two, we obtain the bound,
\begin{align*}
    \prob_{Z \sim \hat{p}_\epsilon^{k^\M}}&(\dist(Z, \M^\epsilon) \geq r)\\
    &\leq \frac{\sum_{p=0}^{\lfloor d^*/2 \rfloor} w_p \prob_{\xi \sim N(0, 2\sigma_\epsilon^2I_{d-d_\M^*+2p})}(\|\xi\| \geq r) }{\sum_{p=0}^{\lfloor d_\M^*/2 \rfloor} 64^{-(d-d_\M^*+2p)/2} w_p} \exp \bigg ( \frac{5(K_\M^2 + \mu_\epsilon^2 \gap)^2}{2\sigma_\epsilon^2} + \frac{2\mu_\epsilon^2\varepsilon^2}{\sigma_\epsilon^2} + 2 \delta^{-1} \bigg )\\
    &\leq 8^d \exp \bigg ( \frac{5(K_\M^2 + \mu_\epsilon^2 \gap)^2}{2\sigma_\epsilon^2} + \frac{2\mu_\epsilon^2\varepsilon^2}{\sigma_\epsilon^2} + 2 \delta^{-1} \bigg ) \max_{p \in \{0, ..., \lfloor d^*/2 \rfloor\}} \prob_{\xi \sim N(0, 2\sigma_\epsilon^2I_{d-d^*+2p})}(\|\xi\| \geq r).
\end{align*}
We then bound this further using the concentration of the chi-squared distribution (see Example 2.28 of \citep{Wainwright2019-rz}), obtaining
\begin{align*}
    \prob_{\xi \sim N(0, 2\sigma_\epsilon^2I_k)}(\|\xi\| \geq r) &\leq \exp \bigg ( - \frac{(r-\sqrt{2\sigma_\epsilon^2 k})^2}{4 \sigma_\epsilon^2 d} \bigg )\\
    &\leq \exp \bigg ( - \frac{(r-\sqrt{2\sigma_\epsilon^2 d})^2}{4 \sigma_\epsilon^2 d} \bigg ),
\end{align*}
completing the proof of the bound.
\end{proof}

\subsection{Proof of Theorem \ref{thm:density_ratio_bound_simplified}}
With the pointwise bounds on $\E[\Delta \lse_\M(x)|S]$ from the previous subsection, we are now prepared to derive the R\'enyi divergence bound in Theorem \ref{thm:density_ratio_bound_simplified}.

\begin{proof}[Proof of Theorem \ref{thm:density_ratio_bound_simplified}]
To bound the R\'enyi divergence, we begin with the following expression which follows from \eqref{eq:density_ratio_delta_2}:
\begin{align*}
    D_q(\hat{p}_\epsilon^{k^\M} \| \hat{p}_\epsilon^k) &= \frac{1}{q - 1} \log \int \bigg ( \frac{\hat{p}_\epsilon^{k^\M}(x)}{\hat{p}_\epsilon^k(x)} \bigg )^{q - 1} \hat{p}_\epsilon^{k^\M}(dx)\\
    &= \frac{1}{q - 1} \log \int \exp((q-1) \E[\Delta \lse_\M(x)|S]) \, \hat{p}_\epsilon^{k^\M}(dx) + \log (C^\M/C),
\end{align*}
where the normalisation constants, \(C\) and \(C^\M\), are as defined in \eqref{eq:pf_normalisation}. Furthermore, we obtain the following relationship between the normalisation constants:
\begin{align*}
    C &= \int \exp \bigg ( \int \log \hat{p}_\epsilon(y) k_x(dy) - \int \log \hat{p}_\epsilon(y) k^\M_x(dy) + \int \log \hat{p}_\epsilon(y) k^\M_x(dy) \bigg ) dx\\
    &= C^\M \int \exp (\E[\Delta \lse_\M(x)|S]) \hat{p}_\epsilon^{k^\M}(dx).
\end{align*}
Therefore, using Jensen's inequality, we deduce the bound,
\begin{align}
    D_q(\hat{p}_\epsilon^k \| \hat{p}_\epsilon^{k^\M}) &\leq \frac{2}{\beta} \log \int \exp (\beta |\E[\Delta \lse_\M(x)|S]|) \hat{p}_\epsilon^{k^\M}(dx) \nonumber\\
    &= \frac{2}{\beta} \log \E[\exp (\beta | \E[\Delta \lse_\M(Z)|S, Z] |) | S],\label{eq:thm_11}
\end{align}
where, for the sake of brevity, we use the shorthand \(\beta = (q - 1) \vee 1\) and \(Z \sim \hat{p}_\epsilon^{k^\M}(dx)\).

We proceed by applying the bound on \(\Delta \lse_\M\) developed in Lemma \ref{lem:tighter_delta_bound}. The assumptions of Lemma \ref{lem:tighter_delta_bound} hold with \(\gap = 0\) and hence, we have the bound,
\begin{align}
    \E[|\Delta \lse_\M(Z)|S, Z] &\leq \frac{16 K}{\tau} \bigg ( 1 + \log \Big ( \frac{\tau}{K} \Big )_+ + \frac{(5K_{\max}^{1/2} K^{1/2} + D_Z)^2}{\sigma_\epsilon^2} \mathbbm{1}_{5 K_{\max} + D_Z \geq \mu_\epsilon \tau/4} \nonumber\\
    & \qquad + \inf_{\varepsilon_0 > 0} \bigg \{ \log \Big ( \E [ \hat{\mu}_{\operatorname{data}}(B_{\varepsilon_0}(\Pi_\M(Y/\mu_\epsilon)))^{-1} | S, Z ] \Big )_+ + \big ( 1 + \tfrac{K + D_Z}{\tau} \big ) \frac{\varepsilon_0^2}{\sigma_\epsilon^2} \bigg \} \bigg ),\label{eq:thm_10}
\end{align}
where we have used the fact that \(|\zeta| \leq 2 K_{\max}\), (see Lemma \ref{lem:zeta_bound}) and therefore, \(\prob(\mathcal{E}_Z|S, Z)^{1/2} \leq \mathbbm{1}_{5 K_{\max} + D_Z \geq \mu_\epsilon \tau / 4}\). To control the infimum term, we utilise the bound on balls of \(\hmud\) given in Lemma \ref{lem:lower_emp_meas}. In this lemma, it is shown that whenever \(\varepsilon_0 \leq \tau\), with probability \(1-\delta\), we have,
\begin{equation}\label{eq:geometric_bias_3}
    \sup_{y \in \M} \hat{\mu}_{\operatorname{data}}(B_{\varepsilon_0}(y))^{-1} \leq c_\mu^{-1} \varepsilon_0^{-d^*} / 4,
\end{equation}
once \(N\) is sufficiently large so that the condition in \eqref{eq:lower_emp_meas_1} is satisfied with \(r = \varepsilon_0\). Indeed, there exists a value \(\varepsilon_0 \in (0, \tau/64]\) satisfying this condition as soon as we assume that,
\begin{equation}
    N \geq 64 \Big ( (144 d^* \log(128/\tau) + 144 \log(2/c_\mu)) \vee \tfrac{\log(\delta^{-1})}{2} \Big ) (\tau/64)^{-2d^*} c_\mu^{-2} =: N_{\min}(d^*, \tau, \delta, c_\mu).\label{eq:N_min}
\end{equation}
If we set \(\varepsilon_0^2 = d^* \sigma_\epsilon^2\) and require that \(\sigma_\epsilon^2 \leq (\tau/64)^2/d^*\), then once \(N\) is sufficiently large so that \eqref{eq:lower_emp_meas_1} is satisfied, we obtain the bound,
\begin{align}
    &\inf_{\varepsilon_0 > 0} \bigg \{ \log \Big ( \E_{Y \sim k_Z} [ \hat{\mu}_{\operatorname{data}}(B_{\varepsilon_0}(\Pi_\M(Y/\mu_\epsilon)))^{-1} | S, Z ] \Big )_+ + \big ( 1 + \tfrac{K + D_Z}{\tau} \big ) \frac{\varepsilon_0^2}{\sigma_\epsilon^2} \bigg \}\nonumber\\
    & \qquad \qquad \leq \inf_{\varepsilon_0 \in (0, \tau_\M/64]} \sup_{y \in \M} \bigg \{ \log \big ( \hat{\mu}_{\operatorname{data}}(B_{\varepsilon_0}(y))^{-1} \big )_+ + \big ( 1 + \tfrac{K + D_Z}{\tau} \big ) \frac{\varepsilon_0^2}{\sigma_\epsilon^2} \bigg \}\nonumber\\
    & \qquad \qquad \lesssim \Big ( 1 + \tfrac{K + D_Z}{\tau} \Big ) d^*.\label{eq:thm_8}
\end{align}
Indeed, this would require \(N \gtrsim (d^* + 1) c_\mu^{-2} (d^* \sigma_\epsilon^2)^{-d^*}\). If \(N\) does not satisfy this, then we instead set \(\varepsilon_0\) to be the smallest \(r\) such that \eqref{eq:lower_emp_meas_1} is satisfied, arriving at a quantity with \((c_\mu^2 N)^{-1/d^*} \lesssim \varepsilon_0^2 \lesssim (c_\mu^2 N)^{-1/d^*}\) whenever \(d^* > 0\), and \(\varepsilon^2_0 = 0\) whenever \(d^* = 0\). Thus, we obtain the bound,
\begin{align}
    &\inf_{\varepsilon_0 > 0} \bigg \{ \log \Big ( \E_{Y \sim k_Z} [ \hat{\mu}_{\operatorname{data}}(B_{\varepsilon_0}(\Pi_\M(Y/\mu_\epsilon)))^{-1} | S, Z ] \Big )_+ + \big ( 1 + \tfrac{K + D_Z}{\tau} \big ) \frac{\varepsilon_0^2}{\sigma_\epsilon^2} \bigg \} \lesssim 1 + \big ( 1 + \tfrac{K + D_Z}{\tau} \big ) \tfrac{(c_\mu^2 N)^{-1/d^*}}{\sigma_\epsilon^2} \mathbbm{1}_{d^* > 0}.\label{eq:thm_7}
\end{align}
We can then combine the bounds in \eqref{eq:thm_8} and \eqref{eq:thm_7}, to obtain that there exists a quantity \(C_2 > 0\) that depends only logarithmically on structural parameters and satisfies,
\begin{equation}\label{eq:thm_9}
    2\inf_{\varepsilon_0 > 0} \bigg \{ \log \Big ( \E[ \hat{\mu}_{\operatorname{data}}(B_{\varepsilon_0}(\Pi_\M(Y/\mu_\epsilon)))^{-1} | S, Z ] \Big )_+ + \big ( 1 + \tfrac{K + D_Z}{\tau} \big ) \frac{\varepsilon_0^2}{\sigma_\epsilon^2} \bigg \} \leq C_2 \bigg ( 1 + \Big ( 1 + \tfrac{D_Z}{\tau} \Big ) d^* \vee ((c_\mu^2 N)^{-\frac{1}{d^*}} \sigma_\epsilon^{-2}) \bigg ),
\end{equation}
where we have also used the fact that \(K \leq K_{\max} \leq \tau/96\).

Returning to bounding \eqref{eq:thm_11}, we use \eqref{eq:thm_10} and \eqref{eq:thm_9} to derive the upper bound,
\begin{align}
    &\frac{2}{\beta} \log \E[\exp (\beta | \E[\Delta \lse_\M(Z)|S, Z] |) | S]\nonumber\\
    & \quad \leq \frac{2}{\beta} \log \E \Big [ \exp \Big ( \tfrac{8 \beta K}{\tau} \Big ( 1 + 2\log \big ( \tfrac{\tau}{K} \big )_+ + \tfrac{(5K_{\max}^{1/2} K^{1/2} + D_Z)^2}{\sigma_\epsilon^2} \mathbbm{1}_{5 K_{\max} + D_Z \geq \mu_\epsilon \tau / 4} + C_2 \nonumber\\
    & \qquad \quad + C_2 \Big ( 1 + \tfrac{D_Z}{\tau} \Big ) d^* \vee \tfrac{(c_\mu^2 N)^{-\frac{1}{d^*}}}{\sigma_\epsilon^{2}} \Big ) \Big ) \Big | S \Big ]\nonumber\\
    & \quad \leq \frac{16 K}{\tau} \bigg ( 1 + 2 \log \big ( \frac{\tau}{K} \big ) + C_2 + C_2 d^* \vee \tfrac{(c_\mu^2 N)^{-\frac{1}{d^*}}}{\sigma_\epsilon^{2}} \bigg ) + \frac{1}{20\beta} \log \E \bigg [ \exp \bigg ( \tfrac{40\beta K}{\tau} \tfrac{C_2 D_Z}{\tau} d^* \vee \tfrac{(c_\mu^2 N)^{-\frac{1}{d^*}}}{\sigma_\epsilon^{2}} \bigg ) \bigg | S \bigg ]\nonumber\\
    & \quad \qquad + \frac{1}{5\beta} \log \E \bigg [ \exp \bigg ( \tfrac{10\beta K}{\tau} \frac{(5 K^{1/2} K_{\max}^{1/2} + D_Z)^2}{\sigma^2_\epsilon} \mathbbm{1}_{5 K_{\max} + D_Z \geq \mu_\epsilon \tau / 4} \bigg ) \bigg | S \bigg ],\label{eq:thm1_2}
\end{align}
where in the second inequality, we use H\"older's inequality.
We now bound the last two terms, starting with the second.

We utilise Lemma \ref{lem:man_conc_simple} which we apply with \(\varepsilon = \tau/64\). As a result of \eqref{eq:geometric_bias_3} and the assumed lower bound on \(N\), the assumptions of Lemma \ref{lem:man_conc_simple} are satisfied with \(\delta = c_\mu (\tau/64)^{d^*}/4\), and so, for any \(r^2 \geq 4 \sigma_\epsilon^2 d\),
\begin{align*}
    \prob(D_Z \geq r|S) \leq \exp \Big ( C - \tfrac{(r-\sqrt{2\sigma_\epsilon^2 d})^2}{4 \sigma_\epsilon^2} \Big ), \qquad C := \log(8) d + \frac{5K^2 + 2^{-10} \mu_\epsilon^2 \tau^2}{2\sigma_\epsilon^2} + 8 c_\mu^{-1} (\tau/64)^{-d^*}.
\end{align*}
Thus, for any \(c, R > 0\), we have the bound,
\begin{align*}
    \E \big [ \exp \big ( c D_Z \big ) \big | S \big ] &= \int_0^\infty \prob ( D_Z \geq \log(r)/c ) dr\\
    &\qquad \leq \exp(cR) + \int_{\exp(cR)}^\infty \prob ( D_Z \geq \log(r)/c ) dr\\
    &\qquad \leq \exp(cR) + \int_{\exp(cR)}^\infty \exp \bigg ( C - \frac{(\log(r)/c - \sqrt{2 \sigma_\epsilon^2 d})^2}{4 c^2 \sigma_\epsilon^2} \bigg ) dr.
\end{align*}
This is simplified using the change of variables,
\begin{align*}
    \int_{\exp(cR)}^\infty \exp \bigg ( C - \frac{(\log(r)/c - \sqrt{2 \sigma_\epsilon^2 d})^2}{4 c^2 \sigma_\epsilon^2} \bigg ) dr &= c \int_0^\infty \exp \bigg ( C - \frac{(u+R-\sqrt{2 \sigma_\epsilon^2 d})^2}{4 \sigma_\epsilon^2} + c(u + R) \bigg ) du.
\end{align*}
We then choose \(R := \sqrt{2 \sigma_\epsilon^2 d} + 16 c \sigma_\epsilon^2 + \sqrt{256 c^2 \sigma_\epsilon^4 + 32 \sigma_\epsilon^2 C}\) to further simplify the expression, obtaining,
\begin{align*}
    \int_{\exp(cR)}^\infty \exp \bigg ( C - \frac{(\log(r)/c - \sqrt{2 \sigma_\epsilon^2 d})^2}{32 c^2 \sigma_\epsilon^2} \bigg ) dr &= c \int_0^\infty \exp \bigg ( - \frac{u^2}{32 \sigma_\epsilon^2} + u \bigg ( c - \frac{R}{16 \sigma_\epsilon^2} \bigg ) + c \sqrt{2 \sigma_\epsilon^2 d} \bigg ) du\\
    &\leq c \exp(c \sqrt{2 \sigma_\epsilon^2 d}) \int_0^\infty \exp \bigg ( - \frac{u^2}{32 \sigma_\epsilon^2} \bigg ) du\\
    &= (32 \pi)^{1/2} \exp(c \sqrt{2 \sigma_\epsilon^2 d}) c \sigma_\epsilon,
\end{align*}
where the last line follows from the Gaussian integral. With this, we obtain a bound on the MGF:
\begin{align*}
    \log \E \big [ \exp \big ( c D_Z \big ) \big | S \big ] &\leq \log \big ( \exp(cR) + (32 \pi)^{1/2} \exp(c \sqrt{2 \sigma_\epsilon^2 d}) c \sigma_\epsilon \big )\nonumber\\
    &\leq cR + \log \big (1 + (32 \pi)^{1/2} \exp(c \sqrt{2 \sigma_\epsilon^2 d} - cR) c \sigma_\epsilon \big )\nonumber\\
    &\leq cR + (32 \pi)^{1/2} c \sigma_\epsilon\nonumber\\
    &\lesssim \sqrt{\sigma_\epsilon^2 d} + c^2 \sigma_\epsilon^2 + c(K + \tau) + c \sigma_\epsilon (\sqrt{d} + c_\mu^{-1/2} (\tau/64)^{-d^*/2}).
\end{align*}
Substituting values for \(c\) and \(R\), we obtain the bound,
\begin{align}
    \frac{1}{20\beta} &\log \E \bigg [ \exp \bigg ( \tfrac{40\beta K}{\tau} \tfrac{C_2 D_Z}{\tau} d^* \vee \tfrac{(c_\mu^2 N)^{-\frac{1}{d^*}}}{\sigma_\epsilon^{2}} \bigg ) \bigg | S \bigg ]\nonumber\\
    &\lesssim \sqrt{\sigma_\epsilon^2 d} + \frac{K}{\tau^2} \bigg ( d^* \vee \tfrac{(c_\mu^2 N)^{-\frac{1}{d^*}}}{\sigma_\epsilon^{2}} \bigg ) \bigg ( \tau + \sigma_\epsilon (\sqrt{d} + c_\mu^{-1/2} (\tau/64)^{-d^*/2}) + \sigma_\epsilon^2 \frac{K}{\tau^2} \bigg ( d^* \vee \tfrac{(c_\mu^2 N)^{-\frac{1}{d^*}}}{\sigma_\epsilon^{2}} \bigg ) \bigg )\nonumber\\
    &\lesssim \sqrt{\sigma_\epsilon^2 d} + \frac{K}{\tau} \bigg ( d^* \vee \tfrac{(c_\mu^2 N)^{-\frac{1}{d^*}}}{\sigma_\epsilon^{2}} \bigg ) \bigg ( 1 + \frac{\sigma_\epsilon}{\tau} (\sqrt{d} + c_\mu^{-1/2} (\tau/64)^{-d^*/2}) + \frac{K}{\tau^3} \bigg ( (\sigma_\epsilon^2 d^*) \vee \tau^2 \bigg ) \bigg ).\label{eq:thm1_7}
\end{align}
Thus, as soon as we require that,
\begin{equation*}
    \sigma_\epsilon^2 \leq \Big ( \frac{K^2}{d\tau^2} \Big ) \wedge \frac{\tau^2}{(\sqrt{d} + c_\mu^{-1/2} (\tau/64)^{-d^*/2})^2} \wedge \Big ( \frac{\tau^2}{d^*} \Big ) =: \sigma_{\max, 1}^2(d, d^*, \tau, K, c_\mu),
\end{equation*}
we obtain the bound,
\begin{equation}\label{eq:thm1_4}
    \frac{1}{20\beta} \log \E \bigg [ \exp \bigg ( \tfrac{40\beta K}{\tau} \tfrac{C_2 D_Z}{\tau} d^* \vee \tfrac{(c_\mu^2 N)^{-\frac{1}{d^*}}}{\sigma_\epsilon^{2}} \bigg ) \bigg | S \bigg ] \lesssim \frac{K}{\tau} \bigg ( (d^* + 1) \vee \tfrac{(c_\mu^2 N)^{-\frac{1}{d^*}}}{\sigma_\epsilon^{2}} \bigg ).
\end{equation}

We now bound the third term of \eqref{eq:thm1_2}. We once again use the integral formula for the expectation to obtain,
\begin{align}
    \E \bigg [ & \exp \bigg ( \tfrac{10\beta K}{\tau} \frac{(5 K^{1/2} K_{\max}^{1/2} + D_Z)^2}{\sigma^2_\epsilon} \mathbbm{1}_{5 K_{\max} + D_Z \geq \mu_\epsilon \tau / 4} \bigg ) \bigg | S \bigg ]\nonumber\\
    &\qquad = \int_0^\infty \prob \bigg ( \exp \bigg ( \tfrac{10\beta K}{\tau} \frac{(5 K^{1/2} K_{\max}^{1/2} + D_Z)^2}{\sigma^2_\epsilon} \mathbbm{1}_{5 K_{\max} + D_Z \geq \mu_\epsilon \tau / 4} \bigg ) \geq r \bigg | S \bigg ) dr\nonumber\\
    &\qquad \leq \int_0^\infty \prob \bigg ( (5 K_{\max} + D_Z)^2 \mathbbm{1}_{5 K_{\max} + D_Z \geq \mu_\epsilon \tau / 4} \geq  \frac{\sigma^2_\epsilon \tau}{10\beta K} \log(r) \bigg | S \bigg ) dr\nonumber\\
    &\qquad = 1 + \Big ( \exp \Big ( \tfrac{10\beta K}{\sigma^2_\epsilon \tau}(\mu_\epsilon \tau/4)^2 \Big ) - 1 \Big ) \prob ( D_Z \geq \mu_\epsilon \tau/4 - 5 K_{\max} | S ) \nonumber\\
    &\qquad \qquad + \int_{\exp \big ( \tfrac{10\beta K}{\sigma^2_\epsilon \tau}(\mu_\epsilon \tau/4)^2 \big )}^\infty \prob \bigg ( (5 K^{1/2} K_{\max}^{1/2} + D_Z)^2 \geq \frac{\sigma^2_\epsilon \tau}{10\beta K} \log(r) \bigg | S \bigg ) dr.\label{eq:thm1_6}
\end{align}
We further bound the last term using \(K \leq K_{\max}\) along with a change of variables, to derive,
\begin{align*}
    &\int_{\exp \big ( \tfrac{10\beta K}{\sigma^2_\epsilon \tau}( \mu_\epsilon \tau/4)^2 \big )}^\infty \prob \bigg ( (5 K_{\max} + D_Z)^2 \geq \frac{\sigma^2_\epsilon \tau}{10\beta K} \log(r) \bigg | S \bigg ) dr\\
    &\qquad = \frac{20\beta K} {\sigma^2_\epsilon \tau}\int_{\mu_\epsilon \tau/4}^\infty \prob  ( 5 K_{\max} + D_Z \geq u | S ) \exp \bigg ( \frac{10\beta K}{\sigma^2_\epsilon \tau} u^2 \bigg ) u du\\
    &\qquad \leq \frac{20\beta K} {\sigma^2_\epsilon \tau}\int_{\mu_\epsilon \tau/4}^\infty \exp \bigg ( C - \frac{1}{4 \sigma_\epsilon^2} \big ( u - 5 K_{\max} - \sqrt{2 \sigma_\epsilon^2 d} \big )^2 + \frac{10\beta K}{\sigma^2_\epsilon \tau} u^2 \bigg )u du
\end{align*}
Setting \(c = 1 - \frac{40\beta K}{\tau} > 1/2\), we can simplify this expression by,
\begin{align*}
    &\frac{20\beta K} {\sigma^2_\epsilon \tau}\int_{\mu_\epsilon \tau/4}^\infty \exp \bigg ( C - \frac{1}{4 \sigma_\epsilon^2} \big ( u - 5 K_{\max} - \sqrt{2 \sigma_\epsilon^2 d} \big )^2 + \frac{10\beta K}{\sigma^2_\epsilon \tau} u^2 \bigg )u du\\
    &\qquad = \frac{20\beta K}{\sigma^2_\epsilon \tau}\int_{\mu_\epsilon \tau/4}^\infty \exp \bigg ( C - \frac{c}{4\sigma_\epsilon^2} \bigg ( u - (5 K_{\max} + \sqrt{2 \sigma_\epsilon^2 d}) c^{-1} \bigg )^2 + \frac{1-c^{-1}}{4\sigma_\epsilon^2} (5 K_{\max} + \sqrt{2 \sigma_\epsilon^2 d})^2 \bigg )u du\\
    &\qquad \leq \frac{20\beta K}{\sigma^2_\epsilon \tau} (4 \sigma_\epsilon^2 \pi/c)^{1/2} \exp \bigg ( C - \frac{c}{4\sigma_\epsilon^2} \bigg ( \mu_\epsilon \tau/4 - (5 K_{\max} + \sqrt{2 \sigma_\epsilon^2 d})c^{-1} \bigg )^2 \bigg ) \sqrt{2 \sigma_\epsilon^2/c},
\end{align*}
where the final inequality follows from the concentration of the Gaussian random variable. Using the fact that \(K_{\max} \leq \tau/96\), \(c^{-1} \leq 2\), and assuming that \(\sigma_\epsilon^2 \leq \frac{K_{\max}^2}{2d}\),
\begin{align*}
    &\frac{20\beta K} {\sigma^2_\epsilon \tau}\int_{\mu_\epsilon \tau/4}^\infty \exp \bigg ( C - \frac{1}{4 \sigma_\epsilon^2} \big ( u - 5 K_{\max} - \sqrt{2 \sigma_\epsilon^2 d} \big )^2 + \frac{10\beta K}{\sigma^2_\epsilon \tau} u^2 \bigg )u du \\
    & \qquad \leq \frac{40\sqrt{2 \pi}\beta K}{\tau c} \exp \bigg ( C - \frac{1}{8\sigma_\epsilon^2} \big ( \mu_\epsilon \tau/8 \big )^2 \bigg )
\end{align*}
Similarly, we can bound the second term of \eqref{eq:thm1_6}, in total, obtaining,
\begin{align*}
    \frac{1}{5\beta} \log & \E \bigg [ \exp \bigg ( \tfrac{10\beta K}{\tau} \frac{(5 K^{1/2} K_{\max}^{1/2} + D_Z)^2}{\sigma^2_\epsilon} \mathbbm{1}_{5 K^{1/2} K_{\max}^{1/2}+D_Z \geq \mu_\epsilon \tau/4} \bigg ) \bigg | S \bigg ]\\
    &\lesssim \Big ( \exp \Big ( \tfrac{10\beta K}{\sigma^2_\epsilon \tau}(\mu_\epsilon \tau/4)^2 \Big ) - 1 \Big ) \prob ( D_Z \geq \mu_\epsilon \tau/4 - 5 K_{\max} | S ) \\
    &\qquad \qquad + \int_{\exp \big ( \tfrac{10\beta K}{\sigma^2_\epsilon \tau}(\mu_\epsilon \tau/4)^2 \big )}^\infty \prob \bigg ( (5 K^{1/2} K_{\max}^{1/2} + D_Z)^2 \geq \frac{\sigma^2_\epsilon \tau}{10\beta K} \log(r) \bigg | S \bigg ) dr\\
    &\lesssim \exp \bigg ( C + \tfrac{10\beta K}{\sigma^2_\epsilon \tau}(\mu_\epsilon \tau/4)^2 - \frac{1}{4\sigma_\epsilon^2} \big ( 3 \mu_\epsilon \tau/16 \big )^2 \bigg ) + \frac{K}{\tau} \exp \bigg ( C - \frac{1}{8\sigma_\epsilon^2} \big ( \mu_\epsilon \tau/8 \big )^2 \bigg )\\
    &\lesssim \exp \bigg ( C - \frac{1}{8\sigma_\epsilon^2} \big ( \mu_\epsilon \tau/8 \big )^2 \bigg ).
\end{align*}
Thus, by requiring that, 
\begin{equation*}
    \sigma_\epsilon^2 \leq 2^{10} \Big ( \log(8) d + 8 c_\mu^{-1} (\tau/64)^{-d^*} \Big ) \log (\tau/K) =: \sigma^2_{\max, 2}(d^*, d, \tau, K, c_\mu),
\end{equation*}
we obtain the upper bound,
\begin{equation}
    \frac{1}{5\beta} \log \E \bigg [ \exp \bigg ( \tfrac{10\beta K}{\tau} \frac{(5 K^{1/2} K_{\max}^{1/2} + D_Z)^2}{\sigma^2_\epsilon} \mathbbm{1}_{5 K^{1/2} K_{\max}^{1/2}+D_Z \geq \mu_\epsilon \tau/4} \bigg ) \bigg | S \bigg ] \lesssim \frac{K}{\tau}.\label{eq:thm1_5}
\end{equation}
Thus, by substituting \eqref{eq:thm1_4} and \eqref{eq:thm1_5} into \eqref{eq:thm1_2} we obtain that, with a probability of at least \(1-\delta\), we have,
\begin{equation*}
    \frac{2}{\beta} \log \E[\exp (\beta | \E[\Delta \lse_\M(Z)|S, Z] |) | S] \lesssim \frac{K}{\tau} \max \bigg \{ d^* + 1, \frac{(c_\mu^2 N)^{-\frac{1}{d^*}}}{\sigma_\epsilon^{2}} \bigg \},
\end{equation*}
completing the proof. We collect the required upper bounds on \(\sigma_\epsilon^2\), which when combined with the fact that \(\sigma_\epsilon^2 \leq \epsilon\), leads to the sufficient condition,
\begin{equation}\label{eq:epsilon_min}
    \epsilon \leq \sigma_{\max, 1}^2(d, d^*, \tau, K, c_\mu) \wedge \sigma^2_{\max, 2}(d^*, d, \tau, K, c_\mu) \wedge \frac{K^2_{\max}}{2d} \wedge \frac{\tau^2}{(64)^2d^*} =: \epsilon_{\max}(d, d^*, \tau, K, K_{\max}, c_\mu).
\end{equation}
\end{proof}

\subsection{Proof of Theorem \ref{thm:geometric_bias}}

\begin{proof}
Let \(\M \in \mathbbm{M}_\mu\) and let \(\tau_\M, d^*_\M, K_\M, K_{\max, \M}, \gap\) and \(c_{\mu, \M}\) be as defined in Section \ref{sec:geometric_bias_theory} and we assume \(\gap < \infty\) and \(K_{\max, \M} \leq \tau/96\). Using the same argument that produced \eqref{eq:thm_10}, we obtain the bound,
\begin{equation*}
    D_2(\hat{p}_\epsilon^k \| \hat{p}_\epsilon^{k^\M}) \leq 2 \log \E[\exp ( | \E[\Delta \lse_\M(Z)|S, Z] |) | S],
\end{equation*}
where \(Z \sim \hat{p}_\epsilon^{k^\M}(dx)\). Using Lemma \ref{lem:tighter_delta_bound}, we obtain pointwise bound,
\begin{align*}
    &\E[|\Delta \lse_\M(Z)|S, Z]\\
    & \qquad \leq \frac{8C_1 K_\M}{\tau_\M} \bigg ( 1 + 2\log \big ( \tfrac{\tau_\M}{K_\M} \big ) + \frac{(5 K_{\max, \M}^{1/2} K_\M^{1/2} + D_Z)^2}{\sigma_\epsilon^2} \mathbbm{1}_{D_Z \leq \mu_\epsilon \tau/4} + \Big (\gap + \tau_\M + K_\M + D_Z \Big ) \frac{12 \gap}{\sigma_\epsilon^2}\\
    & \qquad \qquad + 2\inf_{\varepsilon_0 > 0} \bigg \{ \log \Big ( \E [ \hat{\mu}_{\operatorname{data}}(B_{\varepsilon_0}(\Pi_\M(Y/\mu_\epsilon)))^{-1} | S, Z ] \Big ) + \big ( 1 + \tfrac{K_\M + D_Z}{\tau} \big ) \frac{\varepsilon_0^2}{\sigma_\epsilon^2} \bigg \} \bigg ).
\end{align*}
To bound the term with the infimum, we proceed similarly to the proof of Theorem \ref{thm:density_ratio_bound_simplified}, bounding it using Lemma \ref{lem:lower_emp_meas}. Given that \(N\) is sufficiently large, we obtain from this lemma that,
\begin{equation*}
    \sup_{y \in \M} \hat{\mu}_{\operatorname{data}}(B_{\varepsilon_0}(x))^{-1} \leq c_\mu^{-1} \varepsilon_0^{-d^*} / 4, \qquad \text{ for all } r \in [\sqrt{d^*\sigma_\epsilon^2} \wedge (K/2), \tau],
\end{equation*}
with probability \(1-\delta\). With this we can choose \(\varepsilon_0^2 = d^* \sigma_\epsilon^2\) to obtain the upper bound,
\begin{equation*}
    2\inf_{\varepsilon_0 > 0} \bigg \{ \log \Big ( \E_{Y \sim k_Z} [ \hat{\mu}_{\operatorname{data}}(B_{\varepsilon_0}(\Pi_\M(Y/\mu_\epsilon)))^{-1} | S, Z ] \Big ) + \big ( 1 + \tfrac{K_\M + D_Z}{\tau_\M} \big ) \frac{\varepsilon_0^2}{\sigma_\epsilon^2} \bigg \} \leq C_3\bigg ( 1 + \Big ( 1 + \tfrac{K_\M + D_Z}{\tau_\M} \Big ) d_\M^* \bigg ),
\end{equation*}
for some quantity \(C_3 > 0\) which depends only logarithmically on structural parameters. With this, we obtain the bound,
\begin{align}
    & 2\log \E[\exp ( | \E[\Delta \lse_\M(Z)|S, Z] |) | S] \nonumber\\
    & \qquad \lesssim \frac{K_\M}{\tau_\M} \bigg ( 1 + \Big (\gap + \tau_\M + K_\M\Big ) \frac{\gap}{\sigma_\epsilon^2} + d^* \bigg ) + \log \E \bigg [ \exp \bigg ( \tfrac{40 K_\M}{\tau_\M} \Big ( \tfrac{C_3}{\tau_\M} + \tfrac{\gap^2}{\sigma^2_\epsilon} \Big ) d_\M^* D_Z \bigg ) \bigg | S \bigg ]\nonumber\\
    & \qquad \qquad + \log \E \bigg [ \exp \bigg ( \tfrac{10 K}{\tau} \frac{(5 K^{1/2} K_{\max}^{1/2} + D_Z)^2}{\sigma^2_\epsilon} \mathbbm{1}_{5 K_{\max} + D_Z \geq \mu_\epsilon \tau / 4} \bigg ) \bigg | S \bigg ].\label{eq:geom_2}
\end{align}

We bound the second and third terms similarly to as in the proof of Theorem \ref{thm:density_ratio_bound_simplified}. However, the concentration of \(D_Z\) differs slightly due to the additional error from \(\gap > 0\). We apply Lemma \ref{lem:man_conc_simple} with \(\varepsilon = K/2\) to obtain that for any \(r^2 \geq 4 \sigma_\epsilon^2 d\),
\begin{align*}
    \prob(D_Z \geq r|S) \leq \exp \Big ( C - \tfrac{(r-\sqrt{2 \sigma_\epsilon^2 d})^2}{4 \sigma_\epsilon^2} \Big ),
\end{align*}
for all \(r^2 \geq 4 \sigma_\epsilon^2 d\), where the constant is given by
\begin{align*}
    C &= \big ( \tfrac{1}{8} + \log(32) \big ) d + \tfrac{5(K + \mu_\epsilon \gap)^2 + K^2}{2\sigma_\epsilon^2} + 8 c_\mu^{-1} (\tau/64)^{-d^*}\\
    &\leq  \big ( \tfrac{1}{8} + \log(32) \big ) d + \tfrac{6(K + \mu_\epsilon \gap)^2}{2\sigma_\epsilon^2} + 8 c_\mu^{-1} (\tau/64)^{-d^*}.
\end{align*}
To bound the third term of \eqref{eq:geom_2}, we can directly use the argument in the proof of Theorem \ref{thm:density_ratio_bound_simplified} that produces \eqref{eq:thm1_6}. Indeed, taking \(\sigma_\epsilon^2\) sufficiently small, we obtain,
\begin{equation*}
    \log \E \bigg [ \exp \bigg ( \tfrac{10 K_\M}{\tau_\M} \tfrac{D_Z^2}{\sigma^2_\epsilon} \mathbbm{1}_{D_Z \leq \mu_\epsilon \tau_\M/4} \bigg ) \bigg | S \bigg ] \lesssim \frac{K_\M}{\tau_\M}.
\end{equation*}
Similarly, we can borrow the argument that produces \eqref{eq:thm1_7}, to obtain
\begin{align*}
    &\log \E \bigg [ \exp \bigg ( \tfrac{40 K_\M}{\tau_\M} \Big ( \tfrac{C_3}{\tau_\M} + \tfrac{\gap^2}{\sigma^2_\epsilon} \Big ) d_\M^* D_Z \bigg ) \bigg | S \bigg ] \\
    &\lesssim \sqrt{\sigma_\epsilon^2 d} + c^2 \sigma_\epsilon^2 + cK_\M + c \sigma_\epsilon (\sqrt{d} + c_\mu^{-1/2} (\tau/64)^{-d^*/2})\\
    &\lesssim \sqrt{\sigma_\epsilon^2 d} + \frac{K_\M}{\tau_\M} \Big ( \tfrac{1}{\tau_\M} + \tfrac{\gap^2}{\sigma^2_\epsilon} \Big ) d_\M^* \bigg ( K_\M + \gap +  \sigma_\epsilon (\sqrt{d} + c_\mu^{-1/2} (\tau/64)^{-d^*/2}) + \sigma_\epsilon^2 \frac{K_\M}{\tau_\M} \Big ( \tfrac{1}{\tau_\M} + \tfrac{\gap^2}{\sigma^2_\epsilon} \Big ) d_\M^* \bigg ).
\end{align*}
With this, we obtain that as soon as \(\sigma_\epsilon^2\) is sufficiently small, we obtain,
\begin{equation*}
    \log \E \bigg [ \exp \bigg ( \tfrac{40 K_\M}{\tau_\M} \Big ( \tfrac{C_3}{\tau_\M} + \tfrac{\gap^2}{\sigma^2_\epsilon} \Big ) d_\M^* D_Z \bigg ) \bigg | S \bigg ] \lesssim \frac{K_\M^2}{\tau_\M} \bigg ( \frac{1}{\tau_\M} + \frac{\gap^2}{\sigma^2_\epsilon} \bigg ) d_\M^*.
\end{equation*}
Therefore, returning to \eqref{eq:geom_2} we obtain the bound,
\begin{align*}
    &2\log \E[\exp ( | \E[\Delta \lse_\M(Z)|S, Z] |) | S] \\
    & \qquad \leq \frac{K_\M}{\tau_\M} \bigg ( 1 + \Big (\gap + \tau_\M + K_\M\Big ) \frac{\gap}{\sigma_\epsilon^2} + d^* \bigg ) + \frac{K_\M^2}{\tau_\M} \bigg ( \frac{1}{\tau_\M} + \frac{\gap^2}{\sigma^2_\epsilon} \bigg ) d_\M^*\\
    & \qquad \lesssim \frac{K_\M (d^*+1)}{\tau_\M} + \frac{K_\M \gap}{\sigma_\epsilon^2} \bigg (1 + \frac{K_\M d^*_\M}{\tau_\M} \gap \bigg ).
\end{align*}
\end{proof}

\section{Proofs of other results}
This appendix contains the proofs for Corollary \ref{prop:manifold_concentration} and Proposition \ref{prop:distributed_on_manifold}. These results further elucidate the desirable properties of log-domain smoothing, specifically regarding how it concentrates mass near the manifold and distributes mass along it.

\subsection{Proof of Corollary \ref{prop:manifold_concentration}}

To prove this corollary, we utilise Lemma \ref{lem:man_conc_simple}. As remarked in the proof of Theorem \ref{thm:density_ratio_bound_simplified}, there exists a \(\varepsilon_0 \in [0, \tau/16]\) such that with a probability of at least \(1-\delta\), we obtain with probability bound,
\begin{align*}
    \prob_{Y \sim \hat{p}^{k^\M}_\epsilon} ( \operatorname{dist}(Y, \M) \geq r|S) \leq \exp \Big ( C - \tfrac{(r-\sqrt{2 \sigma_\epsilon^2 d})^2}{4 \sigma_\epsilon^2} \Big ), \qquad C := \log(8) d + \tfrac{5K^2 + 4\mu_\epsilon^2 \varepsilon_0^2}{2\sigma_\epsilon^2} + 8 c_\mu^{-1} \varepsilon_0^{-d^*},
\end{align*}
for any \(r^2 \geq 4 \sigma_\epsilon^2 d\), where \(\varepsilon_0\) must be chosen according to the size of \(N\) and the condition in \eqref{eq:lower_emp_meas_1}. In particular, we have the bound,
\begin{equation*}
    \prob_{Y \sim \hat{p}^{k^\M}_\epsilon} ( \operatorname{dist}(Y, \M) \geq r + \sqrt{2 \sigma_\epsilon^2 d} + \sqrt{4 \sigma_\epsilon^2 C^2} | S) \leq \exp \Big ( - \tfrac{r^2}{4 \sigma_\epsilon^2} \Big ).
\end{equation*}
When \(N\) is large enough, we choose the optimal value of \(\varepsilon_0^2 = \sigma_\epsilon^{\frac{4}{d^*+2}} c_\mu^{-\frac{2}{d^*+2}}\), so that,
\begin{equation*}
    \sigma_\epsilon^2 \bigg ( \tfrac{4\mu_\epsilon^2 \varepsilon_0^2}{2\sigma_\epsilon^2} + 8 c_\mu^{-1} \varepsilon_0^{-d^*} \bigg ) \lesssim \sigma_\epsilon^{\frac{4}{d^*+2}} c_\mu^{-\frac{2}{d^*+2}}.
\end{equation*}
If \(N\) is not sufficiently large, we choose \(\varepsilon_0\) to be the smallest value such that \eqref{eq:lower_emp_meas_1} is satisfied, yielding \((c_\mu^2 N)^{-1/d^*} \lesssim \varepsilon_0^2 \lesssim (c_\mu^2 N)^{-1/d^*}\) and also,
\begin{equation*}
    \sigma_\epsilon^2 \bigg ( \tfrac{4\mu_\epsilon^2 \varepsilon_0^2}{2\sigma_\epsilon^2} + 8 c_\mu^{-1} \varepsilon_0^{-d^*} \bigg ) \lesssim (c_\mu^2 N)^{-1/d^*}.
\end{equation*}
With these choices of \(\varepsilon_0\), we obtain the bound,
\begin{equation*}
    \sqrt{4 \sigma_\epsilon^2 C} \lesssim \sqrt{\sigma_\epsilon^2 d} + K + \max\{(c_\mu^2 N)^{-1/2d^*}, \sigma_\epsilon^{\frac{2}{d^*+2}} c_\mu^{-\frac{1}{d^*+2}}\}.
\end{equation*}
Next we transfer this concentration property to the measure \(\hat{p}_\epsilon^{k}\) by utilising the bound R\'enyi divergence bound in Theorem \ref{thm:density_ratio_bound_simplified}. By utilising Lemma 21 of \cite{Chewi2022-of}, we obtain the bound,
\begin{align*}
    \prob_{Y \sim \hat{p}^{k}_\epsilon} \Big ( \operatorname{dist}(Y, \M) \geq r + \sqrt{2 \sigma_\epsilon^2 d} + \sqrt{4 \sigma_\epsilon^2 C^2} + \sqrt{4 \sigma_\epsilon^2 D_2(\hat{p}^{k}_\epsilon \| \hat{p}^{k^\M}_\epsilon)} \Big | S \Big ) \leq 2 \exp \Big ( - \tfrac{r^2}{8 \sigma_\epsilon^2} \Big ).
\end{align*}
Finally we use Theorem \ref{thm:density_ratio_bound_simplified} to bound the \(2\)-R\'{e}nyi divergence with probability at least \(1-\delta\).

\subsection{Proof of Proposition \ref{prop:distributed_on_manifold}}
Let \(x \in \M\), choose and let \(\gamma^i_t\) be the shortest constant velocity path on the manifold \(\M\) connecting points \(x\) and \(x_i\). Choose \(i\) such that \(\int \|\partial_t \gamma_t\| dt\) is minimised, i.e. the closest example according to the metric on the manifold. Then the density ratio at \(x\) compared with \(x_i\) can be expressed as
\begin{align*}
    \frac{\hat{p}^{k}_\epsilon(x)}{\hat{p}^{k}_\epsilon(x_i)} &= \exp \bigg ( \int \log \hat{p}_\epsilon(y) k_{x}(dy) - \int \log \hat{p}_\epsilon(y) k_{x_i}(dy) \bigg )\\
    &= \exp \bigg ( \int \int \Big ( \log \hat{p}_\epsilon(y) + \frac{d}{2}\log(2\pi\sigma_\epsilon^2) \Big ) \langle \nabla_x k_{x}(y)\vert_{x=\gamma_t}, \partial_t \gamma_t \rangle dy dt \bigg ).
\end{align*}
We further control this using the Fisher information matrix,
\begin{equation*}
    \mathcal{I}(x) = \int (\nabla_x \log k_x(y)) (\nabla_x \log k_x(y))^T k_x(dy),
\end{equation*}
and the quantity \(F = \sup_{x \in \M} \sup_{v \in N_x \M}  \frac{v^T \mathcal{I}(x) v}{\|v\|^2}\). From the Cauchy-Schwarz inequality, we obtain the following bound:
\begin{align*}
    \int \Big ( \log & \hat{p}_\epsilon(y) + \frac{d}{2}\log(2\pi\sigma_\epsilon^2) \Big ) \langle \nabla_x k_{x}(y)\vert_{x=\gamma_t}, \partial_t \gamma_t \rangle dy\\
    &\leq \bigg ( \int |\langle \nabla_x \log k_{x}(y)\vert_{x=\gamma_t}, \partial_t \gamma_t \rangle|^2 k_{x}(dy) \bigg )^{1/2} \bigg ( \int \Big ( \log \hat{p}_\epsilon(y) + \frac{d}{2}\log(2\pi\sigma_\epsilon^2) \Big )^2 k_{x}(dy) \bigg )^{1/2}\\
    &\leq \Big ( (\partial_t \gamma_t)^T \mathcal{I}(x) (\partial_t \gamma_t) \Big )^{1/2} \bigg ( \int \Big ( \log \hat{p}_\epsilon(y) + \frac{d}{2}\log(2\pi\sigma_\epsilon^2) \Big )^2 k_{x}(dy) \bigg )^{1/2}\\
    &\leq F^{1/2} \bigg ( \int \Big ( \log \hat{p}_\epsilon(y) + \frac{d}{2}\log(2\pi\sigma_\epsilon^2) \Big )^2 k_{x}(dy) \bigg )^{1/2}.
\end{align*}
Using the argument in the proof of Corollary \ref{prop:manifold_concentration}, we can choose \(\varepsilon_0 \in [0, \tau/64]\) so that,
\begin{equation*}
    2\mu_\epsilon^2\varepsilon_0^2 + 8 \sigma_\epsilon^2 \varepsilon_0^{-d^*} \lesssim \max\{(c_\mu^2 N)^{-1/d^*}, \sigma_\epsilon^{\frac{4}{d^*+2}} c_\mu^{-\frac{2}{d^*+2}}\}.
\end{equation*}
Thus, borrowing the argument from \eqref{eq:log_bound} in the proof of Lemma \ref{lem:man_conc_simple}, we have that for all \(y \in \R^d\),
\begin{align*}
    \Big | \log \hat{p}_\epsilon(y) + \frac{d}{2}\log(2\pi\sigma_\epsilon^2) \Big | &\lesssim \frac{(\dist(y, \M^\epsilon) + K)^2}{\sigma_\epsilon^2} + \max\{(c_\mu^2 N)^{-1/d^*}, \sigma_\epsilon^{\frac{4}{d^*+2}} c_\mu^{-\frac{2}{d^*+2}}\}.
\end{align*}
Since \(x \in \M\), we also have the bound,
\begin{align*}
    \E_{Y \sim k_x}[(\dist(Y, \M^\epsilon) + K)^4]^{1/4} &\leq \E[(\dist(y, \M^\epsilon) - \dist(x, \M^\epsilon)^4]^{1/4} + K\\
    &\leq 2 K_{\max}^{1/2} K^{1/2}.
\end{align*}
In particular, there exists a quantity \(C \lesssim 1\), such that the following bound on the density ratio holds:
\begin{align*}
    \frac{\hat{p}^{k}_\epsilon(x)}{\hat{p}^{k}_\epsilon(x_i)} &\leq \exp \bigg ( \frac{C F (K K_{\max} + (c_\mu^2 N)^{-1/2d^*})}{\sigma_\epsilon^2} \int \|\partial_t \gamma_t\| dt \bigg ).
\end{align*}

Finally, we must control the distance of the path \(\gamma_t\). We recall from the definition of \(N_{\min}(\delta)\) in \eqref{eq:N_min} and Lemma \ref{lem:lower_emp_meas}, that whenever \(N \geq N_{\min}(\delta)\), we obtain
\begin{equation*}
    \prob \bigg ( \inf_{x \in M} \hmud(B_{\tau/2}(x)) \geq c_\mu \frac{(\tau/2)^{d^*}}{4} \bigg ) \geq 1 - \delta.
\end{equation*}
From this, it follows that with high probability, \(\inf_{i \in [N]} \|x - x_i\| \leq \tau/2\). Once we combine this with Lemma \ref{lem:geo_to_euc}, we obtain the bound,
\begin{equation*}
    \int \|\partial_t \gamma_t\| dt \leq 2 \|x - x_i\|,
\end{equation*}
completing the proof of the proposition.

\section{Experimental details} \label{app:experiment_details}

In this section, we provide detailed descriptions of the experimental settings used in the paper.

\subsection{2-dimensional circle example}\label{app:circle_details}

The plots in \Cref{fig:2d_circle_smoothing} illustrate the trade-off that score-smoothing can provide for generalisation, to complement the theoretical results and discussion in \cref{sec:theory-manifold-adapted-smoothing}. We consider an empirical dataset of 12 uniformly spaced points on the unit circle, and generate samples using the smoothed score function with an isotropic Gaussian kernel. We use a variance exploding diffusion model with $T=9$ and a geometric noise schedule, an Euler-Maruyama discretisation with 100 steps, and 1000 samples in the smoothing evaluation. In \cref{fig:2d_circle_smoothing} we show how the resulting samples behave for different smoothing parameter $\sigma$. Too little smoothing generates only training data, while too much smoothing causes generated samples to move towards the centre of the circle. There is a good choice of smoothing that balances between the two, promoting generalisation along the manifold (a phenomenon noticed by \citet{scarvelis2023closedformdiffusionmodels}).

To further illustrate this trade-off, we also plot how the population negative log-likelihood changes as the degree of smoothing increases, averaged over 1000 points on the true circular manifold. We recall that one can calculate the log-likelihood of a point by integrating the divergence of the probability-flow ODE drift function along the probability flow ODE trajectory \citep{chen18_neuralodes}. In our case, the drift of the probability flow ODE is a smoothed empirical score function.
As we are considering isotropic Gaussian smoothing, the divergence and the kernel convolution can be interchanged, allowing us to compute the log-likelihood by integrating the smoothed divergence of the empirical score function. The resulting plot exhibits a U-shape, clearly demonstrating the generalisation trade-off that arises from varying the smoothing level.

\subsection{Comparing Gaussian smoothing and KDE}
\label{app:KDE_vs_smooth_experiments}
We here describe the experimental setup used in \cref{sec:KDE_vs_smooth_experiments}.

\paragraph{VAE}
We train the VAE on 10,000 samples from the MNIST database.
The VAE uses 16 initial feature channels, with scaling multiples of $(1,2,2,2)$ during downsampling, a convolutional kernel size of 3, a dropout rate of 0.1, and 4 groups for group normalisation. It maps into a 32-dimensional latent space. It is trained for 10,000 training steps with a batch size of 64, using the Adam optimiser \citep{kingma15_adam} with learning rate 1e-3 and default parameters 0.9, 0.999.

\paragraph{Dataset construction} We use the remaining 60,000 points not used to train the VAE. The latent representations of the individual classes comprise lower-dimensional structures in the space, which we verify in \cref{fig:PCA_explained_variance}. To show this, for a particular class we map all points to the latent space. We then pick one of the latent points $z$, and look at its 50 closest neighbours $z_i$. We perform a PCA decomposition on the vectors $z - z_i$ to analyse their local structure. In \cref{fig:PCA_explained_variance} we plot the cumulative explained variance, and observe that most of the variance is captured by fewer than the full 32 dimensions, suggesting that the latent representations for a particular class lie on a lower-dimensional manifold within the latent space.

We restrict to considering the digit 4, of which there are 5842 samples. We use these points to approximate the `true' manifold $\mathcal{M}$, and randomly choose 100 points for use as the empirical dataset.

\paragraph{Experiment hyperparameters}
We use a variance-exploding diffusion model with T = 9.0, a geometric noise schedule, and 100 generation steps with an Euler-Maruyama discretisation scheme. We generate 500 samples to calculate the $L_2$ distances reported in \cref{fig:mnist_KDE_smoothing_L2_4s}. For the isotropic Gaussian kernel, we used smoothing with standard deviations $\sigma \in \{0.04, 0.05, 0.06, 0.07, 0.08, 0.09, 0.1\}$, and we use 1000 smoothing samples at each generation step. For KDE, we use $\sigma \in \{0.006, 0.01, 0.014, 0.018, 0.022, 0.026\}$, which were chosen to induce comparable average distances to the dataset from the generated samples (plotted along the $x$-axis).

To obtain the samples plotted in \cref{fig:mnist_KDE_smoothing_plots_4s}, we use Gaussian-smoothing with standard deviations $\sigma \in \{0.04, 0.05, 0.06, 0.07\}$, and KDE scales $\sigma \in \{0.015, 0.03, 0.045, 0.06\}$. These were selected to induce comparable lateral distances along the manifold, computed as $( d(x,\hat{\mu}_{data})^2 - d(x,\mathcal{M})^2)^{\frac{1}{2}}$ and averaged over 500 generated samples, which corresponds to inducing the same degree of `novelty' relative to the training set. We display the same plot showing more samples in \cref{fig:additional_plots_vae}.

\begin{figure}[t]
    \centering
    \begin{subfigure}[b]{0.28\linewidth}
        \includegraphics[width=\textwidth]{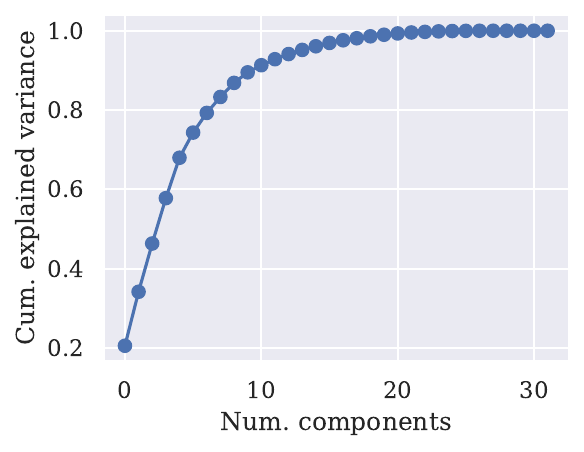}
        \caption{Digit 1}
        \label{fig:mnist_pca_explained_variance_1s}
    \end{subfigure}
    \hfill
    \begin{subfigure}[b]{0.28\linewidth}
        \includegraphics[width=\textwidth]{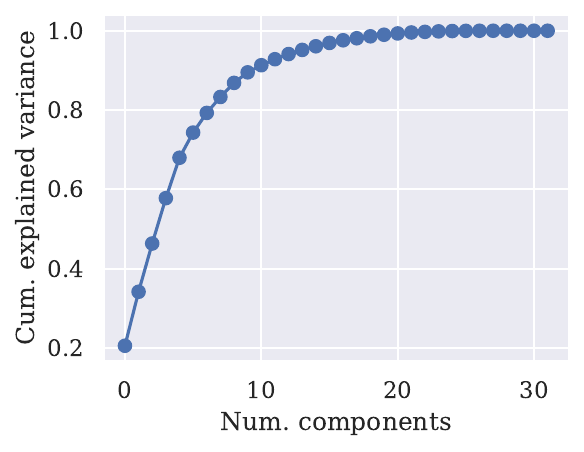}
        \caption{Digit 4}
        \label{fig:mnist_pca_explained_variance_4s}
    \end{subfigure}
    \hfill
    \begin{subfigure}[b]{0.28\linewidth}
        \includegraphics[width=\textwidth]{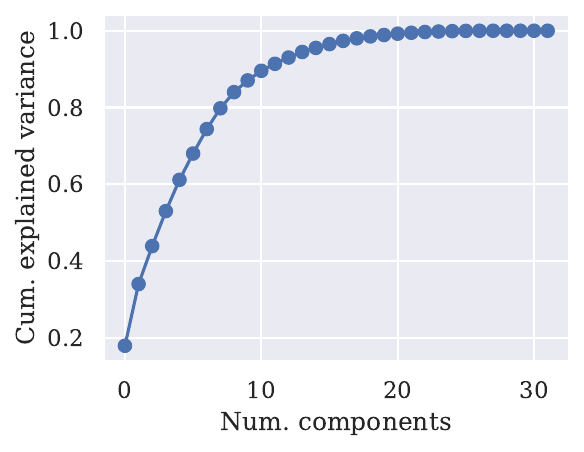}
        \caption{Digit 5}
        \label{fig:mnist_pca_explained_variance_5s}
    \end{subfigure}
    \caption{Verifying that latent representations of a digit class lie on a lower-dimensional structure in the latent space, by performing a PCA decomposition on the differences between nearby points.}
    \label{fig:PCA_explained_variance}
\end{figure}

\begin{figure}[t]
    \centering
    \begin{subfigure}[b]{0.45\linewidth}
        \vspace{-2em}
        \includegraphics[width=\textwidth]{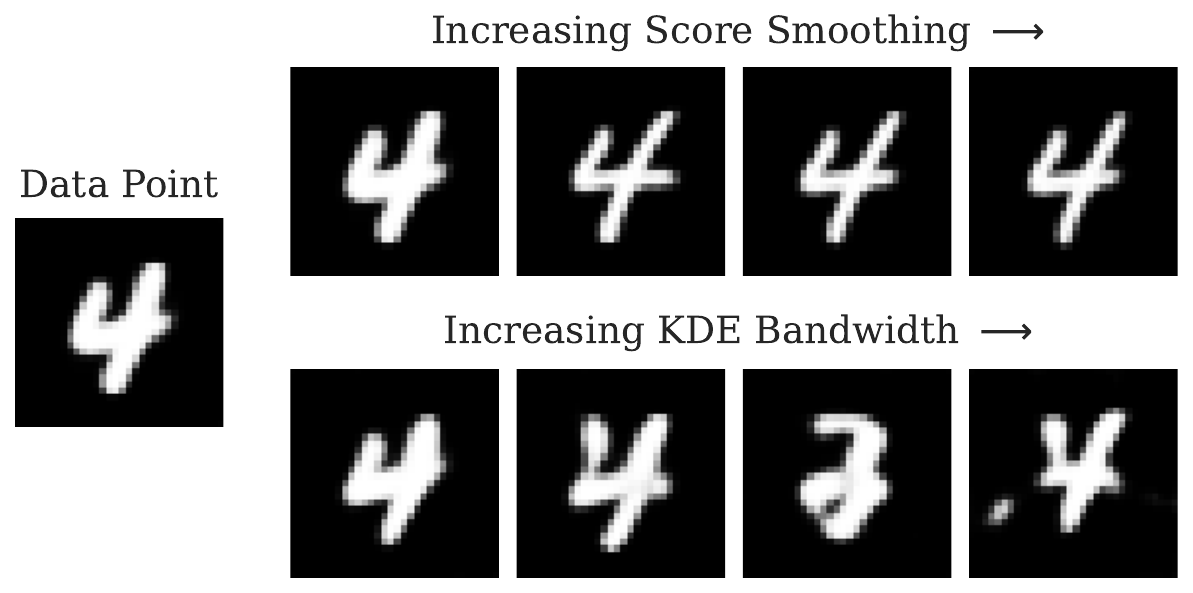}
        \caption{}
        \label{fig:mnist_KDE_smoothing_plots_4s_2}
    \end{subfigure}
    \hfill
    \begin{subfigure}[b]{0.45\linewidth}
        \includegraphics[width=\textwidth]{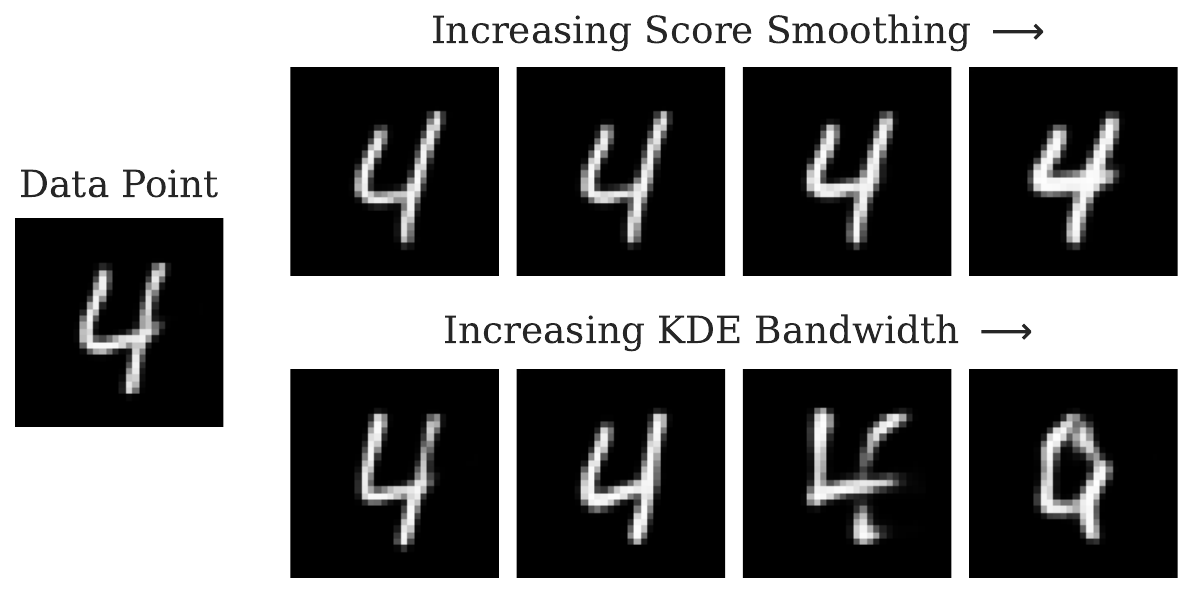}
        \caption{}
        \label{fig:mnist_KDE_smoothing_plots_4s_3}
    \end{subfigure}
    \caption{Additional generations, as in \cref{fig:mnist_KDE_smoothing_plots_4s}.}
    \label{fig:additional_plots_vae}
\end{figure}

\subsection{Synthetic image manifold}\label{app:synthetic_image}

\begin{figure}[]
    \noindent
    \begin{minipage}[t]{0.5\textwidth}
        \begin{figure}[H]
            \centering
            \begin{subfigure}[b]{0.45\linewidth}
                \centering
                \includegraphics[width=\textwidth]{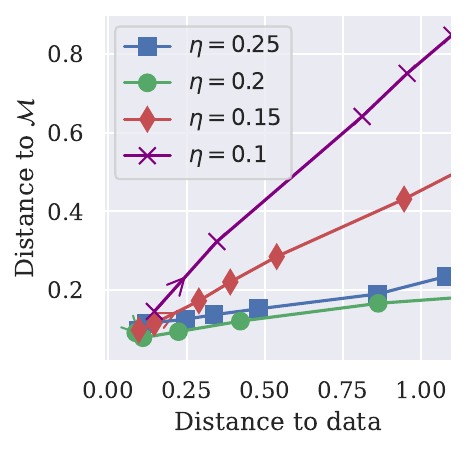}
                \caption{$N=24$}
            \end{subfigure}%
            \hfill
            \begin{subfigure}[b]{0.45\linewidth}
                \centering
                \includegraphics[width=\textwidth]{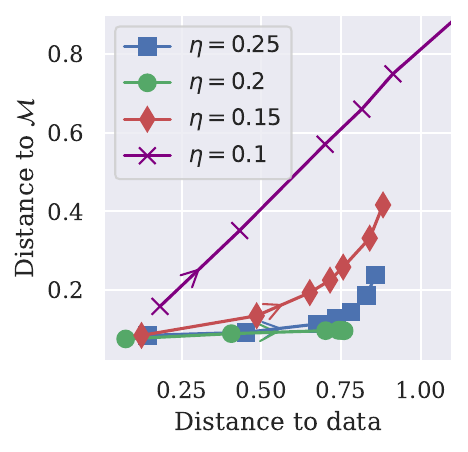}
                \caption{$N=40$}
            \end{subfigure}
            \caption{Additional plots demonstrating the effect of changing curvature on the manifold-adapted smoothing, as in \cref{fig:circle_images_curvature}. Different plots show different numbers of datapoints $N$ in the empirical dataset.}
            \label{fig:eta_appendix_plots}
        \end{figure}
    \end{minipage}%
    \hfill
    \begin{minipage}[t]{0.45\textwidth}
        \begin{figure}[H]
            \centering
            \vspace{-2pt} %
            \includegraphics[width=0.52\textwidth]{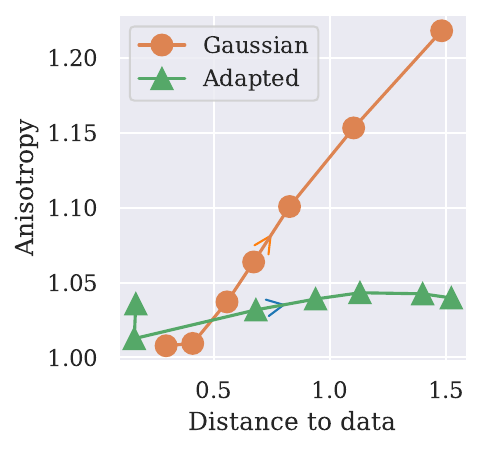}
            \caption{Plot showing the average `anisotropy' of the generated samples, as a measure of the visual quality of samples. Values close to 1 indicate a spherically-symmetric generated sample, as would be the case on the true manifold.}
            \label{fig:gaussian_bump_anisotropy}
        \end{figure}
    \end{minipage}
    \vspace{-1em}
\end{figure}

\begin{figure}[]
    \vspace{20pt}
    \centering
    \begin{subfigure}[b]{0.48\textwidth}
        \centering
        \includegraphics[width=\textwidth]{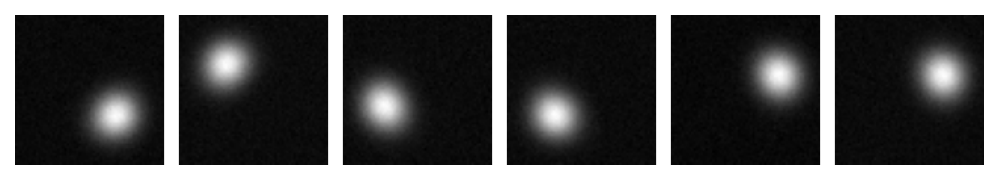}
        \caption{Isotropic Gaussian smoothing, $\sigma = 2.6$}
        \label{fig:one}
    \end{subfigure}%
    \hfill
    \begin{subfigure}[b]{0.48\textwidth}
        \centering
        \includegraphics[width=\textwidth]{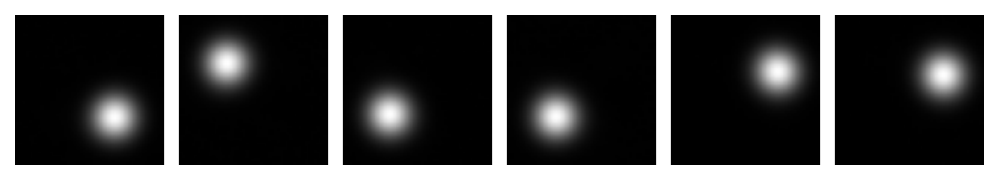}
        \caption{Adapted smoothing, $\sigma = 5.0$}
        \label{fig:two}
    \end{subfigure}
    \caption{Samples generated using the Gaussian and manifold-adapted smoothing kernels. The manifold-adapted smoothing generates samples that are visually more `on-manifold', in that the samples appear more spherically symmetric. See also \Cref{fig:gaussian_bump_anisotropy} for a quantitative measure of the spherical symmetry of the samples.}
    \label{fig:bump_generated_samples}
\end{figure}

We now describe the experimental setup used in \Cref{sec:synthetic image manifold}.

\paragraph{Dataset generation} We construct the synthetic image dataset using a function $\phi : [0, 2\pi) \to \R^{\text{64} \times \text{64}}$ that maps an angle $\theta$ to an `image'. The `image' is constructed as the density of a $\eta^2$-variance Gaussian distribution centred on the point on the circle with radius 0.5 corresponding to the angle $\theta$ (where the overall image corresponds to $[-1,1]\times[-1,1]$). The density is scaled to take values between 0 and 1. The resulting manifold in image space therefore consists of a closed curve of `Gaussian bumps' that move around the 0.5-circle as $\theta$ moves from $0$ to $2\pi$. We provide a visualisation of traversing the manifold in \cref{fig:gaussian_bump_manifold}. For the experiment in \Cref{fig:synth_images_smoothing_comparison}, we use $\eta=0.2$, and use 16 equally spaced points along the curve as the training dataset.

The manifold-adapted smoothing kernel is defined as follows. For a point $x$ in the generation procedure, the projection $\Pi_\mathcal{M}(x)$ onto the manifold is computed. 
We define a shifted manifold as $\mathcal{M} + (x - \Pi_\mathcal{M}(x))$, which is a translated copy of the manifold that passes through $x$. Gaussian noise of standard deviation $\sigma$ is added to $x$, then we project onto this shifted manifold. All manifold projections are approximated by generating 1024 equally spaced points along the manifold and taking the closest one.

\paragraph{Experiment hyperparameters} 
We use a variance-exploding diffusion model with $T=9.0$, a geometric noise schedule, and 100 generation steps with an Euler-Maruyama discretisation scheme. For the isotropic Gaussian kernel, we used smoothing with standard deviations $\sigma \in \{1.0, 1.4, 1.8, 2.0, 2.2, 2.4, 2.6\}$. For the manifold-adapted smoothing, we used $\sigma \in \{1.6, 2.4, 3.2, 3.5, 3.8, 4.4, 5.0\}$. These values were chosen to induce comparable average distances to the training dataset in the generated samples (plotted along the $x$-axis).

For the isotropic Gaussian smoothing, we took 50,000 kernel samples at each generation step in our kernel smoothing evaluation. For the manifold-adapted smoothing, we take 1000 smoothing samples at each generation step (note that this can be much lower than for Gaussian smoothing, as the manifold along which we smooth is only 1-dimensional). We generate 100 samples, and average the closest distances to the manifold and to the empirical dataset. As with the projections, the closest distance to the manifold is calculated by generating 1000 points on the manifold, and taking the minimum $L_2$ distance to these points.

In the experiment in \Cref{fig:circle_images_curvature} in which we vary the parameter $\eta$, we use 32 equally spaced points along the curve. We run the manifold-adapted smoothing with smoothing levels $\sigma \in \{0.8, 1.2, 1.6, 1.75, 1.9, 2.2, 2.5\}$, and vary the parameter to take values $\eta \in \{0.1, 0.15, 0.2, 0.25\}$. In \Cref{fig:eta_appendix_plots} we provide additional results in which we repeat the same experiment, but now change the number of datapoints $N$ in the empirical dataset. In each plot, we see a similar effect to in \Cref{fig:circle_images_curvature}—for larger $\eta$, the manifold-adapted smoothing is able to effectively generate new samples close to the manifold, but this ability is lessened as $\eta$ decreases.

\subsubsection{Additional plots}

The results in \cref{fig:synth_images_smoothing_comparison} indicate that an adapted smoothing kernel can induce different structure in the generations compared to isotropic Gaussian smoothing—as the degree of smoothing increases, the generated samples deviate away from the training data for both kernels, but remain comparatively closer to the manifold structure when using the adapted smoothing kernel. We here include some additional plots that further elucidate this observed effect.

\paragraph{Spread along the manifold}
While the $L_2$ distances reported in \cref{fig:synth_images_smoothing_comparison} show that the generations have deviated away from the training data, it is not necessarily clear how the scale of such deviations corresponds to the degree of spreading along the manifold structure.
We therefore also examine the extent to which the generated samples become spread along the manifold as the smoothing increases, to confirm that the generations do indeed deviate sufficiently far from the training points to reasonably be considered `novel'.

In \cref{fig:gaussian_blob_histograms}, we plot histograms showing the projected $\theta$ values of the generations, in order to see how far the generated distribution has spread along the 1$d$ synthetic manifold. We provide histograms for three different smoothing values used in \cref{fig:synth_images_smoothing_comparison}, for both types of smoothing. For small smoothing levels, we recover only training points as expected, but as the smoothing increases we see that the generations do indeed deviate far from the training datapoints relative to the manifold structure, and spread out to fill the gaps in the manifold between the points in the training dataset.

\begin{figure}[t]
    \centering
    \begin{subfigure}[b]{0.32\linewidth}
        \centering
        \includegraphics[width=\textwidth]{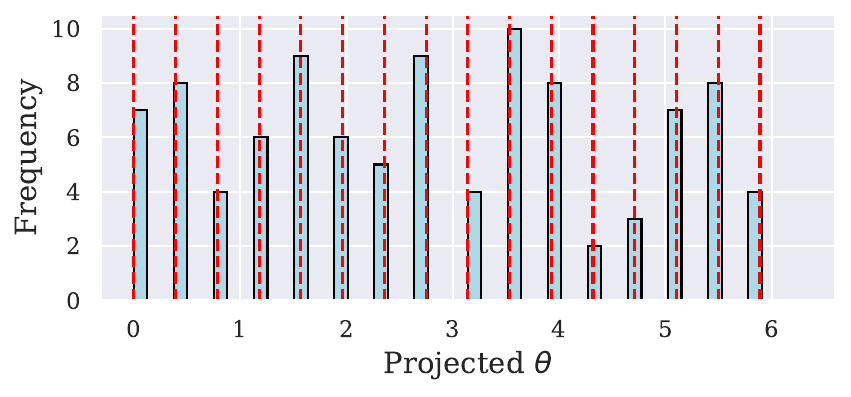}
        \caption{Gaussian kernel, $\sigma=1.4$}
    \end{subfigure}%
    \hfill
    \begin{subfigure}[b]{0.32\linewidth}
        \centering
        \includegraphics[width=\textwidth]{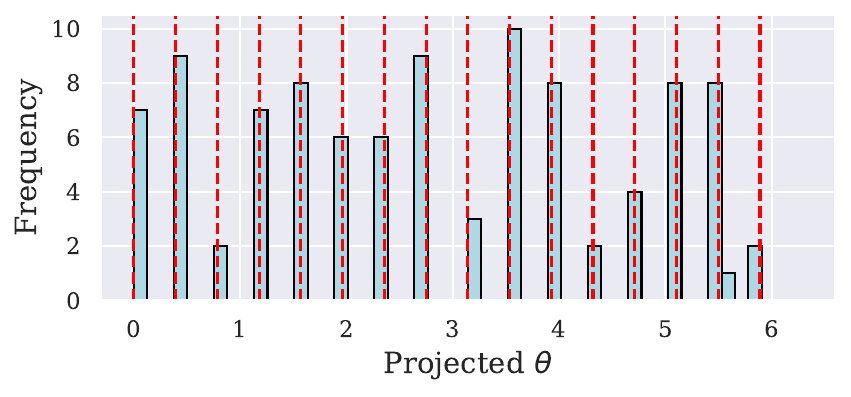}
        \caption{Gaussian kernel, $\sigma=2.2$}
    \end{subfigure}%
    \hfill
    \begin{subfigure}[b]{0.32\linewidth}
        \centering
        \includegraphics[width=\textwidth]{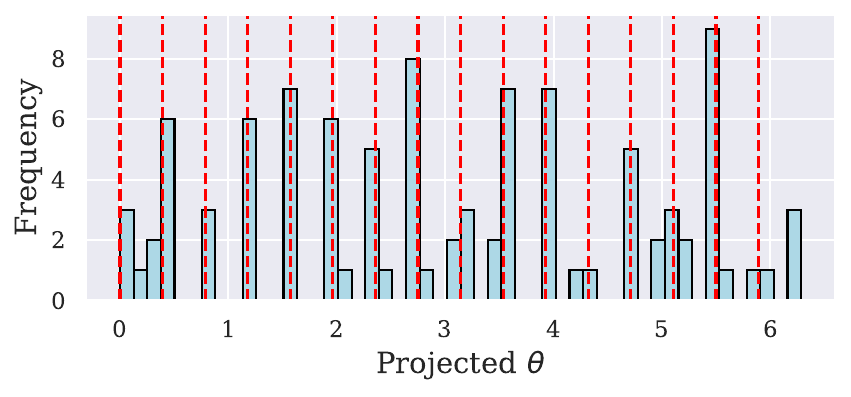}
        \caption{Gaussian kernel, $\sigma=2.6$}
    \end{subfigure}
    
    \vskip\baselineskip

    \begin{subfigure}[b]{0.32\linewidth}
        \centering
        \includegraphics[width=\textwidth]{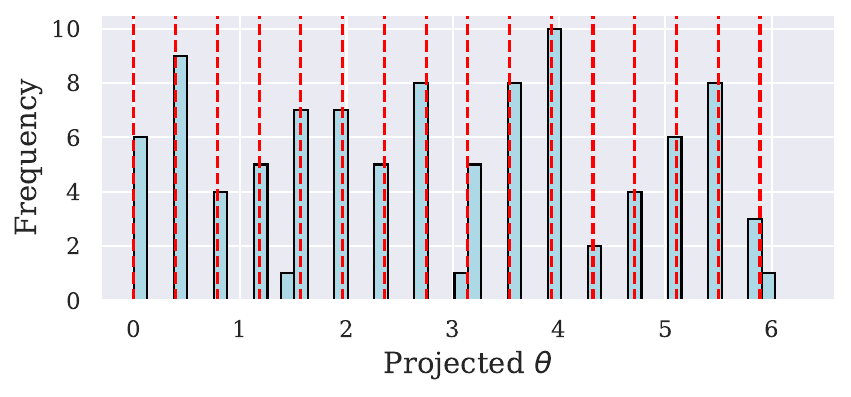}
        \caption{Adapted kernel, $\sigma=2.4$}
    \end{subfigure}%
    \hfill
    \begin{subfigure}[b]{0.32\linewidth}
        \centering
        \includegraphics[width=\textwidth]{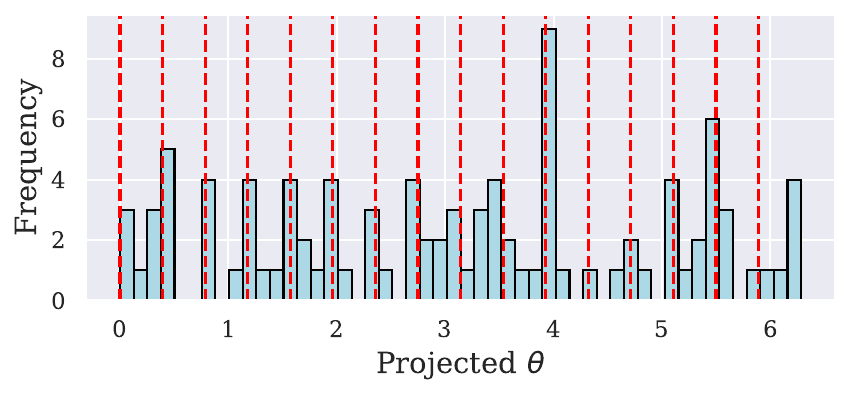}
        \caption{Adapted kernel, $\sigma=3.8$}
    \end{subfigure}%
    \hfill
    \begin{subfigure}[b]{0.32\linewidth}
        \centering
        \includegraphics[width=\textwidth]{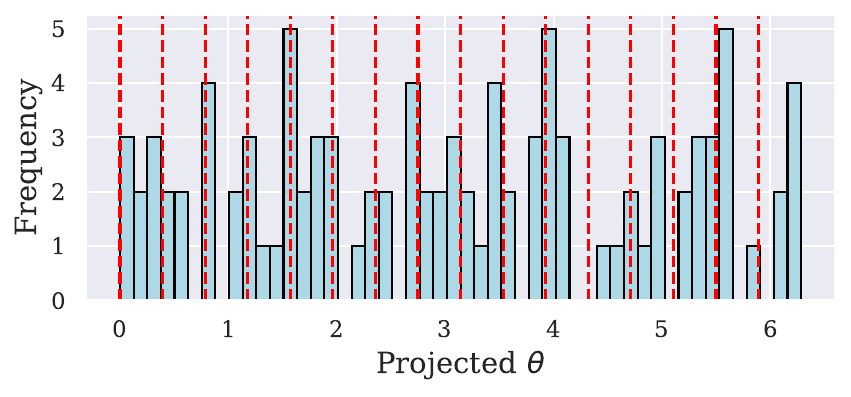}
        \caption{Adapted kernel, $\sigma=5.0$}
    \end{subfigure}

    \caption{Plots showing the projected $\theta$ values for the generated samples, for different amounts of smoothing. As the smoothing increases, the generated samples spread along the manifold structure, and populate the space between the points in the training dataset (indicated by red vertical lines).}
    \label{fig:gaussian_blob_histograms}
\end{figure}

\paragraph{Assessing the visual quality of samples}

The plot in \cref{fig:synth_images_smoothing_comparison} reports the average $L_2$ distance from the manifold $\mathcal{M}$, relative to the average $L_2$ distance from the training dataset. It is clear that the Gaussian-smoothed samples deviate comparatively further from the manifold than the adaptive-smoothing samples according to this distance. We here additionally report an alternative measure of how `on-manifold' the generations are, related to the \textit{visual} properties of the generations.

Note that samples from the true manifold consists of renormalised Gaussian density functions.
Visually, being `on-manifold' therefore corresponds to the generated images being spherically symmetric. 
In \Cref{fig:bump_generated_samples} we display generated samples for the isotropic Gaussian and manifold-adapted smoothing mechanisms (for the largest smoothing values that were used in \Cref{fig:synth_images_smoothing_comparison}), and see that the manifold-adapted smoothing generates samples appear visually more spherically symmetric. This property is however somewhat difficult to assess by eye, as any such changes can be subtle, so in \Cref{fig:gaussian_bump_anisotropy} we also quantitatively measure the spherical symmetry to assess this visual property.

In order to do so, we report the `anisotropy' of the generated samples. Namely, we consider the renormalised generated samples as a probability density function on $[0,1] \times [0,1]$, and record the anisotropy of the corresponding distribution (that is, we compute the covariance matrix $\Sigma \in \mathbb{R}^2$, and report $\frac{\lambda_{max}}{\lambda_{min}}$ for eigenvalues $\lambda_{max}, \lambda_{min}$). Samples that are `on-manifold' will have values close to 1.0. In the computation, we set values less than 0.1 to zero, so that the noise in the generations do not impact the calculation.
We report the results in \cref{fig:gaussian_bump_anisotropy}.

The results are consistent with the pattern of $L_2$ distances reported in \cref{fig:synth_images_smoothing_comparison}—as the degree of smoothing increases and the generated samples deviate away from the training datapoints, the generations using the adapted smoothing have lower anisotropy and are therefore more `round' than those obtained from Gaussian smoothing. Indeed, we know from \citet{scarvelis2023closedformdiffusionmodels} that Gaussian smoothing will generate barycentres of training points, which will skew the generated samples away from being perfectly round; it appears that the manifold-adapted smoothing somewhat mitigates this effect by shaping the geometry of the generated samples towards a different interpolation.

\subsection{MNIST manifold}
\label{app:MNIST_compare_smoothing}

We now provide the details for the experiment in \Cref{sec:image manifold}.

\paragraph{Dataset generation}
Similarly to the synthetic case, we construct a manifold by defining a curve ${\phi : [0,1] \to \R^{\text{32} \times \text{32}}}$ in pixel space, which interpolates between samples of the same digit from the MNIST dataset \citep{lecun2010mnist}. To obtain such an interpolation, we train a convolutional VAE \citep{kingma2022autoencodingvariationalbayes},
We then choose three datapoints from the same digit class (in this case, the digit 4), and draw a triangle between their latent representations. We construct $\phi(t)$ by decoding this triangle, which results in a closed loop in pixel space. We use the decodings of 10 equidistant points along the latent triangular interpolation to define the training dataset.
We emphasise that the VAE is only used to construct a manifold structure in pixel-space, and the actual diffusion procedure takes place directly in the pixel-space without any interaction with the VAE.

\paragraph{Experiment hyperparameters}
We use a variance-exploding diffusion model with $T=9.0$, a geometric noise schedule, and 100 generation steps with an Euler-Maruyama discretisation scheme. 
We used smoothing with standard deviations $\sigma \in \{0.0, 0.3, 0.6, 0.8, 0.9, 1.0, 1.05, 1.1 \}$ for isotropic Gaussian smoothing, and $\sigma \in \{0.0, 0.5, 1.0, 1.5, 2.0, 2.5, 4.0, 7.0\}$ for manifold-adapted smoothing (which again were chosen to induce similar distances from the data points in \cref{fig:mnist_L2_4s}).
For the isotropic Gaussian smoothing, we took 50,000 kernel samples at each generation step in our kernel smoothing evaluation. For the manifold-adapted smoothing, we take 1000 smoothing samples at each generation step (this can be much lower than for Gaussian smoothing, as the manifold along which we smooth is only 1-dimensional). As before, we generate 100 samples, and report the average closest distances to the manifold and to the empirical dataset. The closest distance to the manifold is calculated by generating 1000 points on the manifold, and taking the minimum $L_2$ distance to these points.

\paragraph{FID calculation}
As we work with a 1-dimensional cuve in pixel space, neighbouring points in the empirical dataset look very similar. It is therefore difficult to visually judge the quality of obtained samples from both smoothing mechanisms, so we use FID \citep{heusel17_FID} as measure of similarity to the true manifold that also provides an indication of sample quality.
To obtain the features used for the FID calculation, we train a convolutional classifier (using the 10,000 points also used to train the VAE).
The model consists of two convolutional layers with 32 and 64 features respectively, followed by a fully connected hidden layer of size 128. The 128-dimensional feature vector is used for the FID calculation. It is trained for 5,000 steps using the Adam optimiser with a learning rate of 1e-3. We calculate the FID scores of the generated samples relative to the 1000 random samples from the manifold.

\paragraph{Different manifolds}
We also ran the same experiment with manifolds for different digits, and observe similar behaviour.
Results for the curves for digits 2 and 7 are plotted in \cref{fig:alternative_mnist_curves}. The selected points were generally chosen to be the first three examples of that digit in the dataset (other than when these datapoints induced a poorly-decoded manifold, in which case we used the first that made the constructed manifold of good quality).

\begin{figure}[t]
    \centering
    \begin{subfigure}[b]{0.22\linewidth}
        \vspace{-2em}
        \includegraphics[width=\textwidth]{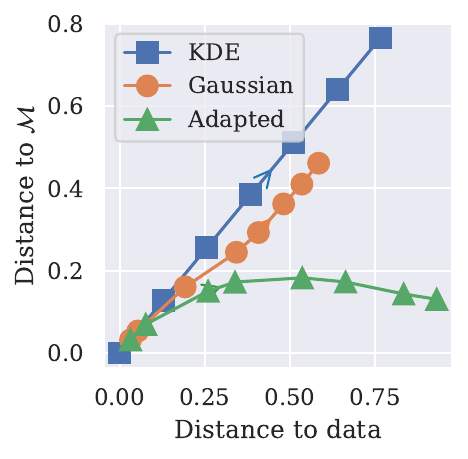}
        \caption{$L_2$ distances, for curve in 2s class.}
        \label{fig:mnist_L2_2s}
    \end{subfigure}
    \hfill
    \begin{subfigure}[b]{0.22\linewidth}
        \includegraphics[width=\textwidth]{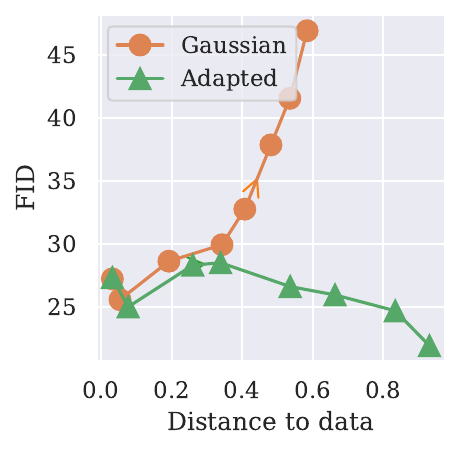}
        \caption{FID, for curve in 2s class.}
        \label{fig:mnist_fid_2s}
    \end{subfigure}
    \hfill
    \qquad
    \hfill
    \begin{subfigure}[b]{0.22\linewidth}
        \vspace{-2em}
        \includegraphics[width=\textwidth]{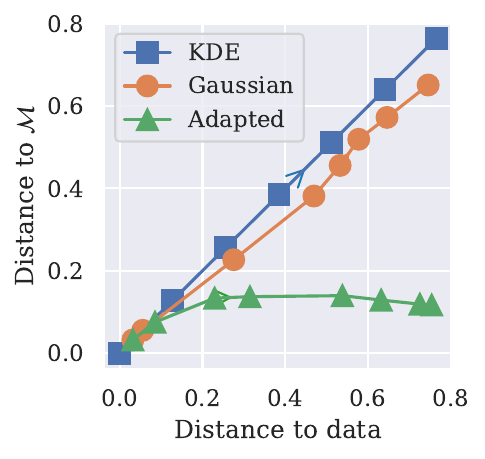}
        \caption{$L_2$ distances, for curve in 7s class.}
        \label{fig:mnist_L2_7s}
    \end{subfigure}
    \hfill
    \begin{subfigure}[b]{0.22\linewidth}
        \includegraphics[width=\textwidth]{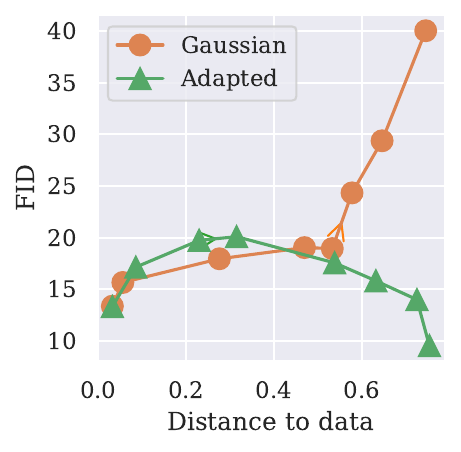}
        \caption{FID, for curve in 7s class.}
        \label{fig:mnist_fid_7s}
    \end{subfigure}
    \caption{Comparison of Gaussian and manifold-adapted smoothing kernels, for alternative curves $\phi$ in the manifold of digits 2 and 7.}
    \label{fig:alternative_mnist_curves}
\end{figure}

\paragraph{Licenses}

\begin{itemize}
    \item MNIST digits classification dataset \citep{lecun2010mnist}, CC BY-SA 3.0 License
\end{itemize}

\end{document}